\documentclass{article} %

\ifdefined\pdfsuppressptexinfo\pdfsuppressptexinfo=-1\fi
\ifdefined\pdfinfoomitdate\pdfinfoomitdate=1\fi
\ifdefined\pdftrailerid\pdftrailerid{}\fi
\PassOptionsToPackage{table}{xcolor}
\usepackage{iclr2027_conference,times}
\iclrfinalcopy

\usepackage{latexsym}   
\usepackage{amsbsy}
\usepackage{amssymb}
\usepackage{amsmath}
\usepackage[final]{graphicx}
\usepackage{enumitem}
\usepackage{amsthm}

\usepackage{booktabs}

\allowdisplaybreaks

\newtheorem{theorem}{Theorem}%
\newtheorem{corollary}{Corollary}%
\newtheorem{proposition}{Proposition}[section]

\newtheorem{remark}{Remark}%

\newcommand{\beq}{\begin{equation}}
\newcommand{\eeq}{\end{equation}}

\renewcommand{\top}{{\mkern-1.5mu\mathsf{T}}}

\newcommand{\Rbb}{{\mathbb{R}}}

\newcommand{\sG}{{\scriptscriptstyle G}}

\newcommand{\din}{{d_{\mathrm{in}}}}
\newcommand{\dout}{{d_{\mathrm{out}}}}

\newcommand{\A}{\mathbf{A}}
\newcommand{\I}{\mathbf{I}}
\newcommand{\Q}{\mathbf{Q}}
\newcommand{\boldS}{\mathbf{S}}
\newcommand{\W}{\mathbf{W}}

\newcommand{\s}{\mathbf{s}}
\newcommand{\w}{{\mathbf{w}}}

\newcommand{\hbW}{{\widehat{\W}}}

\newcommand{\hbw}{{\widehat{\w}}}

\usepackage{amsmath,amsfonts,bm}

\def\eqref#1{equation~\ref{#1}}

\def\sG{{\mathbb{G}}}

\renewcommand{\eqref}[1]{(\ref{#1})}
\renewcommand{\sG}{{\scriptscriptstyle G}}
\usepackage{multirow}
\usepackage{float}
\usepackage[table]{xcolor}
\definecolor{jarqFP16}{HTML}{E7E7E7}
\definecolor{jarqRefined}{HTML}{E4EEF7}
\definecolor{jarqGood}{HTML}{1A7F37}
\definecolor{jarqBad}{HTML}{C0392B}
\newcommand{\pplgood}[1]{{\tiny\textcolor{jarqGood}{#1}}}
\newcommand{\pplbad}[1]{{\tiny\textcolor{jarqBad}{#1}}}

\newlength{\ppldw}
\floatstyle{ruled}
\newfloat{algorithm}{tbp}{loa}
\floatname{algorithm}{Algorithm}
\usepackage{hyperref}
\usepackage{url}
\usepackage{xspace}
\usepackage{booktabs} 
\usepackage{graphicx}
\usepackage{wrapfig}
\usepackage{pgfplots}
\usepgfplotslibrary{groupplots}
\pgfplotsset{compat=1.18}
\newcommand{\ppl}[1]{%
  \pgfmathprintnumber[verbatim,fixed,precision=4,fixed zerofill]{#1}}
\newcommand{\method}{\textsc{JARQ}\xspace}
\newcommand{\methodfull}{%
\textbf{J}oint
\textbf{A}lternating
\textbf{R}efinement for
\textbf{Q}uantization\xspace}

\title{JARQ: Joint Alternating Refinement\\for Quantization}

\author{Xinyu Wang\textsuperscript{1}\thanks{Equal contribution.}\quad
Sicheng Lyu\textsuperscript{1}\footnotemark[1]\quad
Xiao-Wen Chang\textsuperscript{1}\thanks{Corresponding author.}\\[5pt]
\normalfont\textsuperscript{1}McGill University\\[3pt]
\normalfont\small\href{mailto:xinyu.wang5@mail.mcgill.ca}{\texttt{xinyu.wang5@mail.mcgill.ca}}\quad
\href{mailto:sicheng.lyu@mail.mcgill.ca}{\texttt{sicheng.lyu@mail.mcgill.ca}}\\
\normalfont\small\href{mailto:xiaowen.chang@mcgill.ca}{\texttt{xiaowen.chang@mcgill.ca}}
}

\date{}
\hypersetup{
  pdftitle={JARQ: Joint Alternating Refinement for Quantization},
  pdfauthor={Xinyu Wang, Sicheng Lyu, Xiao-Wen Chang},
  pdfsubject={Post-training quantization of large language models},
  hidelinks
}

\begin{document}
\raggedbottom

\maketitle
\fancyhead{}

\begin{abstract}
Group-wise post-training quantizers for large language models round
weights onto a grid that is not refit to the resulting integer codes.
We show that this leaves accuracy on the table: the best grid depends on the codes,
input correlations couple the errors of different groups, and useful
code changes often involve many codes at once. We propose \method{}
(\methodfull{}), a plug-in refinement that starts from any group-wise
quantizer and alternates a joint least-squares fit of all group scales
with bounded Babai proposals that move many codes of a group together
on the current grid. The problem is a bilinear box-constrained
mixed-integer least-squares problem; the solver is
backpropagation-free, does not increase the layer-wise objective under
exact scale solves, and keeps the host's bit width, groups, zero points,
and inference cost. Across Llama-2, Llama-3, and Qwen models
with RTN, GPTQ, OmniQuant, and AWQ hosts, \method{} lowers perplexity
in 90 of 96 comparisons, cuts three-bit RTN perplexity by up to 36\%,
raises mean multiple-choice accuracy in 23 of 24 configurations, and
improves QEP, QuaRot, and OJBKQ outputs, at under a minute per 7B block. Code will be available at
\url{https://github.com/Euphoria040201/JARQ}.
\end{abstract}

\section{Introduction}
\label{sec:introduction}

Post-training quantization (PTQ) makes large language models cheaper
to serve by storing weights in a few bits, using only a small
calibration set~\citep{frantar-gptq,lin2023awq,OmniQuant,xiao2023smoothquant,ashkboos2024quarot}.
At three or four bits, a group-wise quantizer stores each weight as an
integer code on a grid set by a per-group scale. Widely used methods
obtain the codes by rounding onto a grid: GPTQ fixes the grid and rounds
with sequential error compensation, AWQ rounds after activation-aware
scaling, and OmniQuant learns clipping thresholds by gradient descent
and rounds each weight to the current
grid~\citep{frantar-gptq,lin2023awq,OmniQuant}. The final grid is not
refit to the codes it produces, and those codes are not revisited.

Leaving the grid fixed is suboptimal, because the grid and the codes
interact in three ways. First, the best grid depends on the codes: once
some codes change, a different scale fits the layer output better, and
on that grid other codes become optimal, so a grid fixed in advance
goes stale (\emph{grid staleness}). Second, the groups of a channel
interact: correlated inputs mix their output errors, so the best scale
of one group depends on the scales of all others (\emph{cross-group
coupling}). Third, useful code changes are often collective: within a
group, correlated inputs create improvements that no single-code change
can reach (\emph{single-code stationarity}). OJBKQ captures the third
effect by decoding many codes at once, but on a fixed, precomputed
grid~\citep{OJBKQ}. Figure~\ref{fig:teaser} illustrates the first and
third effects on one group of two correlated weights, where refitting
the scale unlocks a two-code move.

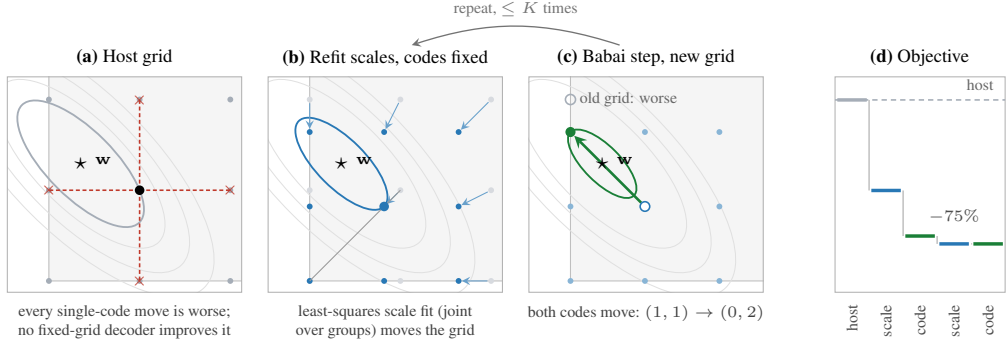
\begin{figure*}[t]
\centering
\definecolor{tzhost}{HTML}{A5ADB7}
\definecolor{tzscale}{HTML}{2878B5}
\definecolor{tzcode}{HTML}{1A7F37}
\definecolor{tzbad}{HTML}{C0392B}
\begin{tikzpicture}[x=1.00cm, y=1.00cm, >=stealth]
\begin{scope}
\begin{scope}
\clip (-0.55,-0.15) rectangle (2.55,2.70);
\fill[black!4] (0,0) rectangle (3.6000,3.6000);
\draw[black!25, line width=0.4pt] (0,0) rectangle (3.6000,3.6000);
\draw[black!12, line width=0.35pt, rotate around={-45:(0.4200,1.5500)}] (0.4200,1.5500) ellipse (1.4639 and 0.5906);
\draw[black!12, line width=0.35pt, rotate around={-45:(0.4200,1.5500)}] (0.4200,1.5500) ellipse (1.8898 and 0.7625);
\draw[black!12, line width=0.35pt, rotate around={-45:(0.4200,1.5500)}] (0.4200,1.5500) ellipse (2.3146 and 0.9339);
\draw[tzhost, line width=0.7pt, rotate around={-45:(0.4200,1.5500)}] (0.4200,1.5500) ellipse (1.0983 and 0.4432);
\fill[tzhost] (0.0000,0.0000) circle (1.1pt);
\fill[tzhost] (0.0000,1.2000) circle (1.1pt);
\fill[tzhost] (0.0000,2.4000) circle (1.1pt);
\fill[tzhost] (1.2000,0.0000) circle (1.1pt);
\fill[tzhost] (1.2000,1.2000) circle (1.1pt);
\fill[tzhost] (1.2000,2.4000) circle (1.1pt);
\fill[tzhost] (2.4000,0.0000) circle (1.1pt);
\fill[tzhost] (2.4000,1.2000) circle (1.1pt);
\fill[tzhost] (2.4000,2.4000) circle (1.1pt);
\draw[tzbad, dash pattern=on 1.5pt off 1pt, line width=0.6pt] (1.2000,1.2000) -- (2.4000,1.2000);
\node[tzbad, font=\tiny, inner sep=0pt] at (2.4000,1.2000) {$\times$};
\draw[tzbad, dash pattern=on 1.5pt off 1pt, line width=0.6pt] (1.2000,1.2000) -- (0.0000,1.2000);
\node[tzbad, font=\tiny, inner sep=0pt] at (0.0000,1.2000) {$\times$};
\draw[tzbad, dash pattern=on 1.5pt off 1pt, line width=0.6pt] (1.2000,1.2000) -- (1.2000,2.4000);
\node[tzbad, font=\tiny, inner sep=0pt] at (1.2000,2.4000) {$\times$};
\draw[tzbad, dash pattern=on 1.5pt off 1pt, line width=0.6pt] (1.2000,1.2000) -- (1.2000,0.0000);
\node[tzbad, font=\tiny, inner sep=0pt] at (1.2000,0.0000) {$\times$};
\node[font=\small, inner sep=0pt] at (0.4200,1.5500) {$\star$};
\node[font=\tiny, anchor=west] at (0.4900,1.6000) {$\mathbf{w}$};
\fill[black] (1.2000,1.2000) circle (1.8pt);
\end{scope}
\draw[black!30] (-0.55,-0.15) rectangle (2.55,2.70);
\node[anchor=south, font=\scriptsize, align=center] at (1.00,2.72) {\textbf{(a)} Host grid};
\node[font=\tiny, align=center, text=black!75, anchor=north] at (1.00,-0.20) {every single-code move is worse;\\no fixed-grid decoder improves it};
\end{scope}
\begin{scope}[xshift=3.45cm]
\begin{scope}
\clip (-0.55,-0.15) rectangle (2.55,2.70);
\fill[black!4] (0,0) rectangle (2.9550,2.9550);
\draw[black!25, line width=0.4pt] (0,0) rectangle (2.9550,2.9550);
\draw[black!12, line width=0.35pt, rotate around={-45:(0.4200,1.5500)}] (0.4200,1.5500) ellipse (1.4639 and 0.5906);
\draw[black!12, line width=0.35pt, rotate around={-45:(0.4200,1.5500)}] (0.4200,1.5500) ellipse (1.8898 and 0.7625);
\draw[black!12, line width=0.35pt, rotate around={-45:(0.4200,1.5500)}] (0.4200,1.5500) ellipse (2.3146 and 0.9339);
\draw[tzscale, line width=0.7pt, rotate around={-45:(0.4200,1.5500)}] (0.4200,1.5500) ellipse (0.7990 and 0.3224);
\fill[tzhost!45] (0.0000,0.0000) circle (1.1pt);
\fill[tzhost!45] (0.0000,1.2000) circle (1.1pt);
\fill[tzhost!45] (0.0000,2.4000) circle (1.1pt);
\fill[tzhost!45] (1.2000,0.0000) circle (1.1pt);
\fill[tzhost!45] (1.2000,1.2000) circle (1.1pt);
\fill[tzhost!45] (1.2000,2.4000) circle (1.1pt);
\fill[tzhost!45] (2.4000,0.0000) circle (1.1pt);
\fill[tzhost!45] (2.4000,1.2000) circle (1.1pt);
\fill[tzhost!45] (2.4000,2.4000) circle (1.1pt);
\draw[->, tzscale!70, line width=0.4pt, shorten >=1.2pt, shorten <=1.4pt] (1.2000,1.2000) -- (0.9850,0.9850);
\draw[->, tzscale!70, line width=0.4pt, shorten >=1.2pt, shorten <=1.4pt] (2.4000,1.2000) -- (1.9700,0.9850);
\draw[->, tzscale!70, line width=0.4pt, shorten >=1.2pt, shorten <=1.4pt] (1.2000,2.4000) -- (0.9850,1.9700);
\draw[->, tzscale!70, line width=0.4pt, shorten >=1.2pt, shorten <=1.4pt] (2.4000,2.4000) -- (1.9700,1.9700);
\draw[->, tzscale!70, line width=0.4pt, shorten >=1.2pt, shorten <=1.4pt] (0.0000,2.4000) -- (0.0000,1.9700);
\draw[->, tzscale!70, line width=0.4pt, shorten >=1.2pt, shorten <=1.4pt] (2.4000,0.0000) -- (1.9700,0.0000);
\fill[tzscale] (0.0000,0.0000) circle (1.1pt);
\fill[tzscale] (0.0000,0.9850) circle (1.1pt);
\fill[tzscale] (0.0000,1.9700) circle (1.1pt);
\fill[tzscale] (0.9850,0.0000) circle (1.1pt);
\fill[tzscale] (0.9850,0.9850) circle (1.1pt);
\fill[tzscale] (0.9850,1.9700) circle (1.1pt);
\fill[tzscale] (1.9700,0.0000) circle (1.1pt);
\fill[tzscale] (1.9700,0.9850) circle (1.1pt);
\fill[tzscale] (1.9700,1.9700) circle (1.1pt);
\draw[black!40, line width=0.4pt] (0,0) -- (1.2000,1.2000);
\node[font=\small, inner sep=0pt] at (0.4200,1.5500) {$\star$};
\node[font=\tiny, anchor=west] at (0.4900,1.6000) {$\mathbf{w}$};
\fill[tzscale] (0.9850,0.9850) circle (1.8pt);
\end{scope}
\draw[black!30] (-0.55,-0.15) rectangle (2.55,2.70);
\node[anchor=south, font=\scriptsize, align=center] at (1.00,2.72) {\textbf{(b)} Refit scales, codes fixed};
\node[font=\tiny, align=center, text=black!75, anchor=north] at (1.00,-0.20) {least-squares scale fit (joint\\over groups) moves the grid};
\end{scope}
\begin{scope}[xshift=6.90cm]
\begin{scope}
\clip (-0.55,-0.15) rectangle (2.55,2.70);
\fill[black!4] (0,0) rectangle (2.9550,2.9550);
\draw[black!25, line width=0.4pt] (0,0) rectangle (2.9550,2.9550);
\draw[black!12, line width=0.35pt, rotate around={-45:(0.4200,1.5500)}] (0.4200,1.5500) ellipse (1.4639 and 0.5906);
\draw[black!12, line width=0.35pt, rotate around={-45:(0.4200,1.5500)}] (0.4200,1.5500) ellipse (1.8898 and 0.7625);
\draw[black!12, line width=0.35pt, rotate around={-45:(0.4200,1.5500)}] (0.4200,1.5500) ellipse (2.3146 and 0.9339);
\draw[tzcode, line width=0.7pt, rotate around={-45:(0.4200,1.5500)}] (0.4200,1.5500) ellipse (0.5940 and 0.2397);
\fill[tzscale!55] (0.0000,0.0000) circle (1.1pt);
\fill[tzscale!55] (0.0000,0.9850) circle (1.1pt);
\fill[tzscale!55] (0.0000,1.9700) circle (1.1pt);
\fill[tzscale!55] (0.9850,0.0000) circle (1.1pt);
\fill[tzscale!55] (0.9850,0.9850) circle (1.1pt);
\fill[tzscale!55] (0.9850,1.9700) circle (1.1pt);
\fill[tzscale!55] (1.9700,0.0000) circle (1.1pt);
\fill[tzscale!55] (1.9700,0.9850) circle (1.1pt);
\fill[tzscale!55] (1.9700,1.9700) circle (1.1pt);
\draw[tzhost, line width=0.6pt] (0.0000,2.4000) circle (1.9pt);
\node[font=\tiny, text=black!60, anchor=west] at (0.0000,2.4000) {old grid: worse};
\draw[->, tzcode, line width=1.0pt, shorten >=2pt, shorten <=2pt] (0.9850,0.9850) -- (0.0000,1.9700);
\draw[tzscale, line width=0.6pt, fill=white] (0.9850,0.9850) circle (1.8pt);
\node[font=\small, inner sep=0pt] at (0.4200,1.5500) {$\star$};
\node[font=\tiny, anchor=west] at (0.4900,1.6000) {$\mathbf{w}$};
\fill[tzcode] (0.0000,1.9700) circle (1.9pt);
\end{scope}
\draw[black!30] (-0.55,-0.15) rectangle (2.55,2.70);
\node[anchor=south, font=\scriptsize, align=center] at (1.00,2.72) {\textbf{(c)} Babai step, new grid};
\node[font=\tiny, align=center, text=black!75, anchor=north] at (1.00,-0.20) {both codes move: $(1,1)\to(0,2)$};
\end{scope}
\draw[->, black!55, line width=0.5pt] (7.55,3.12) to[bend right=18] node[midway, above, font=\tiny] {repeat, $\le K$ times} (4.80,3.12);
\begin{scope}[xshift=10.40cm]
\draw[black!30] (0,-0.15) rectangle (2.25,2.70);
\node[anchor=south, font=\scriptsize] at (1.12,2.72) {\textbf{(d)} Objective};
\draw[tzhost, dash pattern=on 2pt off 1.5pt, line width=0.6pt] (0,2.3946) -- (2.25,2.3946);
\node[font=\tiny, text=black!60, anchor=south east] at (2.22,2.3946) {host};
\draw[tzhost, line width=1.3pt] (0.0300,2.3946) -- (0.4200,2.3946);
\node[font=\tiny, text=black!70, rotate=90, anchor=east] at (0.2250,-0.19) {host};
\draw[tzscale, line width=1.3pt] (0.4800,1.1967) -- (0.8700,1.1967);
\draw[black!35, line width=0.4pt] (0.4500,2.3946) -- (0.4500,1.1967);
\node[font=\tiny, text=black!70, rotate=90, anchor=east] at (0.6750,-0.19) {scale};
\draw[tzcode, line width=1.3pt] (0.9300,0.5942) -- (1.3200,0.5942);
\draw[black!35, line width=0.4pt] (0.9000,1.1967) -- (0.9000,0.5942);
\node[font=\tiny, text=black!70, rotate=90, anchor=east] at (1.1250,-0.19) {code};
\draw[tzscale, line width=1.3pt] (1.3800,0.4900) -- (1.7700,0.4900);
\draw[black!35, line width=0.4pt] (1.3500,0.5942) -- (1.3500,0.4900);
\node[font=\tiny, text=black!70, rotate=90, anchor=east] at (1.5750,-0.19) {scale};
\draw[tzcode, line width=1.3pt] (1.8300,0.4900) -- (2.2200,0.4900);
\draw[black!35, line width=0.4pt] (1.8000,0.4900) -- (1.8000,0.4900);
\node[font=\tiny, text=black!70, rotate=90, anchor=east] at (2.0250,-0.19) {code};
\node[font=\tiny, text=black!70, anchor=south] at (1.5750,0.6442) {$-75\%$};
\end{scope}
\end{tikzpicture}
\vspace{-2pt}
\caption{\method{} on one group of two weights with correlated inputs.
Ellipses are error level sets around $\mathbf{w}$ ($\star$); dots are grid
points. (a)~The host codes are optimal on the host grid. (b)~Refitting
the scale moves the grid. (c)~On the new grid, one Babai step moves both
codes (Proposition~\ref{prop:grid-staleness}). (d)~The objective stays at or below the host
level.}
\label{fig:teaser}
\end{figure*}

We propose \method{} (\methodfull{}), a plug-in refinement for any
group-wise quantizer. Starting from the host's scales, codes, and zero
points, it minimizes the layer's output reconstruction error by
alternating two updates: a joint least-squares fit of
all group scales of an output channel, and a bounded Babai
(nearest-plane) proposal that can move many codes of a group at once
on the current grid. The objective is a bilinear box-constrained
mixed-integer least-squares (BBMILS) problem; the bit width, groups,
zero points, and inference cost stay unchanged.

Our contributions are:
\begin{itemize}
    \item We address grid staleness, cross-group coupling, and
    single-code stationarity together in one BBMILS refinement problem
    that any group-wise quantizer can enter as a host.

    \item We give a backpropagation-free alternating solver and show why
    the grid must move: codes that no search on the host grid can improve
    may still be improved after a scale refit
    (Proposition~\ref{prop:grid-staleness}), and with exact scale solves
    does not increase the layer's calibration objective
    (Theorem~\ref{thm:descent}). On full models, alternation beats a scale refit
    alone in all 16 comparisons, while code updates alone on the host grid
    raise perplexity in 5 of 16 (Section~\ref{sec:ablations}).

    \item Across six models, four hosts, and two bit widths,
    \method{} lowers perplexity in 90 of 96 comparisons and cuts RTN's
    W3A16 WikiText-2 perplexity by up to 36\%. It raises mean
    multiple-choice accuracy in 23 of 24 configurations on three models
    and also improves QEP, QuaRot, and OJBKQ outputs at under a minute per
    7B block (Section~\ref{sec:experiments}).
\end{itemize}

\section{Related Work}
\label{sec:related_work}

\paragraph{Layer-wise PTQ.}
GPTQ and OBC reduce layer-wise reconstruction error with second-order
sequential quantization~\citep{frantar-gptq,frantar2022obc}; AWQ uses
activation-aware scaling and OmniQuant learns clipping and equivalent
transformations~\citep{lin2023awq,OmniQuant}. QEP corrects the target
for errors from preceding layers~\citep{arai2025qep}, and QuaRot and
SpinQuant rotate weights and activations~\citep{ashkboos2024quarot,liu2024spinquant}.
AdaRound and BRECQ learn rounding by layer or block
reconstruction~\citep{nagel2020adaround,li2021brecq}, and
SignRound/AutoRound tunes rounding offsets and clipping with signed
gradient descent~\citep{cheng2024signround}.
These methods decide the target, the representation, or the initial
codes; \method{} starts from their output and re-optimizes the grid and
the codes together.

\paragraph{Integer search on a lattice.}
GPTQ, run from the last to the first dimension, is Babai's
nearest-plane algorithm~\citep{babai1986lovasz} on a lattice defined by
the Hessian~\citep{chen2025geometry}; QuIP relates its LDLQ rounding
to GPTQ, and QuIP\# quantizes to lattice
codebooks~\citep{chee2023quip,tseng2024quipsharp}. OJBKQ applies Babai--Klein
decoding with precomputed scales~\citep{OJBKQ}; Qronos and CoreQ improve
sequential rounding with error correction~\citep{Qronos2026,cha2026coreq};
ADMM-Q alternates continuous updates with discrete projections and
notes that the grid can become stale, allowing one grid
refresh~\citep{lucas2026admmq}. The other methods decode on a grid that
is fixed during the search. \method{} instead refits all group scales
of a channel jointly by least squares at every iteration, so each new
grid opens new code improvements (Proposition~\ref{prop:grid-staleness}).

\paragraph{Joint scale--code optimization.}
Alternating between codes and their coefficients goes back to
multi-bit quantization of recurrent networks~\citep{xu2018alternating};
HQQ fixes the scales and optimizes the zero points without calibration
data by half-quadratic splitting~\citep{badri2023hqq}. For LLMs, QuantEase runs
coordinate descent over codes on the layer reconstruction objective
with the grid held fixed~\citep{behdin2023quantease}.
\method{} instead fits all group scales of a channel jointly under
the output objective and moves a group's codes as a block on each
refitted grid via Babai proposals.

\section{Methodology}
\label{sec:method}

We first write refinement of a quantized layer as one reconstruction
problem over scales and codes (Sections~\ref{sec:preliminaries}
and~\ref{sec:bbmils}). The problem mixes a continuous grid with discrete
assignments, so we solve it by alternating two updates, each designed
around one of the observations in Section~\ref{sec:introduction}: a joint
scale fit that accounts for coupled groups, and a Babai proposal that
moves several codes at once. Repeating the pair lets the grid follow the
codes and the codes follow the grid (Section~\ref{sec:jarq_solver}).

\subsection{Preliminaries and Problem Formulation}
\label{sec:preliminaries}

Let $\mathbf{W}=[\mathbf{w}_1,\ldots,\mathbf{w}_{d_{\mathrm{out}}}]
\in\mathbb{R}^{d_{\mathrm{in}}\times d_{\mathrm{out}}}$ denote the original weight matrix of a linear layer, and let $\mathbf{x}^{\top}\in\mathbb{R}^{1\times d_{\mathrm{in}}}$ denote an input row vector to this layer.
Partition each column $\mathbf{w}_j$, $j=1,\ldots,d_{\mathrm{out}}$,
into $G$ groups
$\mathbf{w}_{i,j}\in\mathbb{R}^{d_i}$, $i=1,\ldots,G$,
where $\sum_{i=1}^{G}d_i=d_{\mathrm{in}}$.
Groups typically have equal size, in which case
$d_i=d_{\mathrm{in}}/G$ for all $i$.
For bit width $b$, define the integer set
$\mathbb{B}=\{0,1,\ldots,2^b-1\}$.
A group-wise uniform quantizer represents each weight column $\w_j$ by
\begin{equation}
\widehat{\mathbf{w}}_j
=
\begin{bmatrix}
s_{1,j}(\mathbf{q}_{1,j}-z_{1,j}\mathbf{1})\\
\vdots\\
s_{\sG,j}(\mathbf{q}_{\sG,j}-z_{\sG,j}\mathbf{1})
\end{bmatrix}
=\mathbf{D}_{\mathbf{s}_j}(\mathbf{q}_j-\mathbf{z}_j), \ \ 
\quad \mathbf{q}_{i,j}\in\mathbb{B}^{d_i},
\label{eq:quantized_group}
\end{equation}
where
$$
\mathbf{s}_j= 
\begin{bmatrix} 
s_{1,j} \\ \vdots \\ s_{\sG,j}
\end{bmatrix}, 
\quad 
\mathbf{D}_{\mathbf{s}_j} =
\begin{bmatrix} 
s_{1,j}\mathbf{I}_{d_1} \\ & \!\! \ddots \!\! \\ & &  s_{\sG,j}\mathbf{I}_{d_{\sG}}
\end{bmatrix}, 
\quad
\mathbf{q}_j= 
\begin{bmatrix} 
\mathbf{q}_{1,j} \\ \vdots \\ \mathbf{q}_{\sG,j} 
\end{bmatrix},
\quad 
\mathbf{z}_j =
\begin{bmatrix}
z_{1,j} \mathbf{1}_{d_1} \\ \vdots \\ z_{\sG,j} \mathbf{1}_{d_{\sG}}
\end{bmatrix}.
$$
The scale sets the grid spacing, the zero point sets its
offset, and the integer code selects a grid value. Collect these variables
as $\mathbf{S}=[\mathbf{s}_1,\ldots,\mathbf{s}_{d_{\mathrm{out}}}]
\in\mathbb{R}^{G\times d_{\mathrm{out}}}$,
$\mathbf{Z}=(z_{i,j})\in\mathbb{Z}^{G\times d_{\mathrm{out}}}$, and $\mathbf{Q}=[\mathbf{q}_1,\ldots,\mathbf{q}_{d_{\mathrm{out}}}]
\in\mathbb{B}^{d_{\mathrm{in}}\times d_{\mathrm{out}}}$.
For later use, we write $\hbW=[\hbw_1,\ldots,\hbw_\dout]
\in \Rbb^{\din \times \dout}$.

Let $\mathbf{X}=[\mathbf{x}_1,\ldots,\mathbf{x}_n]^\top
\in\mathbb{R}^{n\times d_{\mathrm{in}}}$ stack the calibration input
vectors to this layer in the original model as rows.
Let $\widetilde{\mathbf{X}}$ stack the corresponding inputs obtained
using the quantized preceding layers in the same way.
They coincide for the first layer and drift apart in later layers.

Given an existing quantized layer, we refine its scales $\mathbf{S}$ and
integer codes $\mathbf{Q}$ while fixing the bit width, groups, and zero
points $\mathbf{Z}$. The original weight matrix $\mathbf{W}$ remains the reference.
The goal is to choose scales and codes together so that the quantized
layer preserves the reference response.

\subsection{Objective}
\label{sec:bbmils}

We preserve the layer's response
$\mathbf{Y}^{\star}=\widetilde{\mathbf{X}}\mathbf{W}$ to the inputs it
receives in the quantized model by solving
\begin{equation}
\begin{aligned}
& \min_{\mathbf{S},\mathbf{Q}} 
\mathcal{F}(\mathbf{S},\mathbf{Q})
=\|\widetilde{\mathbf{X}}\widehat{\mathbf{W}}
       -\widetilde{\mathbf{X}}\mathbf{W}\|_F^2
  +\nu^2\|\widehat{\mathbf{W}}-\mathbf{W}\|_F^2,\\
& \text{ s.t.}\quad
\mathbf{S}\in\mathbb{R}^{G\times d_{\mathrm{out}}},
\quad \mathbf{Q}\in\mathbb{B}^{d_{\mathrm{in}}\times d_{\mathrm{out}}},
\end{aligned}
\label{eq:adaptive_grid_problem}
\end{equation}
where $\widehat{\mathbf{W}}$ is given by \eqref{eq:quantized_group}.
The output term compares $\widetilde{\mathbf{X}}\widehat{\mathbf{W}}$ with $\widetilde{\mathbf{X}}\mathbf{W}$,
so the target needs no full-precision activations. The output error accounts for correlations
between input coordinates. The weight penalty ($\nu>0$)
restrains changes in directions weakly constrained by the calibration set.

Define the augmented target and design matrices
\begin{equation}
\mathbf{Y}=
\begin{bmatrix}\mathbf{Y}^{\star}\\ \nu\mathbf{W}\end{bmatrix}
\in\mathbb{R}^{m\times d_{\mathrm{out}}},
\qquad
\mathbf{A}=
\begin{bmatrix}\widetilde{\mathbf{X}}\\ \nu\mathbf{I}\end{bmatrix}
\in\mathbb{R}^{m\times d_{\mathrm{in}}},
\qquad m=n+d_{\mathrm{in}}.
\label{eq:strict_target}
\end{equation}
Thus $\mathbf{Y}$ contains the reference outputs followed by the weighted
reference weights; it is constructed from the calibration inputs and
the original layer weights. Writing
$\mathbf{Y}=[\mathbf{y}_1,\ldots,\mathbf{y}_{d_{\mathrm{out}}}]$,
the objective becomes
\begin{equation}
\mathcal{F}(\mathbf{S},\mathbf{Q})
=\|\mathbf{Y}-\mathbf{A}\widehat{\mathbf{W}}\|_F^2
=\sum_{j=1}^{d_{\mathrm{out}}}
\|\mathbf{y}_j-\mathbf{A}\hbw_j\|_2^2
=\sum_{j=1}^{d_{\mathrm{out}}}
\|\mathbf{y}_j-\mathbf{A}\mathbf{D}_{\mathbf{s}_j}
  (\mathbf{q}_j-\mathbf{z}_j)\|_2^2.
\label{eq:bbmils_columns}
\end{equation}
Both $\mathbf{A}$ and $\mathbf{Y}$ stay fixed during a layer's refinement.
The target contains no correction for errors of preceding layers;
mixing in the full-precision response $\mathbf{X}\mathbf{W}$ yields QEP's
propagation-corrected target~\citep{arai2025qep}, for which all results below hold unchanged
(Appendix~\ref{sec:target_dependence}).

Because $\widehat{\mathbf{W}}$ multiplies real scales by bounded integer
codes, \eqref{eq:bbmils_columns} is a bilinear box-constrained
mixed-integer least-squares (BBMILS) problem; fixing the scales
leaves a closest-vector problem.

\subsection{Alternating Update}
\label{sec:jarq_solver}
\label{sec:complete_algorithm}

We now approximately minimize \eqref{eq:adaptive_grid_problem}.
Scales determine the grid, and codes determine the assignments to it.
A host settles the grid and rounds the codes onto it; but once the codes
move, a different grid fits the output best, and on that grid different
codes become optimal. \method{} therefore alternates the two: at
iteration $k$, it refits $\mathbf{S}$ for the codes
$\mathbf{Q}^{(k-1)}$, then updates $\mathbf{Q}$ on the new grid
$\mathbf{S}^{(k)}$. The objective in \eqref{eq:bbmils_columns} separates
over output columns, so each column can be refined independently.

We call the PTQ method supplying the initial quantized layer the
\emph{host}. It may use rounding, error compensation, learned clipping,
or another procedure. We require only its group-wise uniform
representation $(\mathbf{S}^{(0)},\mathbf{Q}^{(0)},\mathbf{Z})$,
and initialize directly from these scales, codes, and zero points. The two updates are
described below, followed by their core loop in
Algorithm~\ref{alg:jarq}. Alternation matters even when the host's
codes are already optimal for the host's grid, as the following
construction shows:

\begin{proposition}[Refitting the grid unlocks descent]
\label{prop:grid-staleness}
There exist a group problem of the form \eqref{eq:group_bils} and a host
state $(s_0,\mathbf{q}_0)$ such that $\mathbf{q}_0$ minimizes the error
over $\mathbb{B}^{d_i}$ at scale $s_0$, yet after a least-squares scale
refit, the bounded Babai
proposal changes two codes and strictly lowers the error.
\end{proposition}
Moving the grid causes the descent; Figure~\ref{fig:teaser} draws an
instance, verified in Appendix~\ref{app:grid-staleness-proof}.

\subsubsection{Updating Scales Given Codes}
\label{sec:scale_update}

Groups are stored separately, but their outputs are not independent:
correlated inputs make one group's error depend on the others' scales.
The scale update fits all groups of a channel jointly.

Partition the design matrix consistently with the weight groups:
$\mathbf{A}=[\mathbf{A}_1,\ldots,\mathbf{A}_{\sG}]$, where
$\mathbf{A}_i\in\mathbb{R}^{m\times d_i}$.
At iteration $k$, in the objective $\mathcal{F}(\boldS,\Q)$
we take $\Q=\Q^{(k-1)}$, which is obtained at iteration $k-1$.
Then, from \eqref{eq:quantized_group}, we can write
$$
\A \hbw_j =\mathbf{B}_j^{(k-1)}\mathbf{s}_j, 
$$
where
\begin{equation}
\mathbf{B}_j^{(k-1)}
=\left[
\mathbf{A}_1(\mathbf{q}_{1,j}^{(k-1)}-z_{1,j}\mathbf{1}_{d_1}),
\ldots,
\mathbf{A}_{\sG}(\mathbf{q}_{\sG,j}^{(k-1)}-z_{\sG,j}\mathbf{1}_{d_\sG})
\right]\in\mathbb{R}^{m\times G}.
\label{eq:scale_design}
\end{equation}
From \eqref{eq:bbmils_columns}, we see that the optimal $\s_j$,
to be denoted by $\mathbf{s}_j^{(k)}$, should be the solution 
of the real least-squares problem:
$$ 
\min_{\mathbf{s}\in\mathbb{R}^{G}}
\|\mathbf{y}_j-\mathbf{B}_j^{(k-1)}\mathbf{s}\|_2^2,
$$
Equivalently,  $\mathbf{s}_j^{(k)}$ satisfies the normal equations:
\beq
(\mathbf{B}_j^{(k-1)})^\top\mathbf{B}_j^{(k-1)}
\mathbf{s}_j^{(k)}
=(\mathbf{B}_j^{(k-1)})^\top\mathbf{y}_j.
\label{eq:scale_normal_equations} 
\eeq 
Thus all $G$ scales of $\hbw_j$ are fitted jointly. 
We solve the normal equations by Cholesky factorization. 
If the Gram matrix  $(\mathbf{B}_j^{(k-1)})^\top\mathbf{B}_j^{(k-1)}$
is nearly singular, we add $\lambda_j\I$ to it with
$\lambda_j=10^{-4}$; otherwise $\lambda_j=0$
(Appendix~\ref{app:numerical-updates}). Scales are
unconstrained in sign and are evaluated as signed values;
Appendix~\ref{app:solver-details} maps negative scales to an equivalent
positive form for formats that require one.

Fitting each group's scale separately ignores the off-diagonal terms of
the Gram matrix and
leaves an error gap that vanishes when the groups' column spaces are
orthogonal and is bounded by a quantity quadratic in their coherence
(Propositions~\ref{prop:joint-scale-gap} and~\ref{prop:independent-gap});
Section~\ref{sec:ablations} measures this gap. The joint solve is cheap because
$G$ is small: a 4,096-wide input with groups of 128 weights has $G=32$, so
the $G\times G$ system costs far less than forming $\mathbf{B}_j$
(Appendix~\ref{app:complexity}).

\subsubsection{Updating Integer Codes Given Scales}
\label{sec:code_update}

On a fixed grid, the best change to a group's codes often moves several
codes together, because correlated inputs let one code's error offset
another's. Single-code search stops at such points
(Proposition~\ref{prop:multicode-descent}); the code update below
proposes all codes of a group at once from one triangular solve.

Hold $\mathbf{s}_j^{(k)}$ fixed and visit groups $i=1,\ldots,G$.
The earlier groups in the sweep already have their new codes; the later
groups still have their previous codes. The target for group $i$ is
therefore
\begin{equation}
\begin{aligned}
\mathbf{y}_j^{(i,k)}
=\mathbf{y}_j
&-\sum_{g<i}s_{g,j}^{(k)}
 \mathbf{A}_g(\mathbf{q}_{g,j}^{(k)}-z_{g,j}\mathbf{1}_{d_g}) -\sum_{g>i}s_{g,j}^{(k)}
 \mathbf{A}_g(\mathbf{q}_{g,j}^{(k-1)}-z_{g,j}\mathbf{1}_{d_g}).
\end{aligned}
\label{eq:conditional_residual}
\end{equation}
This subtracts the latest contributions of all other groups.
Ideally, the new group codes would solve
\begin{equation}
\min_{\mathbf{q}\in\mathbb{B}^{d_i}}
\|\mathbf{y}_j^{(i,k)}
 -s_{i,j}^{(k)}\mathbf{A}_i(\mathbf{q}-z_{i,j}\mathbf{1}_{d_i})\|_2^2.
\label{eq:group_bils}
\end{equation}
Problem \eqref{eq:group_bils} is a closest-vector problem over a
bounded set of lattice points. Such problems are NP-hard in
general~\citep{vanemdeboas1981another,micciancio2002complexity}, and
searching the entire integer set is impractical for typical group sizes.
We instead compute one bounded Babai nearest-plane
proposal~\citep{babai1986lovasz} using a triangular system from the
group's QR factor.

For each group, precompute a thin QR factorization, without an additional
column permutation:
\begin{equation}
\mathbf{A}_i=\mathbf{U}_i\mathbf{R}_i,
\qquad \mathbf{U}_i^\top\mathbf{U}_i=\mathbf{I},
\qquad \mathbf{R}_i\in\mathbb{R}^{d_i\times d_i}\ \text{upper triangular}.
\label{eq:group_qr}
\end{equation}
Since $\nu>0$, $\mathbf{A}$ has full column rank, so each $\mathbf{R}_i$ is
nonsingular. These factors are reused across output columns and iterations.
For a nonzero scale, let
$\mathbf{v}=\mathbf{U}_i^\top\mathbf{y}_j^{(i,k)}/s_{i,j}^{(k)}$.
The proposal $\mathbf{q}^{\mathrm{cand}}$ is computed backward:
\begin{equation}
q_\ell^{\mathrm{cand}}
=\operatorname{clip}_{\mathbb{B}}\!\left(
\operatorname{round}\!\left[
z_{i,j}+
\frac{v_\ell-\sum_{h=\ell+1}^{d_i}(\mathbf{R}_i)_{\ell h}
(q_h^{\mathrm{cand}}-z_{i,j})}
{(\mathbf{R}_i)_{\ell\ell}}
\right]\right),
\quad \ell=d_i,\ldots,1.
\label{eq:babai_update}
\end{equation}
Each step subtracts the contribution of the already chosen integer codes
before rounding and clipping. Thus, codes with larger index, chosen first, shape the choices for
smaller indices. We accept the proposal only if it gives a finite,
strictly smaller residual in \eqref{eq:group_bils}; otherwise we keep
$\mathbf{q}_{i,j}^{(k-1)}$. The resulting codes define
$\mathbf{B}_j^{(k)}$ for the next scale fit.

\begin{proposition}[Descent beyond single-code updates]
\label{prop:multicode-descent}
There exists a box-constrained group problem of the form
\eqref{eq:group_bils}, with its scale optimal for the current codes,
such that every single-code change has strictly larger error even
after optimally refitting the scale, while one bounded Babai proposal
strictly reduces the error at the current scale.
\end{proposition}
Correlated input coordinates can therefore make coordinated code
changes beneficial beyond the reach of single-code moves with
conditional scale refitting (Appendix~\ref{app:multicode-proof}).
Without clipping and under a uniform model for rounding errors, Babai's
expected residual is also no larger than that of rounding the
unconstrained solution, and lower by $s_{i,j}^2/12$ times the off-diagonal energy of $\mathbf{R}_i$
(Proposition~\ref{prop:babai-vs-rounding}).
The update applies the nearest-plane recursion that underlies
GPTQ~\citep{chen2025geometry} to one group on the current grid;
Remark~\ref{rem:gptq-babai} discusses the relation and the differences.

\begin{figure}[t]
\noindent
\begin{minipage}[t]{0.43\textwidth}
\vspace{0pt}
\textbf{Layerwise procedure.}
Algorithm~\ref{alg:jarq} summarizes the two updates with a maximum
budget of $K$ iterations; all main experiments use $K=3$.
For efficient computation, factor $\mathbf{A}=\mathbf{U}\mathbf{R}$
and set $\overline{\mathbf{Y}}=\mathbf{U}^{\top}\mathbf{Y}$.
The reduced objective is
\begin{equation}
\mathcal{J}(\mathbf{S},\mathbf{Q})
=\|\overline{\mathbf{Y}}-\mathbf{R}\widehat{\mathbf{W}}\|_F^2.
\label{eq:reduced_problem}
\end{equation}
It differs from $\mathcal{F}$ by a constant, which is zero for the main
target $\mathbf{Y}=\mathbf{A}\mathbf{W}$
(Appendix~\ref{app:solver-details}). We compute the updates with
$(\mathbf{R},\overline{\mathbf{Y}})$ in place of
$(\mathbf{A},\mathbf{Y})$, caching all QR factors and retaining the
constant for stopping. The terminal scales and codes define the
installed weights, whose outputs supply the next layer's calibration
inputs. Each channel stops once the objective after consecutive scale
updates decreases by at most $\tau=10^{-5}$, relative or absolute
(Appendix~\ref{app:numerical-updates}).
The computational cost is analyzed in Appendix~\ref{app:complexity}.
\end{minipage}\hfill
\begin{minipage}[t]{0.545\textwidth}
\vspace{0pt}
\setlength{\intextsep}{0pt}

\input{sections/algorithm}
\end{minipage}
\end{figure}

\textbf{Descent guarantee.}
Write $\mathcal{F}_j(\mathbf{s},\mathbf{q})$ for the $j$th term of
\eqref{eq:bbmils_columns} and
$\phi_j(\mathbf{q})=\min_{\mathbf{s}}\mathcal{F}_j(\mathbf{s},\mathbf{q})$.
We call $(\mathbf{s},\mathbf{q})$ \emph{scale-optimal and Babai-stable}
if $\mathbf{s}$ minimizes $\mathcal{F}_j(\cdot,\mathbf{q})$ and, for every
group with $|s_i|>\epsilon_{\rm fp}$, the bounded Babai proposal computed
at $(\mathbf{s},\mathbf{q})$ does not strictly lower $\mathcal{F}_j$ any further.

\begin{theorem}[Monotone descent and finite termination]
\label{thm:descent}
Fix an output column $j$ and assume that every scale update in
Algorithm~\ref{alg:jarq} is an exact least-squares solve
($\lambda_j=0$). In exact arithmetic,
(i) $\mathcal{F}_j$ does not increase at any scale update or code update;
(ii) $\phi_j$ strictly decreases over every iteration that changes a code;
(iii) with an unlimited budget and no early stopping, codes change in
finitely many iterations, and after the first
iteration without an accepted proposal the state is a fixed point that
is scale-optimal and Babai-stable.
\end{theorem}

\begin{corollary}[No worse than the host]
\label{cor:host}
In exact arithmetic, if every scale update in the layer is an exact
least-squares solve ($\lambda_j=0$),
the returned state has an objective no larger than that of the host state,
$\mathcal{F}(\mathbf{S},\mathbf{Q})\leq
\mathcal{F}(\mathbf{S}^{(0)},\mathbf{Q}^{(0)})$, for any budget $K$ and any stopping rule.
\end{corollary}
The guarantee is per layer: it compares both states on the calibration
objective with the same inputs $\widetilde{\mathbf{X}}$. Appendix~\ref{app:descent-proof} gives the proofs and the
ridge case. At a Babai-stable point, a group whose proposal needs no
clipping is within $s_{i,j}^2\sum_\ell(\mathbf{R}_i)_{\ell\ell}^2/4$ of
its best real-valued fit (Proposition~\ref{prop:stable-quality}).

\section{Experiments}
\label{sec:experiments}

We test perplexity across six models, four hosts, and two bit widths,
then evaluate downstream accuracy and examine the sources of improvement,
including layer-wise errors and cost.

\subsection{Setup}
\label{sec:experimental_setup}

We quantize Llama-2-7B/13B~\citep{touvron2023llama2}, Llama-3-8B~\citep{grattafiori2024llama3}, Qwen2.5-7B~\citep{qwen2.5}, and
Qwen3-4B/8B~\citep{qwen3} to
W3A16 and W4A16 with groups of 128 weights, using 128 WikiText-2
training sequences of 2,048 tokens. The host is RTN,
GPTQ~\citep{frantar-gptq}, OmniQuant~\citep{OmniQuant} (weight clipping only), or
AWQ~\citep{lin2023awq}; Base is its quantized model and $+$Ours is
\method{} applied to it with the same bit width, groups, and zero points.
All refined models use $K=3$ iterations with early stopping and $\nu=0.6$.
We report WikiText-2 (WT2)~\citep{merity2017pointer} and C4~\citep{raffel2020exploring} perplexity, five multiple-choice tasks,
and generative accuracy on GSM8K and MATH-500 (Appendix~\ref{app:experimental-details}).

\subsection{Main Results}
\label{sec:fullmodel_results}
\label{sec:downstream_results}

\paragraph{Perplexity.}
Table~\ref{tab:main_ppl_matrix} shows lower perplexity in 46 of 48
WT2 comparisons and 44 of 48 C4 comparisons.
The gains are largest for RTN, the only host without calibration data. At W3A16,
refining RTN cuts WT2 perplexity by 36.21\% on Qwen2.5-7B
(11.95$\to$7.63) and by 31.78\% on Llama-3-8B
(12.05$\to$8.22). Part of the RTN gain comes from the calibration data that
refinement adds, but calibrated hosts still leave room: on Qwen3-8B at
W3A16, OmniQuant improves by 8.02\% (11.46$\to$10.55)
and GPTQ by 1.93\% (10.75$\to$10.54). The same holds
at 32B: across OmniQuant and AWQ on Qwen2.5-32B and OmniQuant on
Qwen3-32B, \method{} lowers perplexity in 11 of 12 comparisons, with
the largest gains at three bits (up to 3.0\% on WT2;
Figure~\ref{fig:robustness}c and Appendix~\ref{app:supplementary-results}).

\begin{table*}[t]
\centering
\caption{Perplexity before (Base) and after ($+$Ours) refinement with
$K=3$ (lower is better). Small numbers give the relative change of
$+$Ours over Base in percent (green: lower, red: higher).}
\label{tab:main_ppl_matrix}
\scriptsize
\setlength{\tabcolsep}{1.2pt}
\renewcommand{\arraystretch}{1.05}
\resizebox{\textwidth}{!}{\begin{tabular}{@{}lrrrrrrrrrrrr@{}}
\toprule
 & \multicolumn{2}{c}{Llama-2-7B} & \multicolumn{2}{c}{Llama-2-13B} & \multicolumn{2}{c}{Llama-3-8B} & \multicolumn{2}{c}{Qwen2.5-7B} & \multicolumn{2}{c}{Qwen3-4B} & \multicolumn{2}{c}{Qwen3-8B} \\
\cmidrule(lr){2-3}\cmidrule(lr){4-5}\cmidrule(lr){6-7}\cmidrule(lr){8-9}\cmidrule(lr){10-11}\cmidrule(lr){12-13}
Method & \ooalign{\hphantom{5.472}\cr\hfil WT2\hfil\cr}\hspace{0.6pt}\makebox[\ppldw]{} & \ooalign{\hphantom{6.973}\cr\hfil C4\hfil\cr}\hspace{0.6pt}\makebox[\ppldw]{} & \ooalign{\hphantom{4.884}\cr\hfil WT2\hfil\cr}\hspace{0.6pt}\makebox[\ppldw]{} & \ooalign{\hphantom{6.468}\cr\hfil C4\hfil\cr}\hspace{0.6pt}\makebox[\ppldw]{} & \ooalign{\hphantom{12.048}\cr\hfil WT2\hfil\cr}\hspace{0.6pt}\makebox[\ppldw]{} & \ooalign{\hphantom{16.444}\cr\hfil C4\hfil\cr}\hspace{0.6pt}\makebox[\ppldw]{} & \ooalign{\hphantom{11.954}\cr\hfil WT2\hfil\cr}\hspace{0.6pt}\makebox[\ppldw]{} & \ooalign{\hphantom{10.442}\cr\hfil C4\hfil\cr}\hspace{0.6pt}\makebox[\ppldw]{} & \ooalign{\hphantom{13.638}\cr\hfil WT2\hfil\cr}\hspace{0.6pt}\makebox[\ppldw]{} & \ooalign{\hphantom{16.626}\cr\hfil C4\hfil\cr}\hspace{0.6pt}\makebox[\ppldw]{} & \ooalign{\hphantom{13.456}\cr\hfil WT2\hfil\cr}\hspace{0.6pt}\makebox[\ppldw]{} & \ooalign{\hphantom{13.290}\cr\hfil C4\hfil\cr}\hspace{0.6pt}\makebox[\ppldw]{} \\
\midrule
\multicolumn{13}{@{}l}{\textit{W3A16}} \\
\rowcolor{jarqFP16} FP16 & 5.472\hspace{0.6pt}\makebox[\ppldw][r]{} & 6.973\hspace{0.6pt}\makebox[\ppldw][r]{} & 4.884\hspace{0.6pt}\makebox[\ppldw][r]{} & 6.468\hspace{0.6pt}\makebox[\ppldw][r]{} & 6.136\hspace{0.6pt}\makebox[\ppldw][r]{} & 8.881\hspace{0.6pt}\makebox[\ppldw][r]{} & 6.848\hspace{0.6pt}\makebox[\ppldw][r]{} & 10.442\hspace{0.6pt}\makebox[\ppldw][r]{} & 13.638\hspace{0.6pt}\makebox[\ppldw][r]{} & 16.626\hspace{0.6pt}\makebox[\ppldw][r]{} & 9.715\hspace{0.6pt}\makebox[\ppldw][r]{} & 13.290\hspace{0.6pt}\makebox[\ppldw][r]{} \\
RTN & 6.663\hspace{0.6pt}\makebox[\ppldw][r]{} & 8.405\hspace{0.6pt}\makebox[\ppldw][r]{} & 5.521\hspace{0.6pt}\makebox[\ppldw][r]{} & 7.179\hspace{0.6pt}\makebox[\ppldw][r]{} & 12.048\hspace{0.6pt}\makebox[\ppldw][r]{} & 16.444\hspace{0.6pt}\makebox[\ppldw][r]{} & 11.954\hspace{0.6pt}\makebox[\ppldw][r]{} & 16.054\hspace{0.6pt}\makebox[\ppldw][r]{} & 22.447\hspace{0.6pt}\makebox[\ppldw][r]{} & 27.013\hspace{0.6pt}\makebox[\ppldw][r]{} & 13.456\hspace{0.6pt}\makebox[\ppldw][r]{} & 17.582\hspace{0.6pt}\makebox[\ppldw][r]{} \\
\cellcolor{jarqRefined}\enspace$+$Ours & \cellcolor{jarqRefined}6.204\hspace{0.6pt}\makebox[\ppldw][r]{\pplgood{$-$6.88}} & \cellcolor{jarqRefined}8.217\hspace{0.6pt}\makebox[\ppldw][r]{\pplgood{$-$2.24}} & \cellcolor{jarqRefined}5.270\hspace{0.6pt}\makebox[\ppldw][r]{\pplgood{$-$4.54}} & \cellcolor{jarqRefined}7.158\hspace{0.6pt}\makebox[\ppldw][r]{\pplgood{$-$0.29}} & \cellcolor{jarqRefined}8.219\hspace{0.6pt}\makebox[\ppldw][r]{\pplgood{$-$31.8}} & \cellcolor{jarqRefined}12.667\hspace{0.6pt}\makebox[\ppldw][r]{\pplgood{$-$23.0}} & \cellcolor{jarqRefined}7.625\hspace{0.6pt}\makebox[\ppldw][r]{\pplgood{$-$36.2}} & \cellcolor{jarqRefined}11.916\hspace{0.6pt}\makebox[\ppldw][r]{\pplgood{$-$25.8}} & \cellcolor{jarqRefined}16.194\hspace{0.6pt}\makebox[\ppldw][r]{\pplgood{$-$27.9}} & \cellcolor{jarqRefined}21.625\hspace{0.6pt}\makebox[\ppldw][r]{\pplgood{$-$19.9}} & \cellcolor{jarqRefined}10.712\hspace{0.6pt}\makebox[\ppldw][r]{\pplgood{$-$20.4}} & \cellcolor{jarqRefined}15.103\hspace{0.6pt}\makebox[\ppldw][r]{\pplgood{$-$14.1}} \\
GPTQ & 6.083\hspace{0.6pt}\makebox[\ppldw][r]{} & 7.990\hspace{0.6pt}\makebox[\ppldw][r]{} & 5.248\hspace{0.6pt}\makebox[\ppldw][r]{} & 7.034\hspace{0.6pt}\makebox[\ppldw][r]{} & 7.618\hspace{0.6pt}\makebox[\ppldw][r]{} & 13.242\hspace{0.6pt}\makebox[\ppldw][r]{} & 7.665\hspace{0.6pt}\makebox[\ppldw][r]{} & 11.677\hspace{0.6pt}\makebox[\ppldw][r]{} & 15.153\hspace{0.6pt}\makebox[\ppldw][r]{} & 18.623\hspace{0.6pt}\makebox[\ppldw][r]{} & 10.749\hspace{0.6pt}\makebox[\ppldw][r]{} & 14.720\hspace{0.6pt}\makebox[\ppldw][r]{} \\
\cellcolor{jarqRefined}\enspace$+$Ours & \cellcolor{jarqRefined}6.014\hspace{0.6pt}\makebox[\ppldw][r]{\pplgood{$-$1.13}} & \cellcolor{jarqRefined}7.979\hspace{0.6pt}\makebox[\ppldw][r]{\pplgood{$-$0.14}} & \cellcolor{jarqRefined}5.248\hspace{0.6pt}\makebox[\ppldw][r]{\pplgood{$-$0.01}} & \cellcolor{jarqRefined}7.037\hspace{0.6pt}\makebox[\ppldw][r]{\pplbad{$+$0.04}} & \cellcolor{jarqRefined}7.649\hspace{0.6pt}\makebox[\ppldw][r]{\pplbad{$+$0.41}} & \cellcolor{jarqRefined}12.011\hspace{0.6pt}\makebox[\ppldw][r]{\pplgood{$-$9.30}} & \cellcolor{jarqRefined}7.524\hspace{0.6pt}\makebox[\ppldw][r]{\pplgood{$-$1.85}} & \cellcolor{jarqRefined}11.670\hspace{0.6pt}\makebox[\ppldw][r]{\pplgood{$-$0.05}} & \cellcolor{jarqRefined}15.089\hspace{0.6pt}\makebox[\ppldw][r]{\pplgood{$-$0.42}} & \cellcolor{jarqRefined}18.791\hspace{0.6pt}\makebox[\ppldw][r]{\pplbad{$+$0.90}} & \cellcolor{jarqRefined}10.542\hspace{0.6pt}\makebox[\ppldw][r]{\pplgood{$-$1.93}} & \cellcolor{jarqRefined}14.666\hspace{0.6pt}\makebox[\ppldw][r]{\pplgood{$-$0.36}} \\
OmniQuant & 6.095\hspace{0.6pt}\makebox[\ppldw][r]{} & 7.812\hspace{0.6pt}\makebox[\ppldw][r]{} & 5.322\hspace{0.6pt}\makebox[\ppldw][r]{} & 7.025\hspace{0.6pt}\makebox[\ppldw][r]{} & 8.768\hspace{0.6pt}\makebox[\ppldw][r]{} & 12.476\hspace{0.6pt}\makebox[\ppldw][r]{} & 7.779\hspace{0.6pt}\makebox[\ppldw][r]{} & 11.785\hspace{0.6pt}\makebox[\ppldw][r]{} & 17.221\hspace{0.6pt}\makebox[\ppldw][r]{} & 20.902\hspace{0.6pt}\makebox[\ppldw][r]{} & 11.465\hspace{0.6pt}\makebox[\ppldw][r]{} & 15.454\hspace{0.6pt}\makebox[\ppldw][r]{} \\
\cellcolor{jarqRefined}\enspace$+$Ours & \cellcolor{jarqRefined}6.045\hspace{0.6pt}\makebox[\ppldw][r]{\pplgood{$-$0.82}} & \cellcolor{jarqRefined}7.795\hspace{0.6pt}\makebox[\ppldw][r]{\pplgood{$-$0.22}} & \cellcolor{jarqRefined}5.253\hspace{0.6pt}\makebox[\ppldw][r]{\pplgood{$-$1.30}} & \cellcolor{jarqRefined}7.010\hspace{0.6pt}\makebox[\ppldw][r]{\pplgood{$-$0.22}} & \cellcolor{jarqRefined}7.662\hspace{0.6pt}\makebox[\ppldw][r]{\pplgood{$-$12.6}} & \cellcolor{jarqRefined}11.989\hspace{0.6pt}\makebox[\ppldw][r]{\pplgood{$-$3.91}} & \cellcolor{jarqRefined}7.516\hspace{0.6pt}\makebox[\ppldw][r]{\pplgood{$-$3.38}} & \cellcolor{jarqRefined}11.690\hspace{0.6pt}\makebox[\ppldw][r]{\pplgood{$-$0.81}} & \cellcolor{jarqRefined}16.480\hspace{0.6pt}\makebox[\ppldw][r]{\pplgood{$-$4.30}} & \cellcolor{jarqRefined}20.346\hspace{0.6pt}\makebox[\ppldw][r]{\pplgood{$-$2.66}} & \cellcolor{jarqRefined}10.546\hspace{0.6pt}\makebox[\ppldw][r]{\pplgood{$-$8.02}} & \cellcolor{jarqRefined}14.963\hspace{0.6pt}\makebox[\ppldw][r]{\pplgood{$-$3.18}} \\
AWQ & 6.152\hspace{0.6pt}\makebox[\ppldw][r]{} & 7.771\hspace{0.6pt}\makebox[\ppldw][r]{} & 5.299\hspace{0.6pt}\makebox[\ppldw][r]{} & 6.946\hspace{0.6pt}\makebox[\ppldw][r]{} & 8.039\hspace{0.6pt}\makebox[\ppldw][r]{} & 11.539\hspace{0.6pt}\makebox[\ppldw][r]{} & 7.966\hspace{0.6pt}\makebox[\ppldw][r]{} & 11.759\hspace{0.6pt}\makebox[\ppldw][r]{} & 16.796\hspace{0.6pt}\makebox[\ppldw][r]{} & 20.418\hspace{0.6pt}\makebox[\ppldw][r]{} & 11.103\hspace{0.6pt}\makebox[\ppldw][r]{} & 14.978\hspace{0.6pt}\makebox[\ppldw][r]{} \\
\cellcolor{jarqRefined}\enspace$+$Ours & \cellcolor{jarqRefined}5.924\hspace{0.6pt}\makebox[\ppldw][r]{\pplgood{$-$3.70}} & \cellcolor{jarqRefined}7.741\hspace{0.6pt}\makebox[\ppldw][r]{\pplgood{$-$0.38}} & \cellcolor{jarqRefined}5.220\hspace{0.6pt}\makebox[\ppldw][r]{\pplgood{$-$1.49}} & \cellcolor{jarqRefined}6.865\hspace{0.6pt}\makebox[\ppldw][r]{\pplgood{$-$1.16}} & \cellcolor{jarqRefined}7.491\hspace{0.6pt}\makebox[\ppldw][r]{\pplgood{$-$6.82}} & \cellcolor{jarqRefined}11.591\hspace{0.6pt}\makebox[\ppldw][r]{\pplbad{$+$0.45}} & \cellcolor{jarqRefined}7.518\hspace{0.6pt}\makebox[\ppldw][r]{\pplgood{$-$5.62}} & \cellcolor{jarqRefined}11.637\hspace{0.6pt}\makebox[\ppldw][r]{\pplgood{$-$1.04}} & \cellcolor{jarqRefined}15.513\hspace{0.6pt}\makebox[\ppldw][r]{\pplgood{$-$7.64}} & \cellcolor{jarqRefined}20.166\hspace{0.6pt}\makebox[\ppldw][r]{\pplgood{$-$1.23}} & \cellcolor{jarqRefined}10.533\hspace{0.6pt}\makebox[\ppldw][r]{\pplgood{$-$5.13}} & \cellcolor{jarqRefined}14.776\hspace{0.6pt}\makebox[\ppldw][r]{\pplgood{$-$1.35}} \\
\midrule
\multicolumn{13}{@{}l}{\textit{W4A16}} \\
\rowcolor{jarqFP16} FP16 & 5.472\hspace{0.6pt}\makebox[\ppldw][r]{} & 6.973\hspace{0.6pt}\makebox[\ppldw][r]{} & 4.884\hspace{0.6pt}\makebox[\ppldw][r]{} & 6.468\hspace{0.6pt}\makebox[\ppldw][r]{} & 6.136\hspace{0.6pt}\makebox[\ppldw][r]{} & 8.881\hspace{0.6pt}\makebox[\ppldw][r]{} & 6.848\hspace{0.6pt}\makebox[\ppldw][r]{} & 10.442\hspace{0.6pt}\makebox[\ppldw][r]{} & 13.638\hspace{0.6pt}\makebox[\ppldw][r]{} & 16.626\hspace{0.6pt}\makebox[\ppldw][r]{} & 9.715\hspace{0.6pt}\makebox[\ppldw][r]{} & 13.290\hspace{0.6pt}\makebox[\ppldw][r]{} \\
RTN & 5.726\hspace{0.6pt}\makebox[\ppldw][r]{} & 7.246\hspace{0.6pt}\makebox[\ppldw][r]{} & 4.984\hspace{0.6pt}\makebox[\ppldw][r]{} & 6.588\hspace{0.6pt}\makebox[\ppldw][r]{} & 6.727\hspace{0.6pt}\makebox[\ppldw][r]{} & 9.636\hspace{0.6pt}\makebox[\ppldw][r]{} & 7.227\hspace{0.6pt}\makebox[\ppldw][r]{} & 10.886\hspace{0.6pt}\makebox[\ppldw][r]{} & 15.968\hspace{0.6pt}\makebox[\ppldw][r]{} & 18.291\hspace{0.6pt}\makebox[\ppldw][r]{} & 10.137\hspace{0.6pt}\makebox[\ppldw][r]{} & 13.643\hspace{0.6pt}\makebox[\ppldw][r]{} \\
\cellcolor{jarqRefined}\enspace$+$Ours & \cellcolor{jarqRefined}5.626\hspace{0.6pt}\makebox[\ppldw][r]{\pplgood{$-$1.76}} & \cellcolor{jarqRefined}7.186\hspace{0.6pt}\makebox[\ppldw][r]{\pplgood{$-$0.83}} & \cellcolor{jarqRefined}4.960\hspace{0.6pt}\makebox[\ppldw][r]{\pplgood{$-$0.48}} & \cellcolor{jarqRefined}6.583\hspace{0.6pt}\makebox[\ppldw][r]{\pplgood{$-$0.07}} & \cellcolor{jarqRefined}6.481\hspace{0.6pt}\makebox[\ppldw][r]{\pplgood{$-$3.66}} & \cellcolor{jarqRefined}9.532\hspace{0.6pt}\makebox[\ppldw][r]{\pplgood{$-$1.08}} & \cellcolor{jarqRefined}7.004\hspace{0.6pt}\makebox[\ppldw][r]{\pplgood{$-$3.09}} & \cellcolor{jarqRefined}10.697\hspace{0.6pt}\makebox[\ppldw][r]{\pplgood{$-$1.73}} & \cellcolor{jarqRefined}14.555\hspace{0.6pt}\makebox[\ppldw][r]{\pplgood{$-$8.85}} & \cellcolor{jarqRefined}17.930\hspace{0.6pt}\makebox[\ppldw][r]{\pplgood{$-$1.98}} & \cellcolor{jarqRefined}9.998\hspace{0.6pt}\makebox[\ppldw][r]{\pplgood{$-$1.37}} & \cellcolor{jarqRefined}13.627\hspace{0.6pt}\makebox[\ppldw][r]{\pplgood{$-$0.12}} \\
GPTQ & 5.585\hspace{0.6pt}\makebox[\ppldw][r]{} & 7.133\hspace{0.6pt}\makebox[\ppldw][r]{} & 4.951\hspace{0.6pt}\makebox[\ppldw][r]{} & 6.574\hspace{0.6pt}\makebox[\ppldw][r]{} & 6.423\hspace{0.6pt}\makebox[\ppldw][r]{} & 9.386\hspace{0.6pt}\makebox[\ppldw][r]{} & 7.029\hspace{0.6pt}\makebox[\ppldw][r]{} & 10.842\hspace{0.6pt}\makebox[\ppldw][r]{} & 13.675\hspace{0.6pt}\makebox[\ppldw][r]{} & 16.980\hspace{0.6pt}\makebox[\ppldw][r]{} & 9.941\hspace{0.6pt}\makebox[\ppldw][r]{} & 13.560\hspace{0.6pt}\makebox[\ppldw][r]{} \\
\cellcolor{jarqRefined}\enspace$+$Ours & \cellcolor{jarqRefined}5.564\hspace{0.6pt}\makebox[\ppldw][r]{\pplgood{$-$0.38}} & \cellcolor{jarqRefined}7.129\hspace{0.6pt}\makebox[\ppldw][r]{\pplgood{$-$0.06}} & \cellcolor{jarqRefined}4.947\hspace{0.6pt}\makebox[\ppldw][r]{\pplgood{$-$0.09}} & \cellcolor{jarqRefined}6.565\hspace{0.6pt}\makebox[\ppldw][r]{\pplgood{$-$0.13}} & \cellcolor{jarqRefined}6.416\hspace{0.6pt}\makebox[\ppldw][r]{\pplgood{$-$0.11}} & \cellcolor{jarqRefined}9.408\hspace{0.6pt}\makebox[\ppldw][r]{\pplbad{$+$0.24}} & \cellcolor{jarqRefined}7.001\hspace{0.6pt}\makebox[\ppldw][r]{\pplgood{$-$0.40}} & \cellcolor{jarqRefined}10.672\hspace{0.6pt}\makebox[\ppldw][r]{\pplgood{$-$1.57}} & \cellcolor{jarqRefined}13.553\hspace{0.6pt}\makebox[\ppldw][r]{\pplgood{$-$0.90}} & \cellcolor{jarqRefined}16.917\hspace{0.6pt}\makebox[\ppldw][r]{\pplgood{$-$0.37}} & \cellcolor{jarqRefined}9.860\hspace{0.6pt}\makebox[\ppldw][r]{\pplgood{$-$0.81}} & \cellcolor{jarqRefined}13.550\hspace{0.6pt}\makebox[\ppldw][r]{\pplgood{$-$0.08}} \\
OmniQuant & 5.592\hspace{0.6pt}\makebox[\ppldw][r]{} & 7.123\hspace{0.6pt}\makebox[\ppldw][r]{} & 4.963\hspace{0.6pt}\makebox[\ppldw][r]{} & 6.572\hspace{0.6pt}\makebox[\ppldw][r]{} & 6.628\hspace{0.6pt}\makebox[\ppldw][r]{} & 9.574\hspace{0.6pt}\makebox[\ppldw][r]{} & 7.082\hspace{0.6pt}\makebox[\ppldw][r]{} & 10.766\hspace{0.6pt}\makebox[\ppldw][r]{} & 14.665\hspace{0.6pt}\makebox[\ppldw][r]{} & 17.909\hspace{0.6pt}\makebox[\ppldw][r]{} & 10.141\hspace{0.6pt}\makebox[\ppldw][r]{} & 13.829\hspace{0.6pt}\makebox[\ppldw][r]{} \\
\cellcolor{jarqRefined}\enspace$+$Ours & \cellcolor{jarqRefined}5.584\hspace{0.6pt}\makebox[\ppldw][r]{\pplgood{$-$0.13}} & \cellcolor{jarqRefined}7.103\hspace{0.6pt}\makebox[\ppldw][r]{\pplgood{$-$0.28}} & \cellcolor{jarqRefined}4.953\hspace{0.6pt}\makebox[\ppldw][r]{\pplgood{$-$0.20}} & \cellcolor{jarqRefined}6.562\hspace{0.6pt}\makebox[\ppldw][r]{\pplgood{$-$0.15}} & \cellcolor{jarqRefined}6.496\hspace{0.6pt}\makebox[\ppldw][r]{\pplgood{$-$1.99}} & \cellcolor{jarqRefined}9.539\hspace{0.6pt}\makebox[\ppldw][r]{\pplgood{$-$0.37}} & \cellcolor{jarqRefined}6.996\hspace{0.6pt}\makebox[\ppldw][r]{\pplgood{$-$1.21}} & \cellcolor{jarqRefined}10.694\hspace{0.6pt}\makebox[\ppldw][r]{\pplgood{$-$0.67}} & \cellcolor{jarqRefined}13.728\hspace{0.6pt}\makebox[\ppldw][r]{\pplgood{$-$6.39}} & \cellcolor{jarqRefined}17.386\hspace{0.6pt}\makebox[\ppldw][r]{\pplgood{$-$2.92}} & \cellcolor{jarqRefined}9.858\hspace{0.6pt}\makebox[\ppldw][r]{\pplgood{$-$2.79}} & \cellcolor{jarqRefined}13.582\hspace{0.6pt}\makebox[\ppldw][r]{\pplgood{$-$1.79}} \\
AWQ & 5.607\hspace{0.6pt}\makebox[\ppldw][r]{} & 7.119\hspace{0.6pt}\makebox[\ppldw][r]{} & 4.972\hspace{0.6pt}\makebox[\ppldw][r]{} & 6.559\hspace{0.6pt}\makebox[\ppldw][r]{} & 6.549\hspace{0.6pt}\makebox[\ppldw][r]{} & 9.410\hspace{0.6pt}\makebox[\ppldw][r]{} & 7.103\hspace{0.6pt}\makebox[\ppldw][r]{} & 10.726\hspace{0.6pt}\makebox[\ppldw][r]{} & 14.620\hspace{0.6pt}\makebox[\ppldw][r]{} & 17.793\hspace{0.6pt}\makebox[\ppldw][r]{} & 9.855\hspace{0.6pt}\makebox[\ppldw][r]{} & 13.576\hspace{0.6pt}\makebox[\ppldw][r]{} \\
\cellcolor{jarqRefined}\enspace$+$Ours & \cellcolor{jarqRefined}5.553\hspace{0.6pt}\makebox[\ppldw][r]{\pplgood{$-$0.97}} & \cellcolor{jarqRefined}7.111\hspace{0.6pt}\makebox[\ppldw][r]{\pplgood{$-$0.12}} & \cellcolor{jarqRefined}4.955\hspace{0.6pt}\makebox[\ppldw][r]{\pplgood{$-$0.34}} & \cellcolor{jarqRefined}6.550\hspace{0.6pt}\makebox[\ppldw][r]{\pplgood{$-$0.13}} & \cellcolor{jarqRefined}6.448\hspace{0.6pt}\makebox[\ppldw][r]{\pplgood{$-$1.55}} & \cellcolor{jarqRefined}9.368\hspace{0.6pt}\makebox[\ppldw][r]{\pplgood{$-$0.45}} & \cellcolor{jarqRefined}7.000\hspace{0.6pt}\makebox[\ppldw][r]{\pplgood{$-$1.46}} & \cellcolor{jarqRefined}10.692\hspace{0.6pt}\makebox[\ppldw][r]{\pplgood{$-$0.32}} & \cellcolor{jarqRefined}13.762\hspace{0.6pt}\makebox[\ppldw][r]{\pplgood{$-$5.87}} & \cellcolor{jarqRefined}17.222\hspace{0.6pt}\makebox[\ppldw][r]{\pplgood{$-$3.21}} & \cellcolor{jarqRefined}9.879\hspace{0.6pt}\makebox[\ppldw][r]{\pplbad{$+$0.24}} & \cellcolor{jarqRefined}13.532\hspace{0.6pt}\makebox[\ppldw][r]{\pplgood{$-$0.32}} \\
\bottomrule
\end{tabular}
}
\vspace{-4pt}
\end{table*}

\paragraph{Downstream accuracy.}
Refinement raises mean multiple-choice accuracy in 23 of 24
configurations (Table~\ref{tab:downstream-selected-five}), including all
eight on Llama-3-8B, by up to +5.57 points (Qwen3-8B, RTN, W3A16). Across
individual tasks, it is better in 82 of 120 comparisons; the one lower
average, Qwen3-8B with OmniQuant at W3A16 ($-$0.84), comes mainly from ARC-E. On generation, refinement raises the W4A16 host average on GSM8K
and MATH-500 for all three models (Figure~\ref{fig:generative-w4}); at
W3A16 it raises GSM8K accuracy in six of eight configurations, by up to
8.56 points (Table~\ref{tab:generative-all}). As with perplexity, the
gains grow as bits shrink: the change in mean multiple-choice accuracy
averages $+1.95$ points at W3A16 and $+0.52$ at W4A16, and the three
largest gains all come from RTN. With the RTN and GPTQ hosts, refinement
raises mean accuracy in all 12 configurations.

\begin{table*}[t]
\centering
\caption{Multiple-choice accuracy (\%; higher is better). Entries are
Base/$+$Ours; Avg.\ is the five-task mean, followed by the change of
$+$Ours in points (green: higher, red: lower). Bold marks the higher value; ties are not bolded.}
\label{tab:downstream-selected-five}
\label{tab:downstream_all}
\label{tab:downstream-qwen-fp16}

\begingroup
\setlength{\tabcolsep}{3.1pt}
\newcommand{\downstreamours}[1]{%
  \begingroup
  \setlength{\fboxsep}{0.4pt}%
  \colorbox{jarqRefined}{\strut #1}%
  \endgroup}
\renewcommand{\arraystretch}{1.12}
\scriptsize

\resizebox{\textwidth}{!}{%
\begin{tabular}{lllccccc!{\hspace{3pt}\vrule width 0.5pt\hspace{3pt}}c}
\toprule
Model & Host & W/A
& ARC-E & ARC-C & MMLU & Wino. & BoolQ & \textbf{Avg.} \\
\midrule

\multirow{9}{*}{Llama-3-8B}
& \cellcolor{jarqFP16}FP16 & \cellcolor{jarqFP16}W16A16 & \cellcolor{jarqFP16}77.53 & \cellcolor{jarqFP16}54.10 & \cellcolor{jarqFP16}65.45 & \cellcolor{jarqFP16}74.03 & \cellcolor{jarqFP16}82.26 & \cellcolor{jarqFP16}70.67 \\
\cmidrule(lr){2-9}
& \multirow{2}{*}{RTN}
& W3A16 & 61.07/\downstreamours{\textbf{67.85}} & 40.53/\downstreamours{\textbf{44.45}} & 47.17/\downstreamours{\textbf{53.11}} & 67.48/\downstreamours{\textbf{71.03}} & 69.11/\downstreamours{\textbf{75.60}} & 57.07/\downstreamours{\textbf{62.41}}\,\pplgood{$+$5.34} \\
& & W4A16 & 77.31/\downstreamours{\textbf{77.78}} & \textbf{51.62}/\downstreamours{51.11} & 62.63/\downstreamours{\textbf{63.57}} & 72.38/\downstreamours{72.38} & 79.69/\downstreamours{\textbf{81.07}} & 68.73/\downstreamours{\textbf{69.18}}\,\pplgood{$+$0.45} \\

& \multirow{2}{*}{GPTQ}
& W3A16 & 63.59/\downstreamours{\textbf{69.32}} & 39.16/\downstreamours{\textbf{43.00}} & 52.81/\downstreamours{\textbf{57.47}} & \textbf{71.35}/\downstreamours{71.19} & 69.88/\downstreamours{\textbf{70.34}} & 59.36/\downstreamours{\textbf{62.26}}\,\pplgood{$+$2.90} \\
& & W4A16 & 77.48/\downstreamours{\textbf{78.75}} & \textbf{53.67}/\downstreamours{53.33} & \textbf{63.71}/\downstreamours{63.45} & 74.11/\downstreamours{\textbf{74.43}} & 81.44/\downstreamours{\textbf{82.51}} & 70.08/\downstreamours{\textbf{70.49}}\,\pplgood{$+$0.41} \\

& \multirow{2}{*}{OmniQuant}
& W3A16 & 63.80/\downstreamours{\textbf{68.52}} & 40.96/\downstreamours{\textbf{43.00}} & \textbf{53.45}/\downstreamours{52.55} & 68.75/\downstreamours{\textbf{70.72}} & \textbf{75.08}/\downstreamours{71.41} & 60.41/\downstreamours{\textbf{61.24}}\,\pplgood{$+$0.83} \\
& & W4A16 & \textbf{76.94}/\downstreamours{76.64} & 50.77/\downstreamours{\textbf{52.82}} & 63.64/\downstreamours{\textbf{63.76}} & 71.98/\downstreamours{\textbf{73.09}} & 80.80/\downstreamours{\textbf{80.83}} & 68.83/\downstreamours{\textbf{69.43}}\,\pplgood{$+$0.60} \\

& \multirow{2}{*}{AWQ}
& W3A16 & 74.71/\downstreamours{\textbf{76.09}} & \textbf{48.63}/\downstreamours{47.53} & 55.33/\downstreamours{\textbf{57.34}} & 71.03/\downstreamours{\textbf{72.06}} & \textbf{80.03}/\downstreamours{78.44} & 65.95/\downstreamours{\textbf{66.29}}\,\pplgood{$+$0.34} \\
& & W4A16 & 75.00/\downstreamours{\textbf{78.07}} & 51.71/\downstreamours{\textbf{52.47}} & \textbf{63.40}/\downstreamours{62.88} & \textbf{72.38}/\downstreamours{72.30} & 80.92/\downstreamours{\textbf{81.38}} & 68.68/\downstreamours{\textbf{69.42}}\,\pplgood{$+$0.74} \\

\midrule

\multirow{9}{*}{Qwen2.5-7B}
& \cellcolor{jarqFP16}FP16 & \cellcolor{jarqFP16}W16A16 & \cellcolor{jarqFP16}77.44 & \cellcolor{jarqFP16}51.11 & \cellcolor{jarqFP16}74.30 & \cellcolor{jarqFP16}72.93 & \cellcolor{jarqFP16}84.65 & \cellcolor{jarqFP16}72.09 \\
\cmidrule(lr){2-9}
& \multirow{2}{*}{RTN}
& W3A16 & 68.73/\downstreamours{\textbf{70.79}} & 47.01/\downstreamours{\textbf{48.21}} & 67.14/\downstreamours{\textbf{67.97}} & 64.48/\downstreamours{\textbf{69.38}} & 77.61/\downstreamours{\textbf{82.84}} & 64.99/\downstreamours{\textbf{67.84}}\,\pplgood{$+$2.85} \\
& & W4A16 & 76.68/\downstreamours{\textbf{78.75}} & 50.00/\downstreamours{\textbf{51.79}} & 72.68/\downstreamours{\textbf{73.14}} & \textbf{71.82}/\downstreamours{71.59} & \textbf{84.01}/\downstreamours{83.15} & 71.04/\downstreamours{\textbf{71.68}}\,\pplgood{$+$0.64} \\

& \multirow{2}{*}{GPTQ}
& W3A16 & 72.47/\downstreamours{\textbf{77.10}} & 48.55/\downstreamours{\textbf{50.77}} & 68.19/\downstreamours{\textbf{68.74}} & 68.59/\downstreamours{\textbf{69.53}} & 83.52/\downstreamours{\textbf{84.25}} & 68.26/\downstreamours{\textbf{70.08}}\,\pplgood{$+$1.82} \\
& & W4A16 & 75.72/\downstreamours{\textbf{77.78}} & 49.74/\downstreamours{\textbf{51.96}} & \textbf{73.31}/\downstreamours{73.16} & \textbf{72.22}/\downstreamours{71.82} & \textbf{83.21}/\downstreamours{82.78} & 70.84/\downstreamours{\textbf{71.50}}\,\pplgood{$+$0.66} \\

& \multirow{2}{*}{OmniQuant}
& W3A16 & 77.57/\downstreamours{\textbf{78.66}} & 50.34/\downstreamours{\textbf{52.56}} & \textbf{69.34}/\downstreamours{68.93} & 68.03/\downstreamours{\textbf{71.11}} & 82.81/\downstreamours{\textbf{83.12}} & 69.62/\downstreamours{\textbf{70.88}}\,\pplgood{$+$1.26} \\
& & W4A16 & 76.60/\downstreamours{\textbf{77.82}} & 49.15/\downstreamours{\textbf{51.71}} & \textbf{73.48}/\downstreamours{72.96} & \textbf{72.22}/\downstreamours{71.90} & 84.98/\downstreamours{84.98} & 71.29/\downstreamours{\textbf{71.87}}\,\pplgood{$+$0.58} \\

& \multirow{2}{*}{AWQ}
& W3A16 & 66.55/\downstreamours{\textbf{70.24}} & \textbf{49.57}/\downstreamours{46.16} & \textbf{69.62}/\downstreamours{68.23} & 67.48/\downstreamours{\textbf{71.51}} & \textbf{81.93}/\downstreamours{80.92} & 67.03/\downstreamours{\textbf{67.41}}\,\pplgood{$+$0.38} \\
& & W4A16 & 77.19/\downstreamours{\textbf{77.48}} & \textbf{51.11}/\downstreamours{50.34} & \textbf{73.32}/\downstreamours{73.07} & \textbf{72.14}/\downstreamours{71.90} & 80.83/\downstreamours{\textbf{82.54}} & 70.92/\downstreamours{\textbf{71.07}}\,\pplgood{$+$0.15} \\

\midrule

\multirow{9}{*}{Qwen3-8B}
& \cellcolor{jarqFP16}FP16 & \cellcolor{jarqFP16}W16A16 & \cellcolor{jarqFP16}80.93 & \cellcolor{jarqFP16}56.57 & \cellcolor{jarqFP16}74.96 & \cellcolor{jarqFP16}67.64 & \cellcolor{jarqFP16}86.61 & \cellcolor{jarqFP16}73.34 \\
\cmidrule(lr){2-9}
& \multirow{2}{*}{RTN}
& W3A16 & 71.93/\downstreamours{\textbf{76.14}} & 48.55/\downstreamours{\textbf{52.13}} & 61.10/\downstreamours{\textbf{68.49}} & 61.56/\downstreamours{\textbf{68.90}} & 79.88/\downstreamours{\textbf{85.20}} & 64.60/\downstreamours{\textbf{70.17}}\,\pplgood{$+$5.57} \\
& & W4A16 & 75.88/\downstreamours{\textbf{78.49}} & 53.07/\downstreamours{\textbf{55.38}} & 73.66/\downstreamours{\textbf{74.23}} & 67.64/\downstreamours{\textbf{67.96}} & 85.29/\downstreamours{\textbf{86.64}} & 71.11/\downstreamours{\textbf{72.54}}\,\pplgood{$+$1.43} \\

& \multirow{2}{*}{GPTQ}
& W3A16 & 71.80/\downstreamours{\textbf{78.24}} & 47.61/\downstreamours{\textbf{51.45}} & 67.11/\downstreamours{\textbf{68.59}} & \textbf{67.88}/\downstreamours{66.38} & 84.77/\downstreamours{\textbf{85.29}} & 67.83/\downstreamours{\textbf{69.99}}\,\pplgood{$+$2.16} \\
& & W4A16 & 79.42/\downstreamours{\textbf{79.84}} & \textbf{55.12}/\downstreamours{54.61} & 73.99/\downstreamours{73.99} & \textbf{68.51}/\downstreamours{68.27} & 86.48/\downstreamours{\textbf{87.16}} & 72.70/\downstreamours{\textbf{72.77}}\,\pplgood{$+$0.07} \\

& \multirow{2}{*}{OmniQuant}
& W3A16 & \textbf{77.48}/\downstreamours{72.56} & \textbf{50.60}/\downstreamours{48.04} & 68.03/\downstreamours{\textbf{69.23}} & 64.88/\downstreamours{\textbf{67.56}} & \textbf{85.35}/\downstreamours{84.74} & \textbf{69.27}/\downstreamours{68.43}\,\pplbad{$-$0.84} \\
& & W4A16 & 78.91/\downstreamours{\textbf{79.42}} & 53.84/\downstreamours{\textbf{55.12}} & \textbf{73.81}/\downstreamours{73.71} & 68.27/\downstreamours{\textbf{68.67}} & 86.45/\downstreamours{\textbf{86.88}} & 72.26/\downstreamours{\textbf{72.76}}\,\pplgood{$+$0.50} \\

& \multirow{2}{*}{AWQ}
& W3A16 & 76.09/\downstreamours{\textbf{76.60}} & 50.43/\downstreamours{\textbf{50.94}} & \textbf{69.10}/\downstreamours{69.00} & 64.01/\downstreamours{\textbf{66.54}} & 84.71/\downstreamours{\textbf{85.11}} & 68.87/\downstreamours{\textbf{69.64}}\,\pplgood{$+$0.77} \\
& & W4A16 & \textbf{80.39}/\downstreamours{79.67} & \textbf{55.55}/\downstreamours{55.12} & 73.66/\downstreamours{\textbf{73.86}} & 68.59/\downstreamours{\textbf{68.98}} & 86.18/\downstreamours{\textbf{86.79}} & 72.87/\downstreamours{\textbf{72.88}}\,\pplgood{$+$0.01} \\

\midrule
\multicolumn{8}{l}{Improved / evaluated configurations (Avg.)}
& 23/24 \\
\bottomrule
\end{tabular}%
}

\endgroup
\vspace{-4pt}
\end{table*}

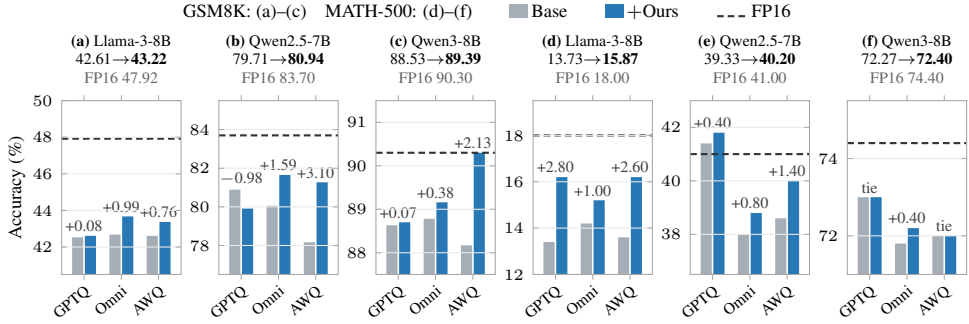
\begin{figure*}[!t]
\centering
\begingroup
\definecolor{genbase}{HTML}{A5ADB7}
\definecolor{genours}{HTML}{2878B5}
\centering
{\scriptsize
GSM8K: (a)--(c)\quad MATH-500: (d)--(f)\qquad
\textcolor{genbase}{\rule{7pt}{5pt}}\,Base\qquad
\textcolor{genours}{\rule{7pt}{5pt}}\,$+$Ours\qquad
\tikz\draw[black!80,dash pattern=on 2.5pt off 1.5pt,line width=0.8pt](0,0)--(0.4,0);\,FP16\par}
\vspace{1pt}
\begin{tikzpicture}
\begin{groupplot}[
 group style={group size=6 by 1,horizontal sep=0.5cm},
 scale only axis,width=0.113\textwidth,height=2.3cm,
 ybar=0.6pt,/pgf/bar width=4.2pt,
 xmin=-0.6,xmax=2.6,xtick={0,1,2},xticklabels={GPTQ,Omni,AWQ},
 xticklabel style={font=\tiny,rotate=40,anchor=north east,inner sep=1pt},
 yticklabel style={font=\tiny,/pgf/number format/fixed,/pgf/number format/precision=1},
 title style={font=\tiny,align=center,yshift=-3pt},
 axis line style={black!35},tick style={black!35},major tick length=2pt,
 ymajorgrids=true,grid style={black!8},axis on top,
 enlarge x limits=false,scaled y ticks=false,clip=false,
]
\nextgroupplot[title={\textbf{(a)} Llama-3-8B\\42.61$\to$\textbf{43.22}\\{\color{black!60}FP16 47.92}},ymin=40.5,ymax=50.0,ylabel={Accuracy (\%)},ylabel style={font=\scriptsize,yshift=-4pt}]
\addplot[fill=genbase,draw=none] coordinates {(0,42.53) (1,42.68) (2,42.61)};
\addplot[fill=genours,draw=none] coordinates {(0,42.61) (1,43.67) (2,43.37)};
\draw[black!80,dash pattern=on 2.5pt off 1.5pt,line width=0.8pt] (axis cs:-0.6,47.92) -- (axis cs:2.6,47.92);
\node[font=\tiny,text=black!75,anchor=south,inner sep=1.5pt] at (axis cs:0,42.61) {+0.08};
\node[font=\tiny,text=black!75,anchor=south,inner sep=1.5pt] at (axis cs:1,43.67) {+0.99};
\node[font=\tiny,text=black!75,anchor=south,inner sep=1.5pt] at (axis cs:2,43.37) {+0.76};
\nextgroupplot[title={\textbf{(b)} Qwen2.5-7B\\79.71$\to$\textbf{80.94}\\{\color{black!60}FP16 83.70}},ymin=76.5,ymax=85.5]
\addplot[fill=genbase,draw=none] coordinates {(0,80.89) (1,80.06) (2,78.17)};
\addplot[fill=genours,draw=none] coordinates {(0,79.91) (1,81.65) (2,81.27)};
\draw[black!80,dash pattern=on 2.5pt off 1.5pt,line width=0.8pt] (axis cs:-0.6,83.70) -- (axis cs:2.6,83.70);
\node[font=\tiny,text=black!75,anchor=south,inner sep=1.5pt] at (axis cs:0,80.89) {$-$0.98};
\node[font=\tiny,text=black!75,anchor=south,inner sep=1.5pt] at (axis cs:1,81.65) {+1.59};
\node[font=\tiny,text=black!75,anchor=south,inner sep=1.5pt] at (axis cs:2,81.27) {+3.10};
\nextgroupplot[title={\textbf{(c)} Qwen3-8B\\88.53$\to$\textbf{89.39}\\{\color{black!60}FP16 90.30}},ymin=87.5,ymax=91.5]
\addplot[fill=genbase,draw=none] coordinates {(0,88.63) (1,88.78) (2,88.17)};
\addplot[fill=genours,draw=none] coordinates {(0,88.70) (1,89.16) (2,90.30)};
\draw[black!80,dash pattern=on 2.5pt off 1.5pt,line width=0.8pt] (axis cs:-0.6,90.30) -- (axis cs:2.6,90.30);
\node[font=\tiny,text=black!75,anchor=south,inner sep=1.5pt] at (axis cs:0,88.70) {+0.07};
\node[font=\tiny,text=black!75,anchor=south,inner sep=1.5pt] at (axis cs:1,89.16) {+0.38};
\node[font=\tiny,text=black!75,anchor=south,inner sep=1.5pt] at (axis cs:2,90.30) {+2.13};
\nextgroupplot[title={\textbf{(d)} Llama-3-8B\\13.73$\to$\textbf{15.87}\\{\color{black!60}FP16 18.00}},ymin=12.0,ymax=19.5]
\addplot[fill=genbase,draw=none] coordinates {(0,13.40) (1,14.20) (2,13.60)};
\addplot[fill=genours,draw=none] coordinates {(0,16.20) (1,15.20) (2,16.20)};
\draw[black!80,dash pattern=on 2.5pt off 1.5pt,line width=0.8pt] (axis cs:-0.6,18.00) -- (axis cs:2.6,18.00);
\node[font=\tiny,text=black!75,anchor=south,inner sep=1.5pt] at (axis cs:0,16.20) {+2.80};
\node[font=\tiny,text=black!75,anchor=south,inner sep=1.5pt] at (axis cs:1,15.20) {+1.00};
\node[font=\tiny,text=black!75,anchor=south,inner sep=1.5pt] at (axis cs:2,16.20) {+2.60};
\nextgroupplot[title={\textbf{(e)} Qwen2.5-7B\\39.33$\to$\textbf{40.20}\\{\color{black!60}FP16 41.00}},ymin=36.5,ymax=43.0]
\addplot[fill=genbase,draw=none] coordinates {(0,41.40) (1,38.00) (2,38.60)};
\addplot[fill=genours,draw=none] coordinates {(0,41.80) (1,38.80) (2,40.00)};
\draw[black!80,dash pattern=on 2.5pt off 1.5pt,line width=0.8pt] (axis cs:-0.6,41.00) -- (axis cs:2.6,41.00);
\node[font=\tiny,text=black!75,anchor=south,inner sep=1.5pt] at (axis cs:0,41.80) {+0.40};
\node[font=\tiny,text=black!75,anchor=south,inner sep=1.5pt] at (axis cs:1,38.80) {+0.80};
\node[font=\tiny,text=black!75,anchor=south,inner sep=1.5pt] at (axis cs:2,40.00) {+1.40};
\nextgroupplot[title={\textbf{(f)} Qwen3-8B\\72.27$\to$\textbf{72.40}\\{\color{black!60}FP16 74.40}},ymin=71.0,ymax=75.5]
\addplot[fill=genbase,draw=none] coordinates {(0,73.00) (1,71.80) (2,72.00)};
\addplot[fill=genours,draw=none] coordinates {(0,73.00) (1,72.20) (2,72.00)};
\draw[black!80,dash pattern=on 2.5pt off 1.5pt,line width=0.8pt] (axis cs:-0.6,74.40) -- (axis cs:2.6,74.40);
\node[font=\tiny,text=black!75,anchor=south,inner sep=1.5pt] at (axis cs:0,73.00) {tie};
\node[font=\tiny,text=black!75,anchor=south,inner sep=1.5pt] at (axis cs:1,72.20) {+0.40};
\node[font=\tiny,text=black!75,anchor=south,inner sep=1.5pt] at (axis cs:2,72.00) {tie};
\end{groupplot}
\end{tikzpicture}
\endgroup
\caption{Generative accuracy at W4A16 on GSM8K (a--c) and MATH-500 (d--f)
for three hosts: Base (gray), $+$Ours (blue), and FP16 (dashed). Labels give the change
in points; titles give the mean over hosts. Axes are zoomed per panel.
W3A16 GSM8K scores are in Table~\ref{tab:generative-all}.}
\label{fig:generative-w4}
\vspace{-6pt}
\end{figure*}

\subsection{Compatibility with Other Quantizers}
\label{sec:compatibility}
\label{sec:qep_results}
\label{sec:rotation_results}

Because \method{} needs only the host's scales, codes, and zero points,
it plugs into target correction, rotation, and lattice-based code
search alike (Table~\ref{tab:host-compatibility}). Starting from QEP+GPTQ, it
lowers WT2 perplexity in all four settings, by 4.4\% on Qwen3-8B at
W3A16. After QuaRot+GPTQ on Llama-2-7B, it lowers both perplexities
at both bit widths with the rotation fixed. After
OJBKQ~\citep{OJBKQ}, which already searches codes with Babai--Klein
decoding on a fixed grid, refitting the grid still lowers WT2
perplexity in all four settings.
QuaRot and OJBKQ use C4 calibration; for Qwen3-4B GPTQ at W3A16,
\method{} raises GSM8K with WT2, C4, or Pile calibration
(Appendix~\ref{app:calibration-corpus}).

\begin{table}[t]
\centering
\caption{Compatibility results for (a) QEP~\citep{arai2025qep}, (b)
QuaRot~\citep{ashkboos2024quarot} on Llama-2-7B, and (c)
OJBKQ~\citep{OJBKQ} (WT2 perplexity, with C4 also in (b); lower is
better). 
Settings are in Appendix~\ref{app:qep-quarot-protocol}.}
\label{tab:host-compatibility}
\label{tab:qep_compatibility}
\label{tab:quarot-compatibility}
\label{tab:ojbkq-compatibility}
\begingroup
\footnotesize
\setlength{\tabcolsep}{2.5pt}
\renewcommand{\ppl}[1]{\pgfmathprintnumber[verbatim,fixed,precision=2,fixed zerofill]{#1}}
\renewcommand{\arraystretch}{1.05}
\begin{minipage}[t]{0.34\textwidth}
\centering
\textbf{(a) QEP}\par\smallskip
\begin{tabular*}{\linewidth}{@{\extracolsep{\fill}}llrr@{}}
\toprule
Model & W/A & Base & $+$Ours \\
\midrule
Qwen2.5-7B & W3A16 & \ppl{7.6690} & \textbf{\ppl{7.5821}} \\
& W4A16 & \ppl{7.0168} & \textbf{\ppl{6.9967}} \\
\addlinespace[2pt]
Qwen3-8B & W3A16 & \ppl{11.0015} & \textbf{\ppl{10.5186}} \\
& W4A16 & \ppl{9.9354} & \textbf{\ppl{9.8786}} \\
\bottomrule
\end{tabular*}
\end{minipage}\hfill
\begin{minipage}[t]{0.28\textwidth}
\centering
\textbf{(b) QuaRot}\par\smallskip
\begin{tabular*}{\linewidth}{@{\extracolsep{\fill}}llrr@{}}
\toprule
W/A & Data & Base & $+$Ours \\
\midrule
W3A16 & WT2 & \ppl{6.3783} & \textbf{\ppl{6.3404}} \\
& C4 & \ppl{8.5945} & \textbf{\ppl{8.4621}} \\
\addlinespace[2pt]
W4A16 & WT2 & \ppl{5.6300} & \textbf{\ppl{5.6121}} \\
& C4 & \ppl{7.4618} & \textbf{\ppl{7.4494}} \\
\bottomrule
\end{tabular*}
\end{minipage}\hfill
\begin{minipage}[t]{0.36\textwidth}
\centering
\textbf{(c) OJBKQ}\par\smallskip
\begin{tabular*}{\linewidth}{@{\extracolsep{\fill}}llrr@{}}
\toprule
Model & W/A & Base & $+$Ours \\
\midrule
Qwen3-8B& W3A16 & 10.98 & \textbf{10.74}\\
& W4A16 & 9.94 & \textbf{9.85}\\
\addlinespace[2pt]
Llama-3-8B & W3A16 & 7.96 & \textbf{7.65} \\
& W4A16 & 6.79 & \textbf{6.45} \\
\bottomrule
\end{tabular*}
\end{minipage}
\endgroup
\vspace{-4pt}
\end{table}

\subsection{Where Do the Gains Come From?}
\label{sec:ablations}
\label{sec:refinement-budget}
\label{sec:efficiency}

\begin{table*}[t]
\centering
\caption{Ablations. (a) Medians over 16 layers at W3A16
(Appendix~\ref{app:refinement-mechanisms}). (b) Full-model WT2/C4 perplexity on Llama-3-8B (L3)
and Qwen3-4B (Q3); bold: lowest (Appendix~\ref{app:mechanism-full-model}).}
\label{tab:mechanism-summary}
\label{tab:ablation-main}
\begingroup
\scriptsize
\setlength{\tabcolsep}{2.5pt}
\begin{minipage}[t]{0.36\textwidth}
\centering
\textbf{(a) Layer level}\par\smallskip
\begin{tabular*}{\linewidth}{@{\extracolsep{\fill}}lrr@{}}
\toprule
Measurement & RTN & GPTQ \\
\midrule
Indep.\ scale fits: gap (\%) & 20.98 & 25.29 \\
3-pass coord.\ descent: gap (\%) & 1.31 & 0.01 \\
3-pass coord.\ descent: time ratio & 1.29 & 1.32 \\
\midrule
Babai gain after 1-code search (\%) & 40.14 & 13.53 \\
Median codes changed & 30 & 6 \\
\midrule
Joint over Fixed, held-out (\%) & 12.32 & 5.47 \\
Layers improved & 16/16 & 16/16 \\
\bottomrule
\end{tabular*}
\end{minipage}\hfill
\begin{minipage}[t]{0.62\textwidth}
\centering
\textbf{(b) Full model (WT2/C4 perplexity)}\par\smallskip
\begin{tabular*}{\linewidth}{@{\extracolsep{\fill}}lllcccc@{}}
\toprule
& Host & & Host & Codes only & Scales only & Full \method{} \\
\midrule
L3 & RTN & W3 & 12.05/16.44 & 8.69/14.17 & 8.58/13.31 & \textbf{8.22}/\textbf{12.67} \\
 &  & W4 & 6.73/9.64 & 6.59/9.61 & 6.49/9.58 & \textbf{6.48}/\textbf{9.53} \\
 & GPTQ & W3 & \textbf{7.62}/13.24 & 7.71/12.32 & 7.78/12.63 & 7.65/\textbf{12.01} \\
 &  & W4 & 6.42/\textbf{9.39} & 6.42/9.42 & 6.42/9.41 & \textbf{6.42}/9.41 \\
\midrule
Q3 & RTN & W3 & 22.45/27.01 & 21.20/27.08 & 17.97/22.32 & \textbf{16.19}/\textbf{21.62} \\
 &  & W4 & 15.97/18.29 & 14.63/\textbf{17.92} & 14.60/17.96 & \textbf{14.55}/17.93 \\
 & GPTQ & W3 & 15.15/\textbf{18.62} & 15.90/18.95 & 15.09/18.86 & \textbf{15.09}/18.79 \\
 &  & W4 & 13.68/16.98 & 13.66/16.93 & 13.57/16.92 & \textbf{13.55}/\textbf{16.92} \\
\bottomrule
\end{tabular*}
\end{minipage}
\endgroup
\vspace{-6pt}
\end{table*}

\paragraph{Each ingredient is needed.}
Table~\ref{tab:ablation-main}(a) isolates the ingredients on 16 layers of
Llama-2-7B and Qwen3-8B at W3A16 (Appendix~\ref{app:refinement-mechanisms}).
\emph{Groups must be fitted jointly:} fitting each group's scale on its
own leaves a median error gap of 21--25\% of the starting objective,
caused by input correlations between groups
(Proposition~\ref{prop:independent-gap}). \emph{Codes must move
together:} after single-code search stops, one Babai proposal still
lowers the objective in 40\% (RTN) and 14\% (GPTQ) of cases.
\emph{The grid must move:} at equal solver time, alternating updates beat
code updates on the host grid in all 16 layers.
Table~\ref{tab:ablation-main}(b) supports this on full models: full
\method{} beats a scale refit alone in 16 of 16 comparisons, whereas code
updates on the host grid alone raise perplexity in 5 of 16.

\begin{figure*}[t]
\centering
\includegraphics[width=\textwidth]{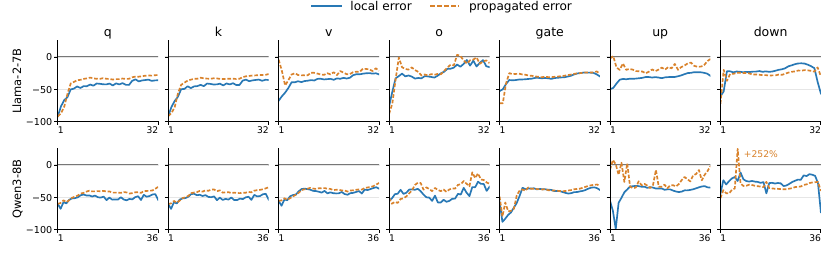}
\vspace{-12pt}
\caption{Per-block error change (\%) from refinement with the AWQ host at
W3A16 on held-out WT2, for the seven linear modules of Llama-2-7B (top)
and Qwen3-8B (bottom); below zero is better. Local error uses fixed
full-precision inputs, propagated error each model's own inputs. Local
error falls in all 476 block--module pairs, propagated error in 468.
Other settings: Appendix~\ref{app:layer-errors}.}
\label{fig:layer-excerpt}
\vspace{-6pt}
\end{figure*}

\paragraph{Gains carry over to held-out data and through depth.}
On held-out WT2 and C4 inputs, refinement lowers the local error of
92.9\% of 7,616 block--module comparisons on Llama-2-7B and Qwen3-8B.
Figure~\ref{fig:layer-excerpt} shows the AWQ host on both models: the
local error of every projection falls at every depth, and the propagated
error falls in 468 of the 476 block--module pairs. At the last block, block-output error falls in 23 of
32 comparisons, including 8/8 for RTN, 7/8 for GPTQ, and 6/8 for AWQ
(Appendix~\ref{app:layer-errors}); OmniQuant, which already fits block
outputs itself, gains least (2/8).

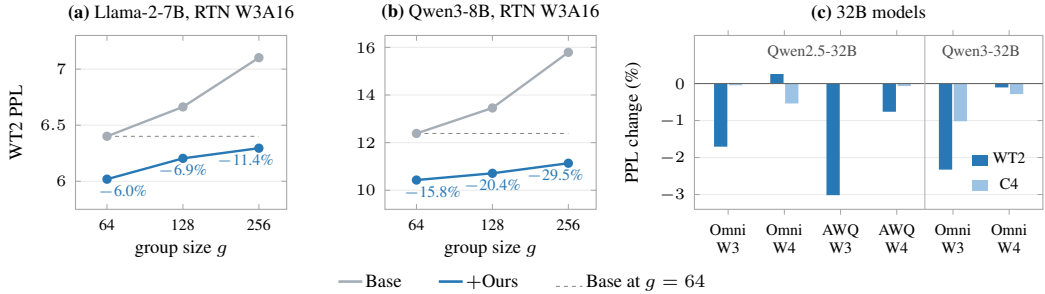
\begin{figure*}[t]
\centering
\begingroup
\definecolor{rbbase}{HTML}{A5ADB7}
\definecolor{rbours}{HTML}{2878B5}
\definecolor{rbc4}{HTML}{9CC3E4}
\centering
\begin{tikzpicture}
\begin{groupplot}[
 group style={group size=3 by 1, horizontal sep=1.3cm},
 scale only axis, height=2.3cm,
 tick label style={font=\tiny}, label style={font=\scriptsize},
 title style={font=\scriptsize, yshift=-4pt},
 axis line style={black!35}, tick style={black!35}, major tick length=2pt,
 ymajorgrids=true, grid style={black!8}, clip=false,
]
\nextgroupplot[width=0.2\textwidth, xmode=log, log basis x=2, xtick={64,128,256}, xticklabels={64,128,256}, xmin=48, xmax=330, xlabel={group size $g$}, xlabel style={yshift=3pt}, title={\textbf{(a)} Llama-2-7B, RTN W3A16}, ymin=5.75, ymax=7.3, ylabel={WT2 PPL}]
\addplot[rbbase, mark=*, mark size=1.3pt, line width=0.9pt] coordinates {(64,6.4009) (128,6.6630) (256,7.1019)};
\addplot[rbours, mark=*, mark size=1.3pt, line width=0.9pt] coordinates {(64,6.0186) (128,6.2043) (256,6.2951)};
\draw[black!45, dash pattern=on 1.5pt off 1.5pt] (axis cs:64,6.4009) -- (axis cs:256,6.4009);
\node[font=\tiny, text=rbours, anchor=north, inner sep=2pt, xshift=6pt] at (axis cs:64,6.0186) {$-$6.0\%};
\node[font=\tiny, text=rbours, anchor=north, inner sep=2pt] at (axis cs:128,6.2043) {$-$6.9\%};
\node[font=\tiny, text=rbours, anchor=north, inner sep=2pt, xshift=-5pt] at (axis cs:256,6.2951) {$-$11.4\%};
\nextgroupplot[width=0.2\textwidth, xmode=log, log basis x=2, xtick={64,128,256}, xticklabels={64,128,256}, xmin=48, xmax=330, xlabel={group size $g$}, xlabel style={yshift=3pt}, title={\textbf{(b)} Qwen3-8B, RTN W3A16}, ymin=9.2, ymax=16.5]
\addplot[rbbase, mark=*, mark size=1.3pt, line width=0.9pt] coordinates {(64,12.3856) (128,13.4564) (256,15.7918)};
\addplot[rbours, mark=*, mark size=1.3pt, line width=0.9pt] coordinates {(64,10.4292) (128,10.7123) (256,11.1354)};
\draw[black!45, dash pattern=on 1.5pt off 1.5pt] (axis cs:64,12.3856) -- (axis cs:256,12.3856);
\node[font=\tiny, text=rbours, anchor=north, inner sep=2pt, xshift=6pt] at (axis cs:64,10.4292) {$-$15.8\%};
\node[font=\tiny, text=rbours, anchor=north, inner sep=2pt] at (axis cs:128,10.7123) {$-$20.4\%};
\node[font=\tiny, text=rbours, anchor=north, inner sep=2pt, xshift=-5pt] at (axis cs:256,11.1354) {$-$29.5\%};
\nextgroupplot[width=0.33\textwidth, title={\textbf{(c)} 32B models}, ybar=0.6pt, /pgf/bar width=5pt, xmin=-0.6, xmax=5.6, xtick={0,...,5}, xticklabels={{Omni\\W3},{Omni\\W4},{AWQ\\W3},{AWQ\\W4},{Omni\\W3},{Omni\\W4}}, xticklabel style={align=center, font=\tiny}, ymin=-3.4, ymax=1.3, ytick={-3,-2,-1,0}, ylabel={PPL change (\%)}, legend style={font=\tiny, draw=none, fill=none, at={(0.99,0.03)}, anchor=south east, legend columns=1, /tikz/every even column/.append style={column sep=4pt}}, legend image code/.code={\draw[#1,draw=none] (0cm,-0.07cm) rectangle (0.2cm,0.07cm);}]
\addplot[fill=rbours, draw=none] coordinates {(0,-1.706) (1,0.258) (2,-3.017) (3,-0.762) (4,-2.327) (5,-0.105)};
\addplot[fill=rbc4, draw=none] coordinates {(0,-0.046) (1,-0.538) (2,-0.003) (3,-0.067) (4,-1.019) (5,-0.285)};
\legend{WT2, C4}
\draw[black!60, line width=0.4pt] (axis cs:-0.6,0) -- (axis cs:5.6,0);
\draw[black!30] (axis cs:3.5,-3.4) -- (axis cs:3.5,1.3);
\node[font=\tiny, text=black!70, anchor=north] at (axis cs:1.5,1.3) {Qwen2.5-32B};
\node[font=\tiny, text=black!70, anchor=north] at (axis cs:4.5,1.3) {Qwen3-32B};
\end{groupplot}
\end{tikzpicture}
\par
{\scriptsize\textcolor{rbbase}{\rule[1.5pt]{9pt}{1pt}}\,Base\qquad\textcolor{rbours}{\rule[1.5pt]{9pt}{1pt}}\,$+$Ours\qquad\tikz\draw[black!45,dash pattern=on 1.5pt off 1.5pt](0,0)--(0.35,0);\,Base at $g=64$\par}
\endgroup
\vspace{-2pt}
\caption{(a, b) WT2 perplexity of RTN at W3A16 versus group size. (c) Perplexity
change from refinement on the 32B models (lower is better).}
\label{fig:robustness}
\vspace{-6pt}
\end{figure*}

\paragraph{Coarser grids gain more, at modest cost.}
With RTN at W3A16, refinement lowers WT2 perplexity by 6.0\%, 6.9\%, and
11.4\% on Llama-2-7B and by 15.8\%, 20.4\%, and 29.5\% on Qwen3-8B as
the group size grows from 64 to 128 to 256 (Figure~\ref{fig:robustness}a,b).
At $g=256$, refined RTN beats unrefined RTN at $g=64$ on WT2 for both
models while storing a quarter of the scales, and every weight-penalty strength
$\nu\in\{0,0.2,\ldots,1\}$ improves both perplexities over RTN
(Appendix~\ref{app:hyperparameter-ablations}). $K=3$ captures most
gains (Figure~\ref{fig:refinement-budgets}) and adds 53--57\,s per
Llama-2-7B block and 36--38\,s per Qwen3-4B block (Table~\ref{tab:host-block-time}); a full Llama-2-7B
RTN run, including calibration, takes 32 min (Table~\ref{tab:group-size-time}).

\section{Conclusion}

Group-wise quantizers round weights onto a grid that is not refit to
the resulting codes. \method{} refits all group scales jointly and
moves codes in blocks on the new grid; code updates are more reliable
after the grid has moved. It lowers perplexity in 90 of 96 comparisons, most at
three bits, with no change to storage or inference cost. Its
guarantee is per layer, and gains are smallest for strong hosts at
four bits. Sharing QEP's target (Appendix~\ref{sec:target_dependence}) and
refitting zero points are next steps.

\section*{AI Use Statement}
Large language model assistants were used to edit the writing, check
the consistency of reported numbers across tables and text, and help
search the literature. All methods, experiments, results, and
conclusions are the authors' own, and the authors verified all
AI-assisted content.

\section*{Reproducibility Statement}
Section~\ref{sec:method} and Algorithm~\ref{alg:jarq} specify the
solver; Appendix~\ref{app:solver-details} gives all numerical settings
and stopping rules; Appendix~\ref{app:experimental-details} gives
models, calibration data, host settings, and evaluation protocols; and
Appendix~\ref{app:theoretical-analysis} contains all proofs of the results
stated in Section~\ref{sec:method}. Code will be available at
\url{https://github.com/Euphoria040201/JARQ}.

\bibliography{references}
\bibliographystyle{jarq}

\clearpage
\appendix

\section{Experimental Settings}
\label{app:experimental-details}

This appendix specifies the quantization settings, the evaluated model
states, and the quality and cost measurements.

\subsection{Calibration and Host Quantizers}
\label{app:strict-protocol}

The four hosts in Table~\ref{tab:main_ppl_matrix} use
128 WikiText-2 (WT2) training sequences of 2,048 tokens, seed 0,
asymmetric weight quantization with group size $g=128$,
3-bit or 4-bit weights, and 16-bit activations.
Quality is evaluated after reconstructing the quantized weights
in the model's floating-point dtype.

RTN rounds weights to the quantization grid.
GPTQ uses activation ordering and damping 0.1; refinement keeps its
group assignments in the reordered coordinates and restores the
original coordinates when writing back the weights.
OmniQuant uses 20 epochs of learnable weight clipping, with learnable
equivalent transformations disabled. AWQ uses its official scale
search and quantization routines with our calibration inputs.
The Qwen OmniQuant and AWQ runs use native scaled dot-product attention.

Refinement starts from the host's scales, codes, and zero points.
It updates scales and codes and keeps the zero points fixed.
The reconstruction objective is \eqref{eq:adaptive_grid_problem}
with $\nu=0.6$. After each layer is refined, its outputs supply the
inputs to the following layer. For every host, the final scales and
codes define the returned weights. Numerical settings and early
stopping are given in Appendix~\ref{app:solver-details}.

\subsection{Perplexity and Iteration Budget}
\label{app:ppl-protocol}

All main experiments use an iteration budget of $K=3$ with early
stopping. This choice follows the budget sweep on Llama-2 in
Figure~\ref{fig:refinement-budgets}, where three iterations capture
most of the WT2 perplexity reduction for RTN hosts, with smaller
gains from additional iterations.
Base denotes the host model ($K=0$), and $+$Ours denotes the model
refined with $K=3$. The QuaRot runs also use $K=3$ and
Algorithm~\ref{alg:jarq}, as detailed in Appendix~\ref{sec:variants}.
The Llama-2-7B and Qwen3-8B RTN entries in Table~\ref{tab:main_ppl_matrix}
are the $g=128$, $\nu=0.6$ states of the group-size and regularization
ablations (Appendix~\ref{app:hyperparameter-ablations}).

\paragraph{Llama-3-8B.}
We use the original base checkpoint, \texttt{meta-llama/Meta-Llama-3-8B}.
Table~\ref{tab:main_ppl_matrix} reports all four hosts at W3A16 and
W4A16. Within each configuration, WT2 and C4 are evaluated on the same
model state. Table~\ref{tab:downstream-selected-five} evaluates the
same states and the original FP16 model on five multiple-choice tasks.

\paragraph{Qwen3-4B.}
We use \texttt{Qwen/Qwen3-4B} with native scaled dot-product attention.
Table~\ref{tab:main_ppl_matrix} includes the FP16 reference and all four
hosts at W3A16 and W4A16. Perplexity is measured from the saved full-model
states after all 36 Transformer blocks have been processed.
The host and refinement streams propagate their own preceding outputs:
host calibration uses the host prefix, while refinement uses the refined
prefix. WT2 and C4 are evaluated on the same saved state for each entry.

\subsection{Downstream Evaluation}
\label{app:downstream-protocol}

\paragraph{Multiple-choice tasks.}
The five tasks are ARC-Easy and ARC-Challenge~\citep{clark2018thinksolvedquestionanswering}, MMLU~\citep{hendryckstest2021}, WinoGrande~\citep{sakaguchi2021winogrande}, and BoolQ~\citep{clark2019boolq}.
ARC-Easy and ARC-Challenge use zero-shot normalized accuracy;
WinoGrande and BoolQ use zero-shot accuracy; MMLU uses five-shot
accuracy. Avg.\ is their unweighted mean,
\[
\mathrm{Avg.}
=\frac{\mathrm{ARC\mbox{-}E}+\mathrm{ARC\mbox{-}C}
+\mathrm{MMLU}+\mathrm{WinoGrande}+\mathrm{BoolQ}}{5}
\]
\paragraph{Generative tasks.}
Table~\ref{tab:generative-all} (Appendix~\ref{app:gsm8k-results})
reports GSM8K accuracy for Llama-3-8B, Qwen2.5-7B, and Qwen3-8B
with GPTQ, OmniQuant, and AWQ at W3A16 and W4A16 (except GPTQ on
Llama-3-8B at W3A16), together with FP16 references.
Figure~\ref{fig:generative-w4} compares GSM8K and MATH-500 at W4A16
on the same models; its panel averages are unweighted means over the
three hosts for each model and task.
GSM8K~\citep{cobbe2021training} uses five-shot strict exact-match scoring and MATH-500~\citep{hendrycks2021math,lightman2024lets} uses
four-shot exact-match scoring.

\paragraph{Evaluation settings.}
All tasks use the full task splits, seed 0, completion prompts,
greedy generation, batch size 1, and context length 4,096.
For Llama-3-8B, completion prompts include the beginning-of-sequence
token. Base and $+$Ours ($K=0$ and $K=3$) use the same questions,
prompts, and scoring settings. FP16 rows use the original models with
the same task data and evaluation settings.

\subsection{Compatibility with Other Quantizers}
\label{app:qep-quarot-protocol}

\paragraph{QEP.}
We use the official QEP correction followed by GPTQ, with
128 WT2 calibration sequences of 2,048 tokens, $g=128$,
seed 0, native scaled dot-product attention, and W3A16/W4A16.
Refinement starts from the resulting quantized model and uses the
original full-precision weights in its reconstruction target;
QEP's correction step is performed before refinement.
The WT2 values in Table~\ref{tab:host-compatibility}(a) use $K=3$, as in
Table~\ref{tab:main_ppl_matrix}.

\paragraph{QuaRot.}
We use the official Hadamard rotation and GPTQ implementation on
Llama-2-7B, with 128 C4 calibration sequences of 2,048 tokens,
$g=128$, seed 0, and W3A16/W4A16.
The rotation and zero points stay fixed during refinement, and
weights and inputs are represented in the same rotated coordinates.
Refinement uses Algorithm~\ref{alg:jarq} with $K=3$ and $\nu=0.6$
(Appendix~\ref{sec:variants}). Evaluation uses the QuaRot evaluator.

\paragraph{OJBKQ.}
Table~\ref{tab:host-compatibility}(c) reports our reproduction of
OJBKQ~\citep{OJBKQ} on Qwen3-8B and Llama-3-8B at W3A16/W4A16, with
128 C4 calibration sequences of 2,048 tokens, $g=128$, and BF16
activations. On Qwen3-8B, the reproduced WT2 perplexities match the
published values (10.98 at W3A16 and 9.94 at W4A16; Table~1 of
\citet{OJBKQ}). Base is the reproduced OJBKQ model; $+$Ours refines
its scales and codes through the same host interface as above, with
$K=3$ and $\nu=0.6$. Both columns report measured WT2 perplexity.

\subsection{Efficiency Measurements}
\label{app:efficiency-protocol}

We measure the first Transformer block of Qwen3-8B at W3A16 and
$g=128$, including its seven attention and MLP linear layers.
Inputs are the same 128 WT2 calibration sequences of 2,048 tokens,
with seed 0. Each of three repetitions uses one NVIDIA A100.
Within a repetition, OmniQuant runs its 20-epoch weight clipping
once, and every budget $K>0$ refines that same result with the same
inputs and with early stopping enabled.
Table~\ref{tab:single-block-efficiency} and
Figure~\ref{fig:refinement-efficiency} report the results.

\begin{table*}[!htbp]
\centering
\caption{Time and memory for the first Qwen3-8B block at W3A16 with
the OmniQuant host. Times are mean $\pm$ sample standard deviation over
three repetitions; total time includes the host. Memory is the
time-weighted mean over both stages.}
\label{tab:single-block-efficiency}
\begingroup
\small
\setlength{\tabcolsep}{4pt}
\renewcommand{\arraystretch}{1.1}
\begin{tabular*}{\textwidth}{@{\extracolsep{\fill}}rccrrr@{}}
\toprule
$K$ & Refinement time (s) & Total time (s) & Overhead
& GPU memory (GiB) & CPU memory (GiB) \\
\midrule
0 & $0.00\pm0.00$ & $128.37\pm0.83$ & \textemdash & 2.07 & 6.63 \\
3 & $48.43\pm0.29$ & $176.80\pm1.12$ & $+37.73\%$ & 2.48 & 5.95 \\
5 & $57.61\pm0.41$ & $185.98\pm1.20$ & $+44.88\%$ & 2.53 & 5.86 \\
7 & $67.15\pm0.74$ & $195.52\pm1.50$ & $+52.31\%$ & 2.58 & 5.79 \\
10 & $80.65\pm0.55$ & $209.02\pm1.37$ & $+62.83\%$ & 2.64 & 5.69 \\
\bottomrule
\end{tabular*}
\endgroup
\end{table*}

\begin{figure}[!htbp]
\centering
\begin{minipage}[t]{0.48\textwidth}
\centering
\includegraphics[width=\linewidth]{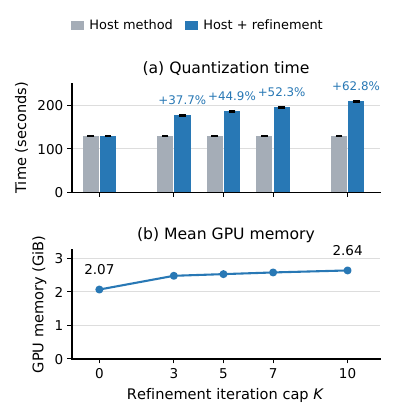}
\caption{Cost on the first Qwen3-8B block (OmniQuant, W3A16).
(a) Host time ($K=0$) versus host-plus-refinement time
(mean $\pm$ sample SD, three runs); labels give the overhead.
(b) Time-weighted mean GPU memory.
Values are in Table~\ref{tab:single-block-efficiency}.}
\label{fig:refinement-efficiency}
\end{minipage}\hfill
\begin{minipage}[t]{0.48\textwidth}
\centering
\includegraphics[width=\linewidth]{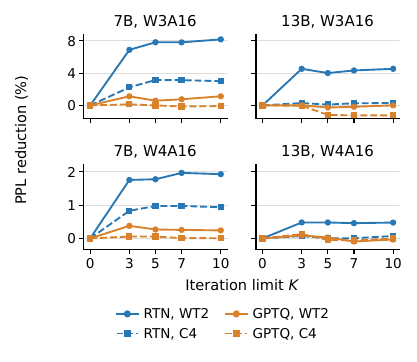}
\caption{Llama-2 perplexity versus iteration budget $K$.
Reduction is $100(1-\mathrm{PPL}_{K}/\mathrm{PPL}_{0})$;
higher is better.}
\label{fig:refinement-budgets}
\end{minipage}
\end{figure}

\paragraph{Time.}
The host timer includes parameter search and quantization. The
refinement timer includes computing calibration responses, QR
reduction, scale and code updates, proposal comparisons, output
propagation, and synchronization. Model loading, input capture,
warmup, serialization, and result validation are excluded.
For repetition $r$, total time is
$T_{\mathrm{total},r}=T_{\mathrm{host},r}+T_{\mathrm{ref},r}$, and
overhead is $100\,\overline{T}_{\mathrm{ref}}/\overline{T}_{\mathrm{host}}$.
These timings measure quantization cost, not inference latency.

\paragraph{Memory.}
GPU memory is allocated tensor memory; CPU memory is process
resident memory. For repetition $r$, we report the time-weighted
mean over the two stages,
\begin{equation}
\overline{M}_{\mathrm{comp},r}=
\frac{T_{\mathrm{host},r}\overline{M}_{\mathrm{host},r}
      +T_{\mathrm{ref},r}\overline{M}_{\mathrm{ref},r}}
     {T_{\mathrm{host},r}+T_{\mathrm{ref},r}},
\label{eq:composed-memory}
\end{equation}
averaged over the three repetitions. The CPU mean can decrease with
$K$ when refinement spends more time in a lower-memory stage.

\paragraph{Cost per host.}
Table~\ref{tab:host-block-time} measures every host on Llama-2-7B and
Qwen3-4B at W3A16, $g=128$, and $K=3$, over the first Transformer block
and over the first two blocks. Each entry is the median of the completed
repetitions (one or two per entry). Refinement adds 53--57\,s per
Llama-2-7B block and 36--38\,s per Qwen3-4B block regardless of the host,
and the added time for two blocks is about twice that for one. At this
rate, refining all 32 blocks of Llama-2-7B takes about 30 minutes,
consistent with the full-model RTN runs in Table~\ref{tab:group-size-time}.

\begin{table}[!htbp]
\centering
\caption{Quantization time (s) of each host alone and with refinement
($+$Ours) at W3A16 over the first block and the first two blocks;
$\Delta$ is the time added by refinement. Medians of the completed
repetitions (one or two per entry).}
\label{tab:host-block-time}
\begingroup
\small
\setlength{\tabcolsep}{5pt}
\begin{tabular}{llrrrrrr}
\toprule
& & \multicolumn{3}{c}{One block} & \multicolumn{3}{c}{Two blocks} \\
\cmidrule(lr){3-5}\cmidrule(lr){6-8}
Model & Host & Host & $+$Ours & $\Delta$ & Host & $+$Ours & $\Delta$ \\
\midrule
\multirow{4}{*}{Llama-2-7B} & RTN & 4.9 & 61.5 & 56.6 & 7.8 & 119.8 & 112.0 \\
 & GPTQ & 33.0 & 86.6 & 53.6 & 64.0 & 170.1 & 106.1 \\
 & AWQ & 28.4 & 81.5 & 53.1 & 49.3 & 154.0 & 104.7 \\
 & OmniQuant & 69.7 & 126.6 & 56.9 & 135.6 & 244.9 & 109.3 \\
\midrule
\multirow{4}{*}{Qwen3-4B} & RTN & 2.8 & 40.3 & 37.5 & 4.8 & 79.4 & 74.6 \\
 & GPTQ & 25.0 & 61.0 & 36.0 & 47.2 & 119.2 & 72.0 \\
 & AWQ & 17.1 & 53.0 & 35.9 & 32.2 & 105.0 & 72.8 \\
 & OmniQuant & 45.4 & 83.1 & 37.7 & 89.4 & 163.8 & 74.4 \\
\bottomrule
\end{tabular}
\endgroup
\end{table}

\section{Additional Results}
\label{app:supplementary-results}

This appendix reports perplexity on the 32B models and the full GSM8K
results. Figure~\ref{fig:qwen32-perplexity} summarizes the 32B models
relative to their FP16 references.

\begin{figure}[!htbp]
\centering
\begin{minipage}[t]{0.48\textwidth}
\centering
\includegraphics[width=\linewidth]{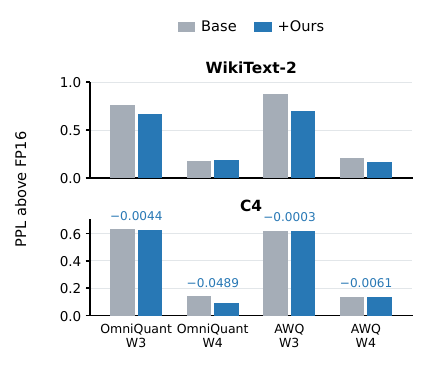}
\par\smallskip{\small (a) Qwen2.5-32B: OmniQuant and AWQ}
\end{minipage}\hfill
\begin{minipage}[t]{0.48\textwidth}
\centering
\includegraphics[width=\linewidth]{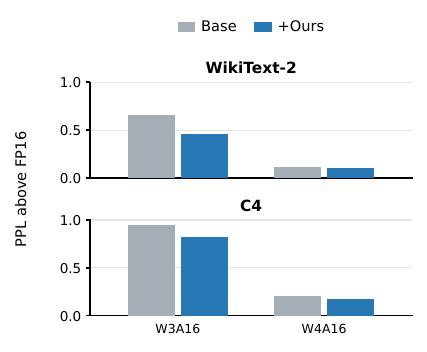}
\par\smallskip{\small (b) Qwen3-32B: OmniQuant}
\end{minipage}
\caption{Perplexity gap to FP16, $\mathrm{PPL}-\mathrm{PPL}_{\mathrm{FP16}}$,
on the 32B models (lower is better). Gray is Base ($K=0$); blue is
$+$Ours ($K=3$). Axes start at zero, with a tighter range for C4 in (a).
Labels give $\mathrm{PPL}_{+\mathrm{Ours}}-\mathrm{PPL}_{\mathrm{Base}}$.
Absolute values are in Tables~\ref{tab:qwen32-ppl}
and~\ref{tab:qwen3-32b-omniquant-ppl}.}
\label{fig:qwen32-perplexity}
\end{figure}

\subsection{Qwen2.5-32B Perplexity}
\label{app:qwen32-budgets}

We use the calibration and host settings of
Appendix~\ref{app:strict-protocol}. Evaluation covers all complete,
nonoverlapping 2,048-token WT2 test windows and 256 C4 validation
windows sampled with seed 0. The evaluated weights use FP16, with
native scaled dot-product attention and FP32 logits and loss.
WT2 and C4 are evaluated on the same model state within each
configuration.

Table~\ref{tab:qwen32-budgets} compares each host with its refinement
at $K=3$ and with the FP16 reference. WT2 perplexity decreases in
three of four configurations, by up to 3.02\%, and C4 perplexity in all
four, by up to 0.54\%.

\begin{table*}[!htbp]
\centering
\caption{Qwen2.5-32B perplexity (lower is better). Base uses $K=0$ and
$+$Ours uses $K=3$. Bold marks the better value in each Base/$+$Ours
pair (compared before rounding).}
\label{tab:qwen32-budgets}
\label{tab:qwen32-ppl}
\begingroup
\small
\setlength{\tabcolsep}{4pt}
\renewcommand{\arraystretch}{1.1}
\begin{tabular*}{\textwidth}{@{\extracolsep{\fill}}llrrrr@{}}
\toprule
& & \multicolumn{2}{c}{WT2} & \multicolumn{2}{c}{C4} \\
\cmidrule(lr){3-4}\cmidrule(lr){5-6}
Host & W/A & Base & $+$Ours & Base & $+$Ours \\
\midrule
FP16 & W16A16 & \multicolumn{2}{c}{\ppl{5.0176}} & \multicolumn{2}{c}{\ppl{8.9503}} \\
\midrule
OmniQuant & W3A16 & \ppl{5.7804} & \textbf{\ppl{5.6818}} & \ppl{9.5809} & \textbf{\ppl{9.5765}} \\
& W4A16 & \textbf{\ppl{5.1951}} & \ppl{5.2085} & \ppl{9.0920} & \textbf{\ppl{9.0431}} \\
\midrule
AWQ & W3A16 & \ppl{5.8960} & \textbf{\ppl{5.7181}} & \ppl{9.5722} & \textbf{\ppl{9.5719}} \\
& W4A16 & \ppl{5.2222} & \textbf{\ppl{5.1824}} & \ppl{9.0894} & \textbf{\ppl{9.0833}} \\
\bottomrule
\end{tabular*}
\endgroup
\end{table*}

\subsection{Qwen3-32B Perplexity}
\label{app:qwen3-32b-omniquant}

We quantize the Qwen3-32B dense model with the OmniQuant host at W3A16
and W4A16.
Table~\ref{tab:qwen3-32b-omniquant-ppl} compares each host with its
refinement at $K=3$. Perplexity decreases on WT2 and C4 at both bit
widths, with larger reductions at W3A16.

\begin{table}[!htbp]
\centering
\caption{Qwen3-32B perplexity with the OmniQuant host (lower is better).
Base uses $K=0$ and $+$Ours uses $K=3$. Bold marks the better value in
each Base/$+$Ours pair (compared before rounding).}
\label{tab:qwen3-32b-omniquant-ppl}
\begingroup
\small
\setlength{\tabcolsep}{4pt}
\renewcommand{\arraystretch}{1.1}
\begin{tabular*}{\textwidth}{@{\extracolsep{\fill}}lrrrr@{}}
\toprule
& \multicolumn{2}{c}{WT2} & \multicolumn{2}{c}{C4} \\
\cmidrule(lr){2-3}\cmidrule(lr){4-5}
W/A & Base & $+$Ours & Base & $+$Ours \\
\midrule
W3A16 & \ppl{8.2562} & \textbf{\ppl{8.0641}} & \ppl{11.7171} & \textbf{\ppl{11.5977}} \\
W4A16 & \ppl{7.7106} & \textbf{\ppl{7.7025}} & \ppl{10.9813} & \textbf{\ppl{10.9500}} \\
\bottomrule
\end{tabular*}
\endgroup
\end{table}

\subsection{Calibration Corpus}
\label{app:calibration-corpus}

Table~\ref{tab:calibration-corpus} varies only the calibration corpus
for GPTQ on Qwen3-4B at W3A16, using 128 sequences of 2,048 tokens from
WT2, C4, or the Pile~\citep{gao2020pile}. For each corpus, \method{}
($K=3$) starts from the GPTQ model calibrated on that corpus and uses
the same calibration data. GSM8K uses five-shot strict exact match and
MATH-500 four-shot Minerva exact match, with the settings of
Appendix~\ref{app:downstream-protocol}. Refinement raises GSM8K accuracy
for all three corpora, by 6.52, 6.21, and 2.04 points. MATH-500 accuracy
rises with C4 and Pile calibration and falls by 1.2 points with WT2.

\begin{table}[!htbp]
\centering
\caption{GSM8K and MATH-500 accuracy (\%; higher is better) of Qwen3-4B
with GPTQ at W3A16 under three calibration corpora. Bold marks the better
value in each Base/$+$Ours pair.}
\label{tab:calibration-corpus}
\begingroup
\small
\setlength{\tabcolsep}{6pt}
\begin{tabular}{lcccc}
\toprule
& \multicolumn{2}{c}{GSM8K} & \multicolumn{2}{c}{MATH-500} \\
\cmidrule(lr){2-3}\cmidrule(lr){4-5}
Calibration & Base & $+$Ours & Base & $+$Ours \\
\midrule
WT2  & 59.67 & \textbf{66.19} & \textbf{29.4} & 28.2 \\
C4   & 61.49 & \textbf{67.70} & 29.8 & \textbf{30.2} \\
Pile & 69.98 & \textbf{72.02} & 42.4 & \textbf{45.6} \\
\bottomrule
\end{tabular}
\endgroup
\end{table}

\subsection{GSM8K Accuracy}
\label{app:gsm8k-results}

Table~\ref{tab:generative-all} reports GSM8K accuracy for Llama-3-8B,
Qwen2.5-7B, and Qwen3-8B with the GPTQ, OmniQuant, and AWQ hosts at
W3A16 and W4A16; the W4A16 pairs are plotted in
Figure~\ref{fig:generative-w4}. Refinement raises accuracy in 14 of the
17 reported pairs: six of eight at W3A16 (one tie) and eight of nine at
W4A16. Evaluation settings are in
Appendix~\ref{app:downstream-protocol}.

\begin{table}[!htbp]
\centering
\caption{GSM8K accuracy (\%; higher is better). Gray marks FP16
references and blue shading marks $+$Ours. Bold marks the better value
in each Base/$+$Ours pair; ties are not bolded. \textemdash{} denotes a
configuration not reported.}
\label{tab:generative-all}
\begingroup
\small
\setlength{\tabcolsep}{4pt}
\renewcommand{\arraystretch}{1.1}
\begin{tabular*}{\textwidth}{
  @{\extracolsep{\fill}}lll
  r>{\columncolor{jarqRefined}[\tabcolsep][0pt]}r@{}}
\toprule
Model & Host & W/A & Base & $+$Ours \\
\midrule
\multirow{7}{*}{Llama-3-8B}
& \cellcolor{jarqFP16}FP16 & \cellcolor{jarqFP16}W16A16
& \multicolumn{2}{>{\columncolor{jarqFP16}}c}{47.92} \\
\cmidrule(lr){2-5}
& \multirow{2}{*}{GPTQ} & W3A16 & \textemdash & \textemdash \\
&  & W4A16 & 42.53 & \textbf{42.61} \\
& \multirow{2}{*}{OmniQuant} & W3A16 & 20.17 & \textbf{22.29} \\
&  & W4A16 & 42.68 & \textbf{43.67} \\
& \multirow{2}{*}{AWQ} & W3A16 & 21.30 & 21.30 \\
&  & W4A16 & 42.61 & \textbf{43.37} \\
\midrule
\multirow{7}{*}{Qwen2.5-7B}
& \cellcolor{jarqFP16}FP16 & \cellcolor{jarqFP16}W16A16
& \multicolumn{2}{>{\columncolor{jarqFP16}}c}{83.70} \\
\cmidrule(lr){2-5}
& \multirow{2}{*}{GPTQ} & W3A16 & 69.22 & \textbf{69.37} \\
&  & W4A16 & \textbf{80.89} & 79.91 \\
& \multirow{2}{*}{OmniQuant} & W3A16 & 60.05 & \textbf{68.61} \\
&  & W4A16 & 80.06 & \textbf{81.65} \\
& \multirow{2}{*}{AWQ} & W3A16 & 67.17 & \textbf{70.74} \\
&  & W4A16 & 78.17 & \textbf{81.27} \\
\midrule
\multirow{7}{*}{Qwen3-8B}
& \cellcolor{jarqFP16}FP16 & \cellcolor{jarqFP16}W16A16
& \multicolumn{2}{>{\columncolor{jarqFP16}}c}{90.30} \\
\cmidrule(lr){2-5}
& \multirow{2}{*}{GPTQ} & W3A16 & 75.21 & \textbf{80.67} \\
&  & W4A16 & 88.63 & \textbf{88.70} \\
& \multirow{2}{*}{OmniQuant} & W3A16 & 81.05 & \textbf{81.20} \\
&  & W4A16 & 88.78 & \textbf{89.16} \\
& \multirow{2}{*}{AWQ} & W3A16 & \textbf{82.56} & 81.43 \\
&  & W4A16 & 88.17 & \textbf{90.30} \\
\bottomrule
\end{tabular*}
\endgroup
\end{table}

\section{Component Ablations}
\label{app:refinement-mechanisms}

\subsection{Setup}

We test 16 linear layers from Llama-2-7B and Qwen3-8B with the RTN and
GPTQ hosts, sampling 32 output channels per layer at W3A16 and
$g=128$. Alternating refinement uses an iteration budget of $K=3$.
Inputs come from the full-precision model and stay fixed during each
test. Separate held-out inputs measure reconstruction beyond the
samples used for fitting.

For each layer, we sum the objective over the sampled channels before
computing relative differences, and we report medians across layers.
Timings use one synchronized FP64 measurement after warmup on an
NVIDIA RTX A6000; scale-update timers include constructing the scale
system. Table~\ref{tab:mechanism-summary} collects the results of the
three tests below.

\subsection{Joint Scale Fitting}
\label{app:mechanism-scales}

With codes fixed, independent fitting matches each group's contribution
separately, coordinate descent updates one scale at a time against the
full residual, and joint fitting solves for all scales at once.
The reported gap is
$100(L_{\mathrm{method}}-L_{\mathrm{joint}})/L_{\mathrm{base}}$,
where $L_{\mathrm{base}}$ is the objective before scale fitting.
The first block of Table~\ref{tab:mechanism-summary} compares the
reconstruction error and solve time of the three fits with the same
fixed codes.

\subsection{Multi-Code Updates}
\label{app:mechanism-codes}

Each case is one output channel and one group of 128 weights.
With the scales and the other groups fixed, we run single-code updates
until no one-code change strictly lowers the objective; a separate FP64
check tests every allowed code for every coordinate. A Babai proposal
is then computed from this state.
The second block of Table~\ref{tab:mechanism-summary} reports the
number and share of accepted proposals and the minimum and median
number of codes they change.

\subsection{Fixed-Grid versus Joint Refinement}
\label{app:mechanism-alternation}

Fixed keeps the host scales and updates only codes; Joint alternates
scale and code updates using Algorithm~\ref{alg:jarq}. Both start from
the same weights. Fixed receives the time used by Joint and returns its
last complete sweep within that time; shared factorization and
diagnostic work are excluded.
The last block of Table~\ref{tab:mechanism-summary} reports the median
held-out error reduction and the number of layers improved.

\subsection{Full-Model Component Ablation}
\label{app:mechanism-full-model}

Table~\ref{tab:full-model-components} applies each update alone to the
full model on Llama-3-8B and Qwen3-4B, with the settings of
Table~\ref{tab:main_ppl_matrix}. \emph{Codes only} runs one code sweep
of Algorithm~\ref{alg:jarq} on the host grid; \emph{Scales only} runs
one joint scale fit with the host codes. Both variants start
independently from the same host state and propagate their own outputs
block by block. With the codes fixed, the
joint fit has a unique solution, so repeating it changes nothing; the
iteration budget of full \method{} is studied in
Figure~\ref{fig:refinement-budgets}.

Each update alone is effective for the RTN host: on Llama-3-8B at
W3A16, codes only lowers WT2 perplexity by 27.9\% and scales only by
28.8\%, so the two partly correct the same error. Their interaction
depends on the grid. Codes only raises perplexity in 5 of the 16
comparisons, including WT2 for both GPTQ settings at W3A16, whose codes
already come from a nearest-plane recursion on a related lattice. Full
\method{}, which alternates both updates, is below Scales only in all 16
comparisons: by 4.2\% and 9.9\% on WT2 for RTN at W3A16, and by at most
0.3\% on WT2 at W4A16. Full \method{} gives the lowest WT2
perplexity in 7 of the 8 settings; on Llama-3-8B with GPTQ at W3A16, no
variant improves WT2 over the host, while full \method{} gives the
lowest C4 perplexity.

\begin{table}[!htbp]
\centering
\caption{Full-model component ablation (WT2/C4 perplexity; lower is
better). Bold marks the lowest value per setting and corpus.}
\label{tab:full-model-components}
\begingroup
\small
\setlength{\tabcolsep}{4pt}
\begin{tabular}{lllcccc}
\toprule
Model & Host & W/A & Host & Codes only & Scales only & Full \method{} \\
\midrule
\multirow{4}{*}{Llama-3-8B} & \multirow{2}{*}{RTN} & W3A16 & 12.048/16.444 & 8.689/14.171 & 8.580/13.311 & \textbf{8.219}/\textbf{12.667} \\
 &  & W4A16 & 6.727/9.636 & 6.586/9.605 & 6.490/9.583 & \textbf{6.481}/\textbf{9.532} \\
 & \multirow{2}{*}{GPTQ} & W3A16 & \textbf{7.618}/13.242 & 7.714/12.318 & 7.784/12.628 & 7.649/\textbf{12.011} \\
 &  & W4A16 & 6.423/\textbf{9.386} & 6.419/9.416 & 6.417/9.411 & \textbf{6.416}/9.408 \\
\midrule
\multirow{4}{*}{Qwen3-4B} & \multirow{2}{*}{RTN} & W3A16 & 22.447/27.013 & 21.205/27.078 & 17.968/22.316 & \textbf{16.194}/\textbf{21.625} \\
 &  & W4A16 & 15.968/18.291 & 14.629/\textbf{17.924} & 14.598/17.965 & \textbf{14.555}/17.930 \\
 & \multirow{2}{*}{GPTQ} & W3A16 & 15.153/\textbf{18.623} & 15.903/18.953 & 15.091/18.859 & \textbf{15.089}/18.791 \\
 &  & W4A16 & 13.675/16.980 & 13.665/16.932 & 13.570/16.923 & \textbf{13.553}/\textbf{16.917} \\
\bottomrule
\end{tabular}
\endgroup
\end{table}

\section{Layer-Wise Error Profiles}
\label{app:layer-errors}

This appendix tracks how refinement changes errors inside the model, from
single linear modules to the output of each Transformer block.

\subsection{Setup}
\label{app:layer-protocol}

We profile every block and linear module of Llama-2-7B (32 blocks) and
Qwen3-8B (36 blocks) with the RTN, GPTQ, AWQ, and OmniQuant hosts at W3A16
and W4A16, comparing each host with host$+$\method{} ($K=3$, $\nu=0.6$,
$g=128$). Calibration uses the 128 WT2 training sequences of the main
experiments. Errors are measured on held-out data: 32 sequences of 2,048
tokens each from the WT2 test set and the C4 validation set, over all
tokens and output channels. The 16 host, bit-width, and model
configurations cover 544 block states and 3,808 matrices, giving 7,616
paired module measurements over the two corpora.

Each portrait page contains one $7\times8$ linear-module grid for one
model and metric.
Rows are the seven projections of a block (\texttt{q}, \texttt{k},
\texttt{v}, \texttt{o}, \texttt{gate}, \texttt{up}, \texttt{down}); the
four left columns are the hosts on WT2 and the four right columns the
same hosts on C4. Blue and orange mark W3A16 and W4A16; solid and dashed
lines mark the host and host$+$\method{}. The horizontal axis is the block
index; all vertical axes are logarithmic with limits set per panel.

\subsection{Metrics}
\label{app:layer-metrics}

Squared errors are summed over all tokens and output channels; NMSE
divides the total squared error by the total energy of the full-precision
reference. Products are computed in FP32 from FP16 activations and summed
in FP64.

\paragraph{Block-output NMSE.}
The full-precision and quantized models each propagate the same token
sequences, and
\begin{equation}
 E_{\ell,s}^{\mathrm{block}}
 =\frac{\|\mathbf H_{\ell}^{s}-\mathbf H_{\ell}^{\mathrm{FP}}\|_F^2}
 {\|\mathbf H_{\ell}^{\mathrm{FP}}\|_F^2}
 \label{eq:layer-block-nmse}
\end{equation}
compares the hidden state after block $\ell$, including normalization,
nonlinearities, and residual connections, for state $s$ (host or
host$+$\method{}); the last block is measured before the final
normalization and output head.

\paragraph{Local linear NMSE.}
Both states receive the same full-precision input to module $a$:
\begin{equation}
 E_{\ell,a,s}^{\mathrm{local}}
 =\frac{\|\mathbf X_{\ell,a}^{\mathrm{ref}}
 (\mathbf W_{\ell,a}^{s}-\mathbf W_{\ell,a}^{\mathrm{ref}})\|_F^2}
 {\|\mathbf Y_{\ell,a}^{\mathrm{ref}}\|_F^2},\qquad
 \mathbf Y^{\mathrm{ref}}=\mathbf X^{\mathrm{ref}}\mathbf W^{\mathrm{ref}}
 +\mathbf b^{\mathrm{ref}}.
 \label{eq:layer-local-nmse}
\end{equation}
This isolates the module's own weight error. For AWQ, the reference block
uses the unclipped, transformed full-precision weights fed by the original
full-precision block inputs, so both states are compared in the same
coordinates.

\paragraph{Propagated output SSE.}
Each state feeds the module its own propagated input:
\begin{equation}
 S_{\ell,a,s}^{\mathrm{out}}
 =\|\mathbf X_{\ell,a}^{\mathrm{FP}}\mathbf W_{\ell,a}^{\mathrm{FP}}
 -\mathbf X_{\ell,a}^{s}\mathbf W_{\ell,a}^{s}\|_F^2 ,
 \label{eq:layer-output-sse}
\end{equation}
so the error includes the effect of all preceding modules.

\paragraph{Combined error.}
As a supplementary metric, we add a weight penalty with coefficient 0.6,
\begin{equation}
 C_{\ell,a,s}=S_{\ell,a,s}^{\mathrm{out}}
 +0.6\|\mathbf W_{\ell,a}^{\mathrm{FP}}-\mathbf W_{\ell,a}^{s}\|_F^2,
 \qquad C_{\ell,s}^{\mathrm{block}}=\sum_{a}C_{\ell,a,s},
 \label{eq:layer-combined}
\end{equation}
where the sum runs over the seven modules of block $\ell$ and the weight
penalty is counted once per matrix and corpus. Unlike the
refinement objective \eqref{eq:adaptive_grid_problem}, which uses the
calibration inputs $\widetilde{\mathbf X}$ and coefficient $\nu^2=0.36$,
these metrics use held-out inputs. For AWQ, propagated quantities also contain
the difference between folded and original coordinates, so their raw
magnitudes are not comparable across hosts or modules.

\subsection{Results}
\label{app:layer-findings}

Table~\ref{tab:layer-error-counts} counts the pairs in which
host$+$\method{} has strictly lower error than the host.

\begin{table}[!htbp]
\centering
\caption{Pairs in which refinement lowers the error. The first three
columns count all modules and blocks of both models, bit widths, and
corpora; the last two count the final block only.}
\label{tab:layer-error-counts}
\small
\setlength{\tabcolsep}{4pt}
\renewcommand{\arraystretch}{1.1}
\begin{tabular}{lrrrrr}
\toprule
Host & \shortstack{Local\\NMSE} & \shortstack{Linear\\output SSE}
& \shortstack{Linear\\combined} & \shortstack{Final block\\NMSE}
& \shortstack{Final block\\combined} \\
\midrule
RTN & 1,818/1,904 & 1,888/1,904 & 1,888/1,904 & 8/8 & 8/8 \\
GPTQ & 1,739/1,904 & 1,447/1,904 & 1,443/1,904 & 7/8 & 6/8 \\
AWQ & 1,721/1,904 & 1,648/1,904 & 1,644/1,904 & 6/8 & 8/8 \\
OmniQuant & 1,794/1,904 & 1,456/1,904 & 1,456/1,904 & 2/8 & 3/8 \\
\bottomrule
\end{tabular}

\end{table}

\paragraph{Local gains hold on held-out data.}
Local NMSE decreases in 7,072 of 7,616 module pairs (92.9\%) and in at
least 90\% of pairs for every host, including OmniQuant (Figures~\ref{fig:layer-local-nmse-llama2-7b}
and~\ref{fig:layer-local-nmse-qwen3-8b}). The reduction in the calibration
objective therefore carries over to held-out inputs.

\paragraph{Propagation depends on the host.}
At the final block, block-output NMSE decreases in 23 of 32 comparisons:
8/8 for RTN, 7/8 for GPTQ, 6/8 for AWQ, and 2/8 for OmniQuant. A likely reason is that
OmniQuant optimizes block outputs directly, and refinement toward the
layer target can undo part of
that compensation. For Llama-2-7B with OmniQuant at W3A16 on C4, all seven
local NMSEs of the last block decrease, and \texttt{down\_proj} falls by
47.51\% (Figure~\ref{fig:layer-local-nmse-llama2-7b}), while its propagated
SSE rises by 37.72\% (Figure~\ref{fig:layer-propagated-output-sse-llama2-7b})
and the final block-output NMSE rises from 0.08358 to 0.10799.

\paragraph{Two propagation-sensitive modules.}
All 14 of the 7,616 module pairs in which refinement more than doubles the
propagated SSE come from the \texttt{down\_proj} of one early block per
model: block~2 of Llama-2-7B and block~7 of Qwen3-8B. There, the host's
propagated SSE is 77--1,133 and 6--34 times that of the preceding block,
and refinement raises it in 9 and 13 of the 16 host, bit-width, and corpus
settings. At block~7 of Qwen3-8B, local NMSE still falls in all 13 of
those settings: the refined weights fit better, but these modules amplify
the small input drift left by the refined preceding blocks. The main target
$\widetilde{\mathbf{X}}\mathbf{W}$ ignores such drift by design; the
propagation-corrected target of Appendix~\ref{sec:target_dependence}
accounts for it.

\paragraph{AWQ coordinates.}
For Qwen3-8B with AWQ at W3A16 on C4, the propagated SSE of the third
block's \texttt{up\_proj} changes from $3.30\times10^{14}$ to
$3.32\times10^{14}$ (Figure~\ref{fig:layer-propagated-output-sse-qwen3-8b}),
while its aligned local NMSE falls by 98.75\%
(Figure~\ref{fig:layer-local-nmse-qwen3-8b}). Large raw SSE values in AWQ
panels reflect its folded coordinates rather than larger quantization
error.

\clearpage
\begin{figure}[p]
\centering
\includegraphics[width=\linewidth]{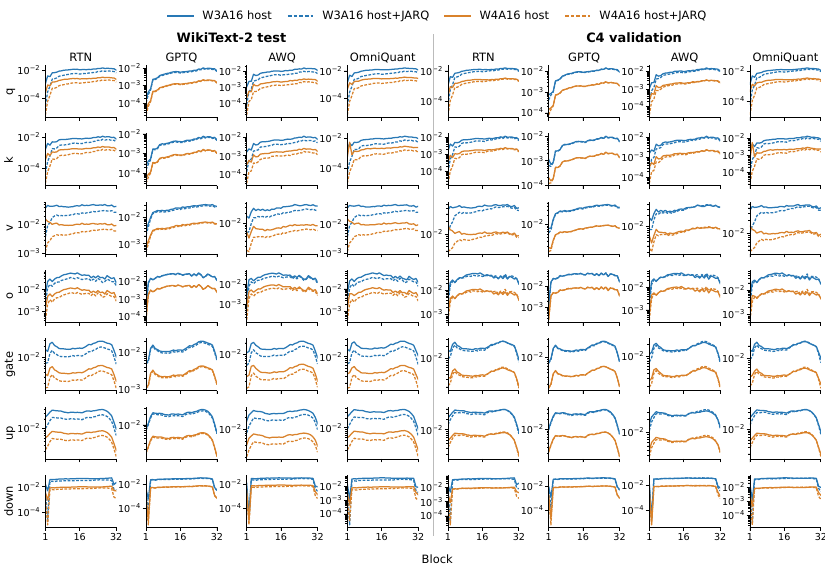}
\caption{Local linear NMSE \eqref{eq:layer-local-nmse} for Llama-2-7B.
Rows are the seven projections; the left and right halves are WT2 and C4.
At the final block, refinement lowers it in 104/112 module pairs (RTN 26/28,
GPTQ 24/28, AWQ 26/28, OmniQuant 28/28).}
\label{fig:layer-local-nmse-llama2-7b}

\vspace{8pt}
\includegraphics[width=\linewidth]{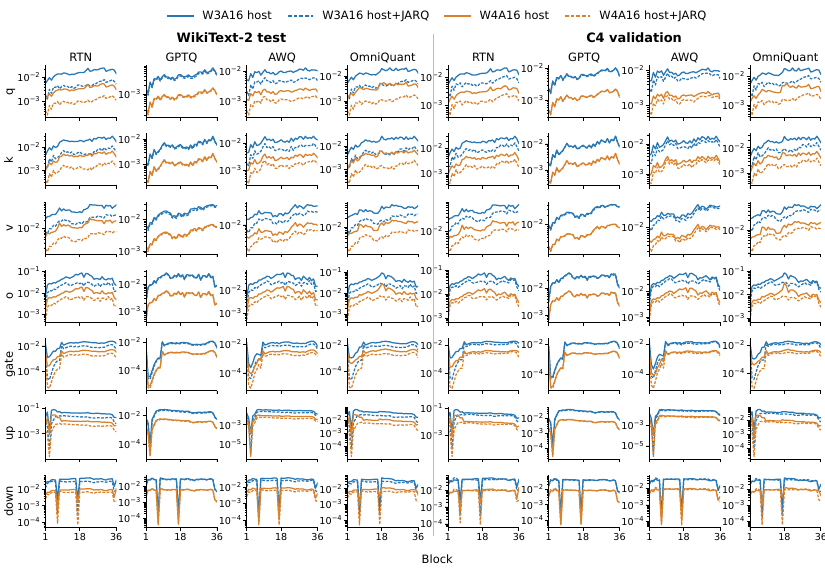}
\caption{Local linear NMSE \eqref{eq:layer-local-nmse} for Qwen3-8B, as in
Figure~\ref{fig:layer-local-nmse-llama2-7b}. At the final block, refinement
lowers it in 104/112 module pairs (RTN 28/28, GPTQ 20/28, AWQ 28/28,
OmniQuant 28/28).}
\label{fig:layer-local-nmse-qwen3-8b}
\end{figure}

\begin{figure}[p]
\centering
\includegraphics[width=\linewidth]{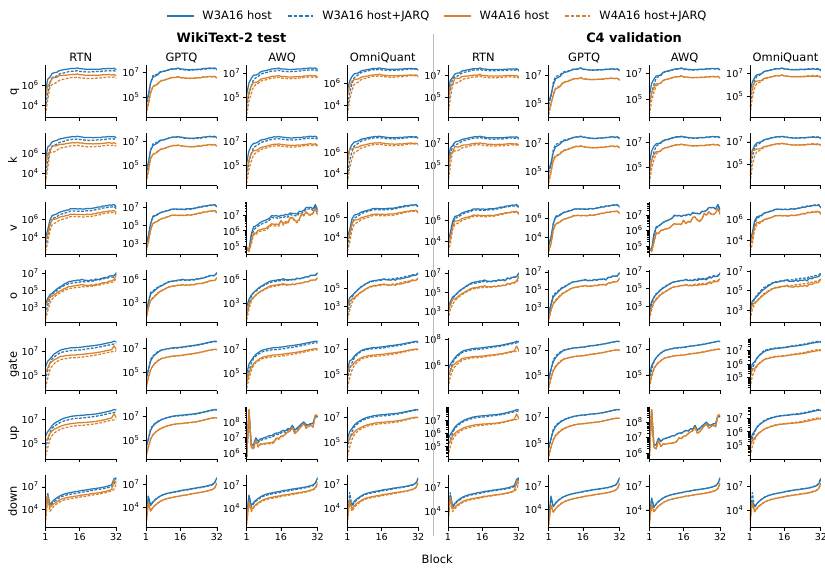}
\caption{Propagated output SSE \eqref{eq:layer-output-sse} for Llama-2-7B,
laid out as in Figure~\ref{fig:layer-local-nmse-llama2-7b}. At the final
block, refinement lowers it in 75/112 module pairs (RTN 27/28, GPTQ 16/28,
AWQ 21/28, OmniQuant 11/28).}
\label{fig:layer-propagated-output-sse-llama2-7b}

\vspace{8pt}
\includegraphics[width=\linewidth]{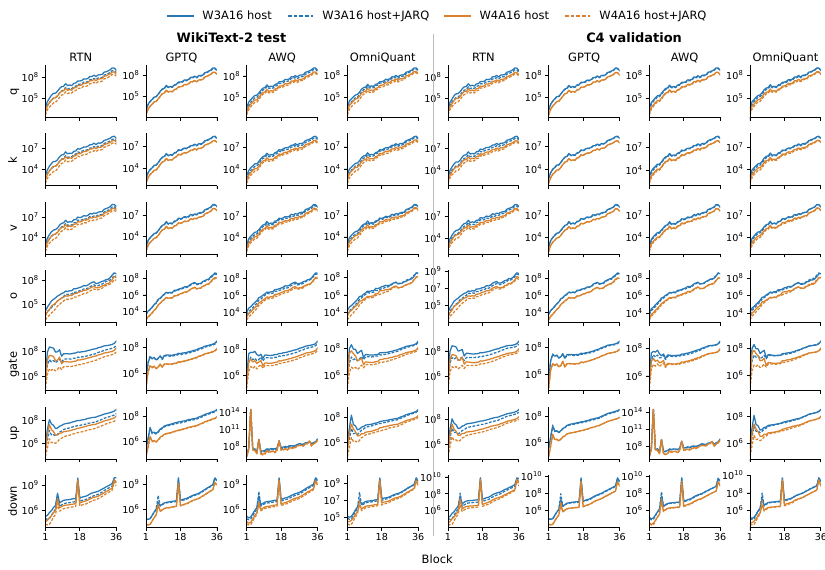}
\caption{Propagated output SSE \eqref{eq:layer-output-sse} for Qwen3-8B,
laid out as in Figure~\ref{fig:layer-local-nmse-llama2-7b}. At the final
block, refinement lowers it in 105/112 module pairs (RTN 28/28, GPTQ 28/28,
AWQ 24/28, OmniQuant 25/28).}
\label{fig:layer-propagated-output-sse-qwen3-8b}
\end{figure}

\begin{figure}[p]
\centering
\includegraphics[width=\linewidth]{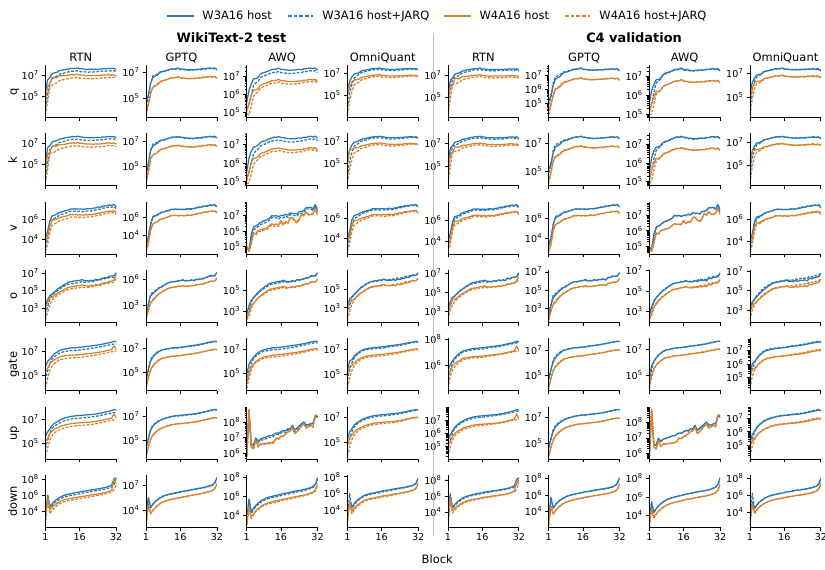}
\caption{Per-module combined error \eqref{eq:layer-combined} for
Llama-2-7B, laid out as in Figure~\ref{fig:layer-local-nmse-llama2-7b}.
At the final block, refinement lowers it in 75/112 module pairs (RTN 27/28,
GPTQ 16/28, AWQ 21/28, OmniQuant 11/28).}
\label{fig:layer-combined-error-llama2-7b}

\vspace{8pt}
\includegraphics[width=\linewidth]{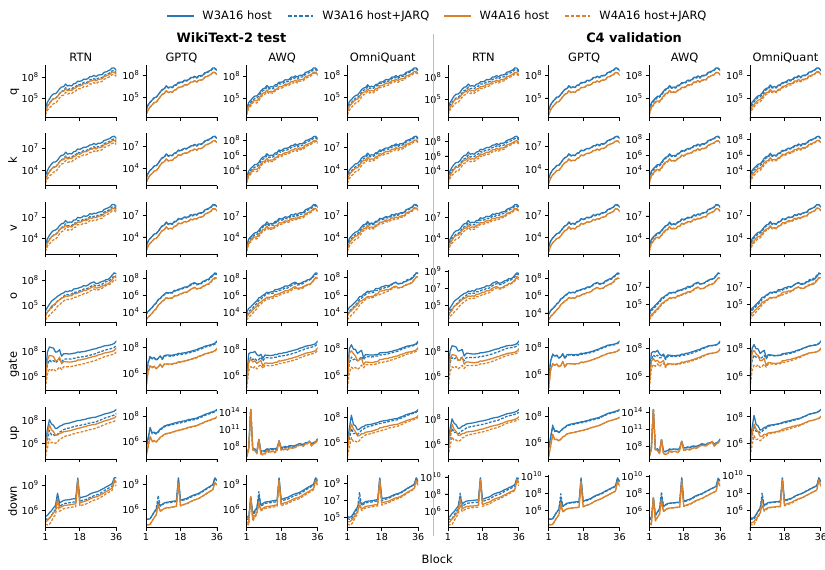}
\caption{Per-module combined error \eqref{eq:layer-combined} for
Qwen3-8B, laid out as in Figure~\ref{fig:layer-local-nmse-llama2-7b}.
At the final block, refinement lowers it in 105/112 module pairs (RTN 28/28,
GPTQ 28/28, AWQ 24/28, OmniQuant 25/28).}
\label{fig:layer-combined-error-qwen3-8b}
\end{figure}

\clearpage

\section{Implementation Details}
\label{app:solver-details}
\label{app:solver-ablation}
\label{sec:safeguards}

\subsection{Reconstruction after QR Reduction}

For the thin QR factorization $\mathbf{A}=\mathbf{U}\mathbf{R}$ and
$\overline{\mathbf{Y}}=\mathbf{U}^{\top}\mathbf{Y}$, orthogonality gives
\begin{equation}
\|\mathbf{Y}-\mathbf{A}\widehat{\mathbf{W}}\|_F^2
=\|\overline{\mathbf{Y}}-\mathbf{R}\widehat{\mathbf{W}}\|_F^2
+\|(\mathbf{I}-\mathbf{U}\mathbf{U}^{\top})\mathbf{Y}\|_F^2.
\label{eq:qr_equivalence}
\end{equation}
The second term does not depend on the quantized weights, so the
original objective \eqref{eq:adaptive_grid_problem} and the reduced
objective \eqref{eq:reduced_problem} have the same minimizers. Appending $\nu\mathbf{I}$ to the activation matrix
makes $\mathbf{A}$ full column rank when $\nu>0$.

\subsection{Numerical Updates and Stopping}
\label{app:numerical-updates}

\paragraph{Order of updates.}
Algorithm~\ref{alg:jarq} summarizes the layerwise alternating updates;
the stopping rule below can stop a channel before its iteration budget
is exhausted.
Within a layer, the input factors and reconstruction target stay fixed.
Each iteration fits the scales and then proposes group-level integer
updates. A code proposal is accepted only if it lowers the reconstruction
objective. The refined layer's outputs become the inputs to the next
layer. Group QR uses no additional column ordering; any ordering
introduced by GPTQ is preserved during refinement and inverted when
the weights are written back.

\paragraph{Solving for scales.}
The scale system is solved in FP32.
Let $\mathbf{G}_j=\mathbf{B}_j^\top\mathbf{B}_j$.
We symmetrize this matrix before factorization.
If its spectrum is nonfinite, its smallest eigenvalue is at most
$r_{\rm cond}\max(|\lambda_{\max}(\mathbf{G}_j)|,\epsilon_{\rm fp})$,
or Cholesky factorization fails, we add $\lambda\mathbf{I}$ and
factorize again, with $\lambda=10^{-4}$, $r_{\rm cond}=10^{-10}$,
and $\epsilon_{\rm fp}$ equal to FP32 machine epsilon. If factorization
still fails, the solve stops with a numerical error. Thus
$\lambda_j=0$ for the ordinary solve and $\lambda_j=10^{-4}$ when
stabilization is triggered. The stabilized branch solves a
ridge-regularized scale subproblem, not the exact least-squares
problem in \eqref{eq:scale_normal_equations}, and its result is kept
without a separate acceptance test on the original reconstruction
objective; iteration-wise monotonicity is therefore not guaranteed
for this branch (Remark~\ref{rem:ridge-acceptance}).

\paragraph{Early stopping.}
Let $L_j^{(k)}$ be the full reconstruction objective immediately after
the scale update, including the constant term in
\eqref{eq:qr_equivalence}. For $k\geq2$, an output channel stops
before the next code update when
\begin{equation}
\frac{L_j^{(k-1)}-L_j^{(k)}}
     {\max(|L_j^{(k-1)}|,\epsilon_{\rm fp})}\leq\tau
\quad\text{or}\quad
|L_j^{(k-1)}-L_j^{(k)}|\leq\tau,
\qquad \tau=10^{-5}.
\label{eq:strict_stopping}
\end{equation}
The new scales are kept. At termination, the current scales and codes
define the returned weights, which are cast to the model's weight
dtype and installed in the layer.

\paragraph{Small scales.}
Code proposals are computed only when $|s_{i,j}|$ exceeds FP32
machine epsilon; groups with smaller scales keep their codes.
A proposal is rejected if its residual is nonfinite or if it does not
lower the conditional reconstruction objective.

\subsection{Representing Negative Scales}

All quality evaluations use the final signed scales directly.
For a deployment format that supports the resulting scales and zero
points, a negative scale can be converted to a positive one without changing the weights
it represents.
Let $q_{\max}=2^b-1$. For $s_{i,j}<0$, set
\begin{equation}
s_{i,j}'=-s_{i,j},\qquad
\mathbf{q}_{i,j}'=q_{\max}\mathbf{1}-\mathbf{q}_{i,j},\qquad
z_{i,j}'=q_{\max}-z_{i,j}.
\label{eq:canonicalization}
\end{equation}
The transformed codes remain in the same allowed integer range, and
\begin{equation}
s_{i,j}'(\mathbf{q}_{i,j}'-z_{i,j}'\mathbf{1})
=(-s_{i,j})(-\mathbf{q}_{i,j}+z_{i,j}\mathbf{1})
=s_{i,j}(\mathbf{q}_{i,j}-z_{i,j}\mathbf{1}).
\label{eq:canonicalization_equivalence}
\end{equation}
This optional conversion is applied after refinement; zero points
stay fixed during optimization. Fitting the offset jointly with the
scale, as an affine term $s_{i,j}z_{i,j}$, would keep the scale step
linear but would yield real-valued offsets, whereas group-wise formats
store integer zero points; fixing $\mathbf{Z}$ keeps the host's format.

\subsection{Refinement with QuaRot}
\label{sec:variants}

The QuaRot results in Table~\ref{tab:host-compatibility}(b) use
Algorithm~\ref{alg:jarq} in the rotated coordinates, starting directly
from the scales, codes, and zero points of QuaRot+GPTQ.
Each iteration jointly fits the scales and then performs group-wise
bounded Babai updates using already rounded higher-index codes during
back substitution, as in \eqref{eq:babai_update}.
The iteration budget is $K=3$, and the reconstruction objective
is \eqref{eq:adaptive_grid_problem} with $\nu=0.6$.
The Hadamard rotation, group assignments, and zero points stay fixed.
Numerical stabilization, stopping, and code-proposal acceptance follow
Appendix~\ref{app:numerical-updates}. The final scales and codes
define the weights installed in the rotated model.

\subsection{Propagation-Corrected Targets and QEP}
\label{sec:target_dependence}

The main target $\mathbf{Y}^{\star}=\widetilde{\mathbf{X}}\mathbf{W}$
asks each layer to reproduce the original weights' response to the
inputs it actually receives. It does not compensate for errors
introduced by preceding layers, so the gains in Section~\ref{sec:experiments}
come from the scale and code updates, not from a corrected target.
A target can also absorb those errors. Keeping the design matrix in
\eqref{eq:strict_target}, replace the activation part of $\mathbf{Y}$ with
\begin{equation}
\mathbf{Y}^{\star}(\mu)
=(1-\mu)\mathbf{X}\mathbf{W}
+\mu\widetilde{\mathbf{X}}\mathbf{W},\qquad \mu\in[0,1],
\label{eq:mu-target}
\end{equation}
where $\mathbf{X}$ and $\widetilde{\mathbf{X}}$ stack the inputs from
full-precision and quantized preceding layers, respectively. Setting
$\mu=1$ recovers the main target.

\paragraph{Equivalence to QEP's corrected weights.}
Let $\mathbf{H}=\widetilde{\mathbf{X}}^{\top}\widetilde{\mathbf{X}}
+\nu^2\mathbf{I}$, which is nonsingular for $\nu>0$, and let
$\mathbf{Y}(\mu)$ be the augmented target with activation part
\eqref{eq:mu-target}. The normal equations of
$\min_{\widehat{\mathbf{W}}}\|\mathbf{Y}(\mu)-\mathbf{A}\widehat{\mathbf{W}}\|_F^2$
over real matrices give
\begin{equation}
\mathbf{W}_{\mathrm{corr}}
=\mathbf{W}+(1-\mu)\,\mathbf{H}^{-1}\widetilde{\mathbf{X}}^{\top}
(\mathbf{X}-\widetilde{\mathbf{X}})\mathbf{W},
\label{eq:w-corr}
\end{equation}
which has the form of QEP's propagation-corrected weights
\citep{arai2025qep} with correction strength $\alpha=1-\mu$ and damping
$\nu^2$; QEP's default $\alpha=0.5$ corresponds to $\mu=0.5$. Because $\mathbf{Y}(\mu)-\mathbf{A}\mathbf{W}_{\mathrm{corr}}$
is orthogonal to the range of $\mathbf{A}$, for every $\widehat{\mathbf{W}}$,
\begin{equation}
\|\mathbf{Y}(\mu)-\mathbf{A}\widehat{\mathbf{W}}\|_F^2
=\|\mathbf{A}(\mathbf{W}_{\mathrm{corr}}-\widehat{\mathbf{W}})\|_F^2
+\|\mathbf{Y}(\mu)-\mathbf{A}\mathbf{W}_{\mathrm{corr}}\|_F^2 .
\end{equation}
Refining with target \eqref{eq:mu-target} is therefore the main problem
with $\mathbf{W}$ replaced by $\mathbf{W}_{\mathrm{corr}}$, plus a constant
that Algorithm~\ref{alg:jarq} carries in $c_j$. Theorem~\ref{thm:descent},
Corollary~\ref{cor:host}, and the propositions of
Appendix~\ref{app:theoretical-analysis} need only a target that is fixed
within the layer, so they hold for every $\mu\in[0,1]$. In this form,
QEP chooses the target and \method{} solves it on a moving grid. All
experiments in this paper use $\mu=1$.

\section{Group Size and Regularization Ablations}
\label{app:hyperparameter-ablations}

We run two full-model ablations on Llama-2-7B and Qwen3-8B with the
RTN host at W3A16 and W4A16. Both use 128 WT2 training sequences of
2,048 tokens, seed 0, and an iteration budget of $K=3$ with early
stopping. Each model is evaluated on the WT2 test set and 256 C4
validation windows. All runs share the same calibration and evaluation
data; the $g=128$, $\nu=0.6$ states are shared by the two ablations,
giving 44 distinct model states.

\subsection{Group Size}
\label{app:group-size-ablation}

We vary the group size $g\in\{64,128,256\}$ while fixing $\nu=0.6$.
Each group size has its own RTN Base ($K=0$).
Table~\ref{tab:group-size-ablation} shows that refinement lowers both
WT2 and C4 perplexity in all 12 model--bit-width--group-size
configurations. The main experiments use $g=128$, the common setting
for group-wise LLM quantization used in AWQ~\citep{lin2023awq}.

\begin{table}[!htbp]
\centering
\caption{Group-size ablation with the RTN host, $\nu=0.6$, and $K=3$.
Entries are Base/$+$Ours perplexity (lower is better). Bold marks the
better value in each pair (compared before rounding).}
\label{tab:group-size-ablation}
\begingroup
\small
\setlength{\tabcolsep}{4pt}
\renewcommand{\arraystretch}{1.1}
\begin{tabular*}{\textwidth}{@{\extracolsep{\fill}}lllrrr@{}}
\toprule
Model & W/A & Data & $g=64$ & $g=128$ & $g=256$ \\
\midrule
\multirow{4}{*}{Llama-2-7B} & \multirow{2}{*}{W3A16} & WT2 & \ppl{6.4009}/\textbf{\ppl{6.0186}} & \ppl{6.6630}/\textbf{\ppl{6.2043}} & \ppl{7.1019}/\textbf{\ppl{6.2951}} \\
 &  & C4 & \ppl{8.1095}/\textbf{\ppl{7.8924}} & \ppl{8.4050}/\textbf{\ppl{8.2166}} & \ppl{8.9828}/\textbf{\ppl{8.4944}} \\
\addlinespace[2pt]
 & \multirow{2}{*}{W4A16} & WT2 & \ppl{5.6780}/\textbf{\ppl{5.5938}} & \ppl{5.7261}/\textbf{\ppl{5.6256}} & \ppl{5.7495}/\textbf{\ppl{5.6474}} \\
 &  & C4 & \ppl{7.2114}/\textbf{\ppl{7.1347}} & \ppl{7.2458}/\textbf{\ppl{7.1857}} & \ppl{7.2983}/\textbf{\ppl{7.2414}} \\
\midrule
\multirow{4}{*}{Qwen3-8B} & \multirow{2}{*}{W3A16} & WT2 & \ppl{12.3856}/\textbf{\ppl{10.4292}} & \ppl{13.4564}/\textbf{\ppl{10.7123}} & \ppl{15.7918}/\textbf{\ppl{11.1354}} \\
 &  & C4 & \ppl{16.0018}/\textbf{\ppl{14.5767}} & \ppl{17.5825}/\textbf{\ppl{15.1033}} & \ppl{20.3924}/\textbf{\ppl{15.8302}} \\
\addlinespace[2pt]
 & \multirow{2}{*}{W4A16} & WT2 & \ppl{9.9778}/\textbf{\ppl{9.8458}} & \ppl{10.1369}/\textbf{\ppl{9.9980}} & \ppl{10.4434}/\textbf{\ppl{10.0580}} \\
 &  & C4 & \ppl{13.5253}/\textbf{\ppl{13.4867}} & \ppl{13.6431}/\textbf{\ppl{13.6271}} & \ppl{13.8796}/\textbf{\ppl{13.7751}} \\
\bottomrule
\end{tabular*}
\endgroup
\end{table}

\paragraph{Runtime.}
Table~\ref{tab:group-size-time} compares the runtime of $g=64$ and
$g=128$. Relative to $g=128$, $g=64$ gives 0.57--3.95\% lower refined
perplexity across the eight model--bit-width--corpus comparisons and
takes 2.45--2.74 times as long.

\begin{table}[!htbp]
\centering
\caption{Full-model quantization and refinement time for the RTN host
with $\nu=0.6$ and $K=3$. Increase is $100(T_{64}/T_{128}-1)$.}
\label{tab:group-size-time}
\begingroup
\small
\setlength{\tabcolsep}{4pt}
\renewcommand{\arraystretch}{1.1}
\begin{tabular*}{\textwidth}{@{\extracolsep{\fill}}llrrr@{}}
\toprule
Model & W/A & $g=128$ (min) & $g=64$ (min) & Increase \\
\midrule
\multirow{2}{*}{Llama-2-7B} & W3A16 & 32.13 & 87.96 & $+173.78\%$ \\
 & W4A16 & 32.06 & 87.04 & $+171.50\%$ \\
\midrule
\multirow{2}{*}{Qwen3-8B} & W3A16 & 38.54 & 95.36 & $+147.43\%$ \\
 & W4A16 & 39.24 & 96.00 & $+144.64\%$ \\
\bottomrule
\end{tabular*}
\endgroup
\end{table}

The timer covers calibration input capture, RTN quantization,
refinement, diagnostics and progress logging, and block propagation;
model loading and perplexity evaluation are excluded. Each entry is one
full-model measurement, with a different timing scope from the
block-level measurements in Appendix~\ref{app:efficiency-protocol}.

\subsection{Regularization Strength}
\label{app:regularization-ablation}

We fix $g=128$ and vary $\nu\in\{0,0.2,0.4,0.6,0.8,1.0\}$ in
\eqref{eq:strict_target}; the weight-error coefficient is $\nu^2$.
Every value, including $\nu=0$, uses $K=3$; only the RTN Base uses
$K=0$. The conditional $10^{-4}$ stabilization of the scale solve
(Appendix~\ref{app:numerical-updates}) is unchanged.
Table~\ref{tab:regularization-ablation} reports all six strengths.

\begin{table}[!htbp]
\centering
\caption{Regularization-strength ablation with the RTN host, $g=128$,
and $K=3$. Entries are WT2/C4 perplexity (lower is better). Bold marks
the best of the six strengths for each model, bit width, and corpus
(compared before rounding).}
\label{tab:regularization-ablation}
\begingroup
\small
\setlength{\tabcolsep}{4pt}
\renewcommand{\arraystretch}{1.1}
\begin{tabular*}{\textwidth}{@{\extracolsep{\fill}}lrrrr@{}}
\toprule
& \multicolumn{2}{c}{Llama-2-7B} & \multicolumn{2}{c}{Qwen3-8B} \\
\cmidrule(lr){2-3}\cmidrule(lr){4-5}
$\nu$ & W3A16 & W4A16 & W3A16 & W4A16 \\
\midrule
Base ($K=0$) & \ppl{6.6630}/\ppl{8.4050} & \ppl{5.7261}/\ppl{7.2458} & \ppl{13.4564}/\ppl{17.5825} & \ppl{10.1369}/\ppl{13.6431} \\
\midrule
0.0 & \ppl{6.2254}/\ppl{8.2287} & \textbf{\ppl{5.6203}}/\ppl{7.1873} & \textbf{\ppl{10.5207}}/\textbf{\ppl{15.0312}} & \ppl{9.9718}/\ppl{13.5872} \\
0.2 & \ppl{6.1879}/\ppl{8.2263} & \ppl{5.6242}/\textbf{\ppl{7.1837}} & \ppl{10.6558}/\ppl{15.1019} & \ppl{9.9430}/\ppl{13.5646} \\
0.4 & \ppl{6.1944}/\ppl{8.2391} & \ppl{5.6295}/\ppl{7.1875} & \ppl{10.7394}/\ppl{15.1621} & \textbf{\ppl{9.9146}}/\textbf{\ppl{13.5575}} \\
0.6 & \ppl{6.2043}/\ppl{8.2166} & \ppl{5.6256}/\ppl{7.1857} & \ppl{10.7123}/\ppl{15.1033} & \ppl{9.9980}/\ppl{13.6271} \\
0.8 & \ppl{6.1906}/\ppl{8.2149} & \ppl{5.6340}/\ppl{7.1848} & \ppl{10.7197}/\ppl{15.1473} & \ppl{10.0092}/\ppl{13.6307} \\
1.0 & \textbf{\ppl{6.1830}}/\textbf{\ppl{8.2041}} & \ppl{5.6354}/\ppl{7.1875} & \ppl{10.7658}/\ppl{15.2126} & \ppl{9.9385}/\ppl{13.5743} \\
\bottomrule
\end{tabular*}
\endgroup
\end{table}

Refinement improves both perplexities over RTN at every tested strength.
The best strength depends on the model and bit width: Llama-2-7B at
W3A16 reaches its lowest values at $\nu=1.0$, whereas Qwen3-8B at
W3A16 and W4A16 reaches them at $\nu=0$ and $\nu=0.4$, respectively.

\section{Analysis of the Refinement Updates}
\label{app:theoretical-analysis}

This appendix proves the results stated in Section~\ref{sec:jarq_solver},
gives two further propositions on the scale and code updates, relates
the code update to GPTQ, and derives the computational cost of the
refinement procedure.

\subsection{Joint Scale Fitting}
\label{app:joint-scale-proof}

\begin{proposition}[Error gap for joint scale fitting]
\label{prop:joint-scale-gap}
Fix the codes of output column $j$ and write
$f_j(\mathbf{s})=\|\mathbf{y}_j-\mathbf{B}_j\mathbf{s}\|_2^2$.
For any least-squares minimizer $\mathbf{s}_j^\star$ and any
$\widetilde{\mathbf{s}}_j\in\mathbb{R}^G$,
\begin{equation}
f_j(\widetilde{\mathbf{s}}_j)-f_j(\mathbf{s}_j^\star)
=\|\mathbf{B}_j(\widetilde{\mathbf{s}}_j-\mathbf{s}_j^\star)\|_2^2
\geq0.
\label{eq:joint-scale-gap}
\end{equation}
Equality holds exactly when
$\mathbf{B}_j(\widetilde{\mathbf{s}}_j-\mathbf{s}_j^\star)=\mathbf{0}$.
\end{proposition}

\begin{proof}[Proof of Proposition~\ref{prop:joint-scale-gap}]
Let $\mathbf{r}_j^\star=\mathbf{y}_j-\mathbf{B}_j\mathbf{s}_j^\star$
and $\boldsymbol{\delta}=\widetilde{\mathbf{s}}_j-\mathbf{s}_j^\star$.
The least-squares normal equations imply
$\mathbf{B}_j^\top\mathbf{r}_j^\star=\mathbf{0}$. Hence
\begin{align*}
f_j(\widetilde{\mathbf{s}}_j)
&=\|\mathbf{r}_j^\star-\mathbf{B}_j\boldsymbol{\delta}\|_2^2\\
&=\|\mathbf{r}_j^\star\|_2^2
  -2\boldsymbol{\delta}^\top\mathbf{B}_j^\top\mathbf{r}_j^\star
  +\|\mathbf{B}_j\boldsymbol{\delta}\|_2^2\\
&=f_j(\mathbf{s}_j^\star)+\|\mathbf{B}_j\boldsymbol{\delta}\|_2^2.
\end{align*}
The final term is nonnegative and vanishes exactly under the stated
condition. The argument does not require $\mathbf{B}_j$ to have full
column rank.
\end{proof}

\paragraph{When do separate group fits suffice?}
For the main target, $\mathbf{y}_j=\mathbf{A}\mathbf{w}_j$.
Write $\mathbf{c}_{i,j}=\mathbf{q}_{i,j}-z_{i,j}\mathbf{1}$ and
$\mathbf{e}_{i,j}=\mathbf{w}_{i,j}-s_{i,j}\mathbf{c}_{i,j}$.
Then
\begin{equation}
f_j(\mathbf{s}_j)
=\sum_i\|\mathbf{A}_i\mathbf{e}_{i,j}\|_2^2
+2\sum_{i<h}\mathbf{e}_{i,j}^\top
\mathbf{A}_i^\top\mathbf{A}_h\mathbf{e}_{h,j}.
\label{eq:group-error-coupling}
\end{equation}
If $\mathbf{A}_i^\top\mathbf{A}_h=\mathbf{0}$ for every $i\neq h$,
minimizing each group's reconstruction error separately also minimizes
the full objective. In general the cross terms remain, so separate fits
need not satisfy the joint normal equations. Proposition~\ref{prop:joint-scale-gap}
measures their resulting error gap, and
Proposition~\ref{prop:independent-gap} below expresses it through the
coupling between groups. The scale-fitting comparisons in
Table~\ref{tab:mechanism-summary} quantify this effect on model layers;
their reported gaps are normalized by the host objective.

The independent fit in Table~\ref{tab:mechanism-summary} fits each
group's scale to that group's own contribution
$\mathbf{A}_i\mathbf{w}_{i,j}$:
\begin{equation}
\widetilde{s}_{i,j}
=\arg\min_{s\in\mathbb{R}}
\|\mathbf{A}_i(\mathbf{w}_{i,j}-s\,\mathbf{c}_{i,j})\|_2^2
=\frac{\mathbf{b}_{i}^\top\mathbf{A}_i\mathbf{w}_{i,j}}
      {\|\mathbf{b}_{i}\|_2^2},
\qquad
\mathbf{b}_{i}=\mathbf{A}_i\mathbf{c}_{i,j},
\label{eq:independent-fit}
\end{equation}
where $\mathbf{b}_i$ is the $i$th column of $\mathbf{B}_j$.
We use this definition below. Fitting each scale alone to the full
target, $\mathbf{b}_i^\top\mathbf{y}_j/\|\mathbf{b}_i\|_2^2$, is a
different estimator and is not the one reported in the table.

\begin{proposition}[Gap of independent scale fits]
\label{prop:independent-gap}
Let $\mathbf{y}_j=\mathbf{A}\mathbf{w}_j$, fix the codes, and assume
every $\mathbf{b}_i\neq\mathbf{0}$ and
$\mathbf{G}_j=\mathbf{B}_j^\top\mathbf{B}_j$ is nonsingular.
Let $\mathbf{e}_{h,j}=\mathbf{w}_{h,j}-\widetilde{s}_{h,j}\mathbf{c}_{h,j}$,
$\mathbf{N}=\mathbf{B}_j\mathbf{D}_j^{-1/2}$ with
$\mathbf{D}_j=\operatorname{diag}(\mathbf{G}_j)$, and let
$\mathbf{C}_j=\mathbf{N}^\top\mathbf{N}
=\mathbf{D}_j^{-1/2}\mathbf{G}_j\mathbf{D}_j^{-1/2}$
be the normalized Gram matrix, with columns
$\mathbf{n}_i=\mathbf{b}_i/\|\mathbf{b}_i\|_2$ of $\mathbf{N}$.
Let $\Delta_j=f_j(\widetilde{\mathbf{s}}_j)-\min_{\mathbf{s}}f_j(\mathbf{s})$.
\begin{enumerate}
\item[(a)] $\Delta_j=\widehat{\mathbf{g}}^\top\mathbf{C}_j^{-1}\widehat{\mathbf{g}}$,
where $\widehat{g}_i=\sum_{h\neq i}\mathbf{n}_i^\top\mathbf{A}_h\mathbf{e}_{h,j}$.
\item[(b)] $\|\widehat{\mathbf{g}}\|_2^2/\lambda_{\max}(\mathbf{C}_j)
\leq\Delta_j\leq
\|\widehat{\mathbf{g}}\|_2^2/\lambda_{\min}(\mathbf{C}_j)$.
\item[(c)] Let $\boldsymbol{\Gamma}\in\mathbb{R}^{G\times G}$ have zero
diagonal and entries $\Gamma_{ih}=\|\mathbf{U}_i^\top\mathbf{U}_h\|_2$
for $i\neq h$, with $\mathbf{U}_i$ from \eqref{eq:group_qr}, and let
$\gamma=\|\boldsymbol{\Gamma}\|_2$. Then
$\|\mathbf{C}_j-\mathbf{I}\|_2\leq\gamma$ and
\begin{equation}
\Delta_j\leq
\frac{\gamma^2}{\lambda_{\min}(\mathbf{C}_j)}
\sum_{h=1}^{G}\|\mathbf{A}_h\mathbf{e}_{h,j}\|_2^2
\leq\frac{\gamma^2}{1-\gamma}
\sum_{h=1}^{G}\|\mathbf{A}_h\mathbf{e}_{h,j}\|_2^2
\quad\text{if }\gamma<1.
\label{eq:independent-gap-bound}
\end{equation}
\end{enumerate}
\end{proposition}
The entry $\Gamma_{ih}$ is the cosine of the smallest principal angle
between the column spaces of $\mathbf{A}_i$ and $\mathbf{A}_h$; it depends
on the inputs and $\nu$ but not on the codes. The sum in
\eqref{eq:independent-gap-bound} is the first term of
\eqref{eq:group-error-coupling} at the independent fit. Thus the gap is
zero when the group column spaces are orthogonal and grows at most
quadratically with their coherence. With many correlated groups,
$\gamma$ can exceed one and part (c) becomes uninformative, while
the exact expression in part (a) still holds.

\begin{proof}[Proof of Proposition~\ref{prop:independent-gap}]
Let $\widetilde{\mathbf{r}}=\mathbf{y}_j-\mathbf{B}_j\widetilde{\mathbf{s}}_j$.
The unique minimizer is
$\mathbf{s}_j^\star=\mathbf{G}_j^{-1}\mathbf{B}_j^\top\mathbf{y}_j$, so
$\widetilde{\mathbf{s}}_j-\mathbf{s}_j^\star
=-\mathbf{G}_j^{-1}\mathbf{B}_j^\top\widetilde{\mathbf{r}}$.
By Proposition~\ref{prop:joint-scale-gap} and
$\mathbf{B}_j=\mathbf{N}\mathbf{D}_j^{1/2}$,
\[
\Delta_j
=\widetilde{\mathbf{r}}^\top\mathbf{B}_j\mathbf{G}_j^{-1}
 \mathbf{B}_j^\top\widetilde{\mathbf{r}}
=\widetilde{\mathbf{r}}^\top\mathbf{N}\mathbf{C}_j^{-1}
 \mathbf{N}^\top\widetilde{\mathbf{r}}.
\]
Since $\mathbf{y}_j=\sum_h\mathbf{A}_h\mathbf{w}_{h,j}$, we have
$\widetilde{\mathbf{r}}=\sum_h\mathbf{A}_h\mathbf{e}_{h,j}$.
The optimality condition of \eqref{eq:independent-fit} gives
$\mathbf{b}_i^\top\mathbf{A}_i\mathbf{e}_{i,j}=0$, so
$\mathbf{n}_i^\top\widetilde{\mathbf{r}}=\widehat{g}_i$. This proves (a).
The matrix $\mathbf{C}_j$ is positive definite because $\mathbf{G}_j$
is, and (b) follows from the Rayleigh quotient.
For (c), $\mathbf{n}_i$ is a unit vector in the range of $\mathbf{U}_i$
and $\mathbf{A}_h\mathbf{e}_{h,j}$ lies in the range of $\mathbf{U}_h$.
Hence $|\mathbf{n}_i^\top\mathbf{A}_h\mathbf{e}_{h,j}|
\leq\Gamma_{ih}\|\mathbf{A}_h\mathbf{e}_{h,j}\|_2$ and
$|(\mathbf{C}_j)_{ih}|=|\mathbf{n}_i^\top\mathbf{n}_h|\leq\Gamma_{ih}$
for $i\neq h$. Let $\boldsymbol{\rho}$ have entries
$\rho_h=\|\mathbf{A}_h\mathbf{e}_{h,j}\|_2$. Then
$\|\widehat{\mathbf{g}}\|_2\leq\|\boldsymbol{\Gamma}\boldsymbol{\rho}\|_2
\leq\gamma\|\boldsymbol{\rho}\|_2$. Since $\boldsymbol{\Gamma}$ is symmetric
and entrywise nonnegative, for any $\mathbf{x}$,
$|\mathbf{x}^\top(\mathbf{C}_j-\mathbf{I})\mathbf{x}|
\leq|\mathbf{x}|^\top\boldsymbol{\Gamma}|\mathbf{x}|
\leq\gamma\|\mathbf{x}\|_2^2$, where $|\mathbf{x}|$ is taken entrywise.
Thus $\|\mathbf{C}_j-\mathbf{I}\|_2\leq\gamma$ and
$\lambda_{\min}(\mathbf{C}_j)\geq1-\gamma$. Combining these bounds with (b)
gives \eqref{eq:independent-gap-bound}.
\end{proof}
For the other estimator,
$\mathbf{b}_i^\top\mathbf{y}_j/\|\mathbf{b}_i\|_2^2$, the same argument
gives the gap
$\widehat{\mathbf{b}}^\top(\mathbf{C}_j-\mathbf{I})\mathbf{C}_j^{-1}
(\mathbf{C}_j-\mathbf{I})\widehat{\mathbf{b}}$ with
$\widehat{\mathbf{b}}=\mathbf{D}_j^{-1/2}\mathbf{B}_j^\top\mathbf{y}_j$,
which is again zero when $\mathbf{C}_j=\mathbf{I}$.

\subsection{A Single-Code Minimum with an Improving Babai Proposal}
\label{app:multicode-proof}

\begin{proof}[Proof of Proposition~\ref{prop:multicode-descent}]
Consider one group of two weights with a scalar scale $s$, zero point
$z=0$, and two-bit codes. Its reduced reconstruction problem is
\begin{equation}
f(s,\mathbf{q})=\|\mathbf{y}-s\mathbf{R}\mathbf{q}\|_2^2,
\qquad
\mathbf{R}=\begin{bmatrix}1&-1\\0&1/10\end{bmatrix},
\qquad
\mathbf{y}=\begin{bmatrix}-1\\3/10\end{bmatrix},
\qquad
\mathbf{q}\in\{0,1,2,3\}^2.
\label{eq:multicode-example}
\end{equation}
The matrix $\mathbf{R}$ is nonsingular. For a nonzero code vector,
the optimal scale and the corresponding error are
\begin{equation}
s^\star(\mathbf{q})
=\frac{(\mathbf{R}\mathbf{q})^\top\mathbf{y}}
       {\|\mathbf{R}\mathbf{q}\|_2^2},
\qquad
\phi(\mathbf{q})=\min_s f(s,\mathbf{q})
=\|\mathbf{y}\|_2^2
-\frac{((\mathbf{R}\mathbf{q})^\top\mathbf{y})^2}
      {\|\mathbf{R}\mathbf{q}\|_2^2}.
\label{eq:profiled-example-loss}
\end{equation}
At $\mathbf{q}_0=(1,2)^\top$, these give
$s_0=53/52$ and $\phi(\mathbf{q}_0)=1/104$.
Table~\ref{tab:single-code-example} enumerates all six single-code
neighbors, and Figure~\ref{fig:multicode-example} shows the refitted
error of every code pair. Every neighbor has larger error even after its scale is
refitted optimally, since their minimum is $9/409>1/104$.

\begin{table}[!htbp]
\centering
\caption{All single-code neighbors of $(1,2)^\top$ in
\eqref{eq:multicode-example}, with an exact scale refit for each
candidate. The current state's error is $1/104$.}
\label{tab:single-code-example}
\small
\begin{tabular}{lcc}
\toprule
Changed coordinate & New codes $\mathbf{q}^\top$ &
Refitted error $\phi(\mathbf{q})$ \\
\midrule
First & $(0,2)$ & $4/101$ \\
      & $(2,2)$ & $1$ \\
      & $(3,2)$ & $25/104$ \\
\addlinespace[2pt]
Second & $(1,0)$ & $9/100$ \\
       & $(1,1)$ & $1$ \\
       & $(1,3)$ & $9/409$ \\
\bottomrule
\end{tabular}
\end{table}

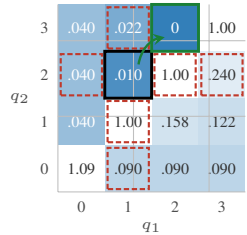
\begin{figure}[!htbp]
\centering
\begin{tikzpicture}[x=0.62cm, y=0.62cm, font=\tiny]
\definecolor{mcblue}{HTML}{2878B5}
\definecolor{mcred}{HTML}{C0392B}
\definecolor{mcgreen}{HTML}{1A7F37}
\fill[mcblue!0] (0,0) rectangle ++(1,1);
\node[text=black!80] at (0.5,0.5) {1.09};
\fill[mcblue!48] (0,1) rectangle ++(1,1);
\node[text=white] at (0.5,1.5) {.040};
\fill[mcblue!48] (0,2) rectangle ++(1,1);
\node[text=white] at (0.5,2.5) {.040};
\fill[mcblue!48] (0,3) rectangle ++(1,1);
\node[text=white] at (0.5,3.5) {.040};
\fill[mcblue!31] (1,0) rectangle ++(1,1);
\node[text=black!80] at (1.5,0.5) {.090};
\fill[mcblue!0] (1,1) rectangle ++(1,1);
\node[text=black!80] at (1.5,1.5) {1.00};
\fill[mcblue!83] (1,2) rectangle ++(1,1);
\node[text=white] at (1.5,2.5) {.010};
\fill[mcblue!62] (1,3) rectangle ++(1,1);
\node[text=white] at (1.5,3.5) {.022};
\fill[mcblue!31] (2,0) rectangle ++(1,1);
\node[text=black!80] at (2.5,0.5) {.090};
\fill[mcblue!21] (2,1) rectangle ++(1,1);
\node[text=black!80] at (2.5,1.5) {.158};
\fill[mcblue!0] (2,2) rectangle ++(1,1);
\node[text=black!80] at (2.5,2.5) {1.00};
\fill[mcblue!95] (2,3) rectangle ++(1,1);
\node[text=white] at (2.5,3.5) {0};
\fill[mcblue!31] (3,0) rectangle ++(1,1);
\node[text=black!80] at (3.5,0.5) {.090};
\fill[mcblue!25] (3,1) rectangle ++(1,1);
\node[text=black!80] at (3.5,1.5) {.122};
\fill[mcblue!14] (3,2) rectangle ++(1,1);
\node[text=black!80] at (3.5,2.5) {.240};
\fill[mcblue!0] (3,3) rectangle ++(1,1);
\node[text=black!80] at (3.5,3.5) {1.00};
\draw[black!25] (0,0) grid (4,4);
\draw[mcred, line width=0.8pt, dash pattern=on 2pt off 1pt] (0.06,2.06) rectangle ++(0.88,0.88);
\draw[mcred, line width=0.8pt, dash pattern=on 2pt off 1pt] (2.06,2.06) rectangle ++(0.88,0.88);
\draw[mcred, line width=0.8pt, dash pattern=on 2pt off 1pt] (3.06,2.06) rectangle ++(0.88,0.88);
\draw[mcred, line width=0.8pt, dash pattern=on 2pt off 1pt] (1.06,0.06) rectangle ++(0.88,0.88);
\draw[mcred, line width=0.8pt, dash pattern=on 2pt off 1pt] (1.06,1.06) rectangle ++(0.88,0.88);
\draw[mcred, line width=0.8pt, dash pattern=on 2pt off 1pt] (1.06,3.06) rectangle ++(0.88,0.88);
\draw[black, line width=1.1pt] (1,2) rectangle ++(1,1);
\draw[mcgreen, line width=1.1pt] (2,3) rectangle ++(1,1);
\draw[-stealth, line width=0.9pt, mcgreen] (1.72,2.72) to[bend left=35] (2.28,3.3);
\node[anchor=north, text=black!70] at (0.5,0) {0};
\node[anchor=east, text=black!70] at (0,0.5) {0};
\node[anchor=north, text=black!70] at (1.5,0) {1};
\node[anchor=east, text=black!70] at (0,1.5) {1};
\node[anchor=north, text=black!70] at (2.5,0) {2};
\node[anchor=east, text=black!70] at (0,2.5) {2};
\node[anchor=north, text=black!70] at (3.5,0) {3};
\node[anchor=east, text=black!70] at (0,3.5) {3};
\node[anchor=north, text=black!70] at (2,-0.45) {$q_1$};
\node[anchor=east, text=black!70] at (-0.4,2) {$q_2$};
\end{tikzpicture}
\caption{Error after refitting the scale for each code pair in
\eqref{eq:multicode-example} (darker is lower). Black: current codes
$(1,2)$; dashed: single-code moves; green: the Babai step to $(2,3)$.}
\label{fig:multicode-example}
\end{figure}

\medskip
At the current scale $s_0$, bounded Babai back substitution gives
\begin{equation}
q_2^+=\operatorname{clip}_{\{0,1,2,3\}}
\!\left(\operatorname{round}\frac{156}{53}\right)=3,
\qquad
q_1^+=\operatorname{clip}_{\{0,1,2,3\}}
\!\left(\operatorname{round}\frac{107}{53}\right)=2.
\label{eq:example-babai-codes}
\end{equation}
Thus $\mathbf{q}^+=(2,3)^\top$ changes both codes and satisfies
\begin{equation}
f(s_0,\mathbf{q}^+)=\frac{109}{270400}
<\frac{1}{104}=f(s_0,\mathbf{q}_0).
\label{eq:example-babai-descent}
\end{equation}
The proposal is therefore accepted. The next scale fit gives $s^+=1$
and zero error because $\mathbf{y}=\mathbf{R}(2,3)^\top$.
\end{proof}

This construction is compatible with the main reconstruction target:
$\mathbf{y}=\mathbf{R}\mathbf{w}$ for $\mathbf{w}=(2,3)^\top$.
For any fixed $\nu>0$ and a sufficiently large constant $c$, the
matrix $c^2\mathbf{R}^\top\mathbf{R}-\nu^2\mathbf{I}$ is positive
definite. Choosing $\widetilde{\mathbf{X}}$ with
$\widetilde{\mathbf{X}}^\top\widetilde{\mathbf{X}}
=c^2\mathbf{R}^\top\mathbf{R}-\nu^2\mathbf{I}$ makes the augmented
objective in \eqref{eq:adaptive_grid_problem} equal to $c^2f(s,\mathbf{q})$,
so the scale fits, Babai codes, and strict error comparisons are unchanged.

\subsection{Refitting the Grid of a Fixed-Grid Optimum}
\label{app:grid-staleness-proof}

\begin{proof}[Proof of Proposition~\ref{prop:grid-staleness}]
Take one group of two weights with zero point $z=0$, two-bit codes
$\mathbb{B}=\{0,1,2,3\}$, and
\begin{equation}
\mathbf{R}=\begin{bmatrix}1&\rho\\0&\sqrt{1-\rho^2}\end{bmatrix},
\qquad \rho=\tfrac{18}{25},
\qquad \mathbf{y}=\mathbf{R}\mathbf{w},
\qquad \mathbf{w}=\bigl(\tfrac{21}{50},\tfrac{31}{20}\bigr)^{\top}.
\label{eq:staleness-example}
\end{equation}
Then $f(s,\mathbf{q})=\|\mathbf{y}-s\mathbf{R}\mathbf{q}\|_2^2
=(\mathbf{w}-s\mathbf{q})^\top\mathbf{M}(\mathbf{w}-s\mathbf{q})$ with
$\mathbf{M}=\mathbf{R}^\top\mathbf{R}=\bigl[\begin{smallmatrix}1&\rho\\\rho&1\end{smallmatrix}\bigr]$,
so the inputs of the two weights have correlation $0.72$.

\emph{Host state.} Let $s_0=6/5$ and $\mathbf{q}_0=(1,1)^\top$, so
$f(s_0,\mathbf{q}_0)=16889/50000\approx0.338$. Evaluating all 16 codes at
$s_0$ gives the next smallest value $f(s_0,(0,2)^\top)=19241/50000\approx0.385$;
hence $\mathbf{q}_0$ minimizes $f(s_0,\cdot)$ over $\mathbb{B}^2$.

\emph{Scale refit.} For $\mathbf{q}_0=(1,1)^\top$, the least-squares scale is
$s_1=(\mathbf{q}_0^\top\mathbf{M}\mathbf{w})/(\mathbf{q}_0^\top\mathbf{M}\mathbf{q}_0)
=(w_1+w_2)/2=197/200$, and $f(s_1,\mathbf{q}_0)=89383/500000\approx0.179$.

\emph{Babai proposal at $s_1$.} Back substitution on
$\mathbf{R}\mathbf{q}\approx\mathbf{y}/s_1$ gives the continuous values
$w_2/s_1=310/197\approx1.574$, so $q_2^+=2$, and then
$w_1/s_1+\rho(w_2/s_1-2)=588/4925\approx0.119$, so $q_1^+=0$; neither
value needs clipping. The proposal $\mathbf{q}^+=(0,2)^\top$ changes both
codes and has $f(s_1,\mathbf{q}^+)=3087/31250\approx0.099<f(s_1,\mathbf{q}_0)$,
so it is accepted.
\end{proof}

At $s_0$ the same back substitution returns $\mathbf{q}_0$ itself
($w_2/s_0=31/24$ and $14/25$ round to $1$ and $1$), in line with
$\mathbf{q}_0$ being optimal on the host grid. One more scale refit gives
$s_2=4631/5000$ and $f(s_2,\mathbf{q}^+)=132741/1562500\approx0.085$, a
$75\%$ reduction from the host state; the Babai proposal at $s_2$ is
$\mathbf{q}^+$ again, so the state is scale-optimal and Babai-stable.
The scaling argument at the end of Appendix~\ref{app:multicode-proof}
realizes this instance with the augmented design of
\eqref{eq:adaptive_grid_problem} for any fixed $\nu>0$.

\subsection{Babai Proposals versus Coordinatewise Rounding}
\label{app:babai-vs-rounding}

Fix an output column $j$, a group $i$, and write $s=s_{i,j}\neq0$,
$z=z_{i,j}$, and $\mathbf{t}=\mathbf{y}_j^{(i,k)}$ for the conditional
target in \eqref{eq:conditional_residual}. With
$\mathbf{v}=\mathbf{U}_i^\top\mathbf{t}/s$ as in \eqref{eq:babai_update}
and $\eta=\|(\mathbf{I}-\mathbf{U}_i\mathbf{U}_i^\top)\mathbf{t}\|_2^2$,
the group error in \eqref{eq:group_bils} is
\begin{equation}
E(\mathbf{q})
=\|\mathbf{t}-s\mathbf{A}_i(\mathbf{q}-z\mathbf{1})\|_2^2
=s^2\|\mathbf{v}-\mathbf{R}_i(\mathbf{q}-z\mathbf{1})\|_2^2+\eta .
\label{eq:group-error-reduced}
\end{equation}
Its unconstrained real minimizer is
$\boldsymbol{\xi}=z\mathbf{1}+\mathbf{R}_i^{-1}\mathbf{v}$.
Let $\mathbf{q}^{\mathrm{B}}$ be the Babai proposal
\eqref{eq:babai_update} and
$\mathbf{q}^{\mathrm{R}}=\operatorname{round}(\boldsymbol{\xi})$ the
coordinatewise rounding of $\boldsymbol{\xi}$.

\begin{proposition}[Babai versus coordinatewise rounding]
\label{prop:babai-vs-rounding}
\begin{enumerate}
\item[(a)] If no clipping is active in \eqref{eq:babai_update}, then
every entry of $\mathbf{v}-\mathbf{R}_i(\mathbf{q}^{\mathrm{B}}-z\mathbf{1})$
has magnitude at most $|(\mathbf{R}_i)_{\ell\ell}|/2$, and
\[
E(\mathbf{q}^{\mathrm{B}})\leq\eta+\frac{s^2}{4}
\sum_{\ell=1}^{d_i}(\mathbf{R}_i)_{\ell\ell}^2 .
\]
\item[(b)] Ignore the code bounds and suppose the fractional parts of
$\boldsymbol{\xi}$ are independent and uniform on $[0,1)$. Then
\[
\mathbb{E}\,E(\mathbf{q}^{\mathrm{B}})
=\eta+\frac{s^2}{12}\sum_{\ell}(\mathbf{R}_i)_{\ell\ell}^2,
\qquad
\mathbb{E}\,E(\mathbf{q}^{\mathrm{R}})
=\eta+\frac{s^2}{12}\|\mathbf{R}_i\|_F^2,
\]
and the difference is
$\frac{s^2}{12}\sum_{\ell<h}(\mathbf{R}_i)_{\ell h}^2
=\frac{s^2}{12}\bigl(\|\mathbf{A}_i\|_F^2-\sum_\ell(\mathbf{R}_i)_{\ell\ell}^2\bigr)
\geq0$, with equality if and only if $\mathbf{A}_i$ has orthogonal columns.
\end{enumerate}
\end{proposition}

Part (a) is deterministic and is the classical nearest-plane
bound~\citep{babai1986lovasz}; \citet{chen2025geometry} note the same
bound for GPTQ without clipping. Part (b) is exact when $\mathbf{v}$ is
uniformly distributed modulo the lattice $\mathbf{R}_i\mathbb{Z}^{d_i}$
and otherwise follows the standard uniform model of rounding error; it
compares expected errors.

\begin{proof}[Proof of Proposition~\ref{prop:babai-vs-rounding}]
Write $\mathbf{R}=\mathbf{R}_i$, $\mathbf{c}=\mathbf{q}^{\mathrm{B}}-z\mathbf{1}$,
and let $x_\ell=(v_\ell-\sum_{h>\ell}R_{\ell h}c_h)/R_{\ell\ell}$.
Since $z$ is an integer and no clipping is active,
$c_\ell=\operatorname{round}(x_\ell)$. The $\ell$th entry of
$\mathbf{v}-\mathbf{R}\mathbf{c}$ is
$v_\ell-\sum_{h\geq\ell}R_{\ell h}c_h=R_{\ell\ell}(x_\ell-c_\ell)$,
and $|x_\ell-c_\ell|\leq1/2$. Summing squares and using
\eqref{eq:group-error-reduced} gives (a).

For (b), write $\boldsymbol{\xi}=\mathbf{m}+\mathbf{u}$ with
$\mathbf{m}$ integer and $\mathbf{u}$ uniform on $[0,1)^{d_i}$.
Since $v_\ell=\sum_{h\geq\ell}R_{\ell h}(\xi_h-z)$,
\[
x_\ell=u_\ell+(m_\ell-z)
+\sum_{h>\ell}\frac{R_{\ell h}}{R_{\ell\ell}}(\xi_h-z-c_h),
\]
where the last sum depends only on $u_h$ for $h>\ell$. Conditional on
these, $x_\ell$ is uniform modulo one, so
$\delta_\ell=x_\ell-c_\ell$ is uniform on $[-1/2,1/2]$ and independent of
$\delta_h$ for $h>\ell$. Hence $\mathbb{E}\,\delta_\ell^2=1/12$ and
$\mathbb{E}\|\mathbf{v}-\mathbf{R}\mathbf{c}\|_2^2
=\sum_\ell R_{\ell\ell}^2/12$. For rounding,
$\mathbf{v}-\mathbf{R}(\mathbf{q}^{\mathrm{R}}-z\mathbf{1})
=\mathbf{R}(\boldsymbol{\xi}-\mathbf{q}^{\mathrm{R}})$, whose entries
$\xi_h-q^{\mathrm{R}}_h$ are independent, uniform on $[-1/2,1/2]$, and
have mean zero. Thus the expectation is $\|\mathbf{R}\|_F^2/12$.
The difference is the strictly upper-triangular energy of $\mathbf{R}$.
Because $\mathbf{U}_i$ has orthonormal columns,
$\|\mathbf{R}\|_F=\|\mathbf{A}_i\|_F$. The difference vanishes if and only
if $\mathbf{R}$ is diagonal. For nonsingular upper-triangular
$\mathbf{R}$, this holds if and only if
$\mathbf{A}_i^\top\mathbf{A}_i=\mathbf{R}^\top\mathbf{R}$ is diagonal,
by induction on the rows of $\mathbf{R}$.
\end{proof}

Proposition~\ref{prop:multicode-descent} shows that coordinated code
changes can help; Proposition~\ref{prop:babai-vs-rounding} gives the
expected size of the advantage over coordinatewise rounding: it is
proportional to the within-group off-diagonal energy of $\mathbf{R}_i$,
which is produced by correlated input coordinates. RTN assigns each
code by rounding its own weight, while GPTQ codes already come from a
nearest-plane recursion on the full Hessian
lattice~\citep{chen2025geometry}. The higher Babai improvement rate
for the RTN host than for the GPTQ host in Table~\ref{tab:mechanism-summary}
(40.14\% versus 13.53\%) is consistent with this view.

\subsection{Monotone Descent and Finite Termination}
\label{app:descent-proof}

The objective \eqref{eq:bbmils_columns} is a sum of per-column terms
$\mathcal{F}_j$, and column $j$'s updates change only
$(\mathbf{s}_j,\mathbf{q}_j)$. It therefore suffices to consider one
column. We write $(\mathbf{s}^{(k)},\mathbf{q}^{(k)})$ for its state
after iteration $k$, with $(\mathbf{s}^{(0)},\mathbf{q}^{(0)})$ from the
host, and drop the index $j$.

\begin{proof}[Proof of Theorem~\ref{thm:descent}]
\emph{Scale update.} Without the ridge term, the Cholesky solve returns
the minimizer of $\mathcal{F}(\cdot,\mathbf{q}^{(k-1)})$, which is unique
because $\lambda_j=0$ means $\mathbf{G}_j$ passed the Cholesky test and
is positive definite. Hence
\[
\mathcal{F}(\mathbf{s}^{(k)},\mathbf{q}^{(k-1)})
=\phi(\mathbf{q}^{(k-1)})
\leq\mathcal{F}(\mathbf{s}^{(k-1)},\mathbf{q}^{(k-1)}).
\]
\emph{Code update.} With the scales and all other groups fixed,
$\mathcal{F}$ as a function of $\mathbf{q}_i$ equals the conditional
residual in \eqref{eq:group_bils}, since \eqref{eq:conditional_residual}
subtracts exactly the current contributions of the other groups.
The QR reduction in \eqref{eq:qr_equivalence} and the projection onto
$\mathbf{U}_i$ in \eqref{eq:group-error-reduced} change this residual
only by constants independent of $\mathbf{q}_i$. The acceptance test
therefore accepts a proposal exactly when it strictly lowers
$\mathcal{F}$. Rejected proposals and skipped groups with
$|s_i|\leq\epsilon_{\rm fp}$ leave the state unchanged. This proves (i).
For $|s_i|>\epsilon_{\rm fp}$ the proposal is computed at the current
scale.

\emph{Profiled objective.} Codes change only through accepted proposals.
If iteration $k$ accepts one, then
$\phi(\mathbf{q}^{(k)})\leq\mathcal{F}(\mathbf{s}^{(k)},\mathbf{q}^{(k)})
<\mathcal{F}(\mathbf{s}^{(k)},\mathbf{q}^{(k-1)})=\phi(\mathbf{q}^{(k-1)})$,
which proves (ii).

\emph{Termination.} By (ii), the code vectors after code-changing
iterations have strictly decreasing values of $\phi$, so none repeats.
The set $\mathbb{B}^{d_{\mathrm{in}}}$ is finite, so at most
$|\mathbb{B}|^{d_{\mathrm{in}}}-1$ iterations change a code. Consider
the first iteration $k$ that accepts no proposal. Throughout its sweep
the state is $(\mathbf{s}^{(k)},\mathbf{q}^{(k-1)})$, so every proposal
was computed at this state and rejected, and $\mathbf{s}^{(k)}$
minimizes $\mathcal{F}(\cdot,\mathbf{q}^{(k-1)})$. The state is thus
scale-optimal and Babai-stable. The next scale update receives the same
codes and returns the same unique minimizer, so the same proposals are
rejected again. This proves (iii).
\end{proof}

\paragraph{Early stopping.}
Under the exact-solve assumption, $L_j^{(k)}$ in
\eqref{eq:strict_stopping} includes the constant in
\eqref{eq:qr_equivalence} and therefore equals
$\phi(\mathbf{q}^{(k-1)})$, so it is nonincreasing in $k$.
If iteration $k-1$ changes no code, then $L_j^{(k)}=L_j^{(k-1)}$ and
the rule stops the column. Hence a column stops at the latest one
iteration after reaching the fixed point in Theorem~\ref{thm:descent}.
It may stop earlier, when the decrease falls below $\tau$; that state
satisfies (i) but need not be Babai-stable.

\begin{proof}[Proof of Corollary~\ref{cor:host}]
For $K=0$ the host state is returned. Otherwise, the host scales are
feasible for the first scale update, so
$\mathcal{F}_j(\mathbf{s}_j^{(1)},\mathbf{q}_j^{(0)})
\leq\mathcal{F}_j(\mathbf{s}_j^{(0)},\mathbf{q}_j^{(0)})$.
By Theorem~\ref{thm:descent}(i), later half-steps do not increase
$\mathcal{F}_j$, and early stopping returns one of these states.
Summing over columns gives the result.
\end{proof}

Both states are evaluated with the refined model's inputs
$\widetilde{\mathbf{X}}$ and the same target. Since
$\widetilde{\mathbf{X}}$ depends on the preceding layers, the guarantee
is per layer; FP32 scale solves and the final cast to the model's
weight dtype can perturb it by rounding.

\begin{remark}[Making the guarantee unconditional]
\label{rem:ridge-acceptance}
When the ridge branch is triggered, the scales
$(\mathbf{G}_j+\lambda\mathbf{I})^{-1}\mathbf{B}_j^\top\mathbf{y}_j$
need not lower $\mathcal{F}_j$ relative to the previous scales, so
Theorem~\ref{thm:descent}(i) can fail at that update. Suppose the
ridge solution were kept only when it does not increase $\mathcal{F}_j$,
and the previous scales retained otherwise. Then (i) and
Corollary~\ref{cor:host} would hold without the exact-solve assumption.
With $\tau>0$, $L_j^{(k)}$ would be nonincreasing and bounded below, so
the absolute test in \eqref{eq:strict_stopping} would stop each column
within $L_j^{(1)}/\tau+2$ iterations. Part (iii) would still
require exact scale solves.
\end{remark}

\subsection{Quality of Babai-Stable Codes}
\label{app:stable-quality}

\begin{proposition}[Error at a Babai-stable point]
\label{prop:stable-quality}
Let $(\mathbf{s}_j,\mathbf{q}_j)$ be scale-optimal and Babai-stable for
column $j$, let $i$ be a group with $|s_{i,j}|>\epsilon_{\rm fp}$, and let
$\mathbf{r}$ be its conditional target, the residual after subtracting
the other groups' contributions. If the bounded Babai proposal for this
group needs no clipping, then
\begin{equation}
\|\mathbf{r}-s_{i,j}\mathbf{A}_i(\mathbf{q}_{i,j}-z_{i,j}\mathbf{1}_{d_i})\|_2^2
\le\min_{\mathbf{u}\in\mathbb{R}^{d_i}}\|\mathbf{r}-\mathbf{A}_i\mathbf{u}\|_2^2
+\frac{s_{i,j}^2}{4}\sum_{\ell=1}^{d_i}(\mathbf{R}_i)_{\ell\ell}^2 .
\label{eq:stable-quality}
\end{equation}
\end{proposition}

\begin{proof}
Write $s=s_{i,j}$, $z=z_{i,j}$, and $\mathbf{A}_i=\mathbf{U}_i\mathbf{R}_i$.
For any codes $\mathbf{q}$, orthogonality gives
\begin{equation*}
\|\mathbf{r}-s\mathbf{A}_i(\mathbf{q}-z\mathbf{1})\|_2^2
=\|(\mathbf{I}-\mathbf{U}_i\mathbf{U}_i^\top)\mathbf{r}\|_2^2
+s^2\|\mathbf{v}-\mathbf{R}_i(\mathbf{q}-z\mathbf{1})\|_2^2,
\qquad \mathbf{v}=\mathbf{U}_i^\top\mathbf{r}/s .
\end{equation*}
Since $\mathbf{R}_i$ is nonsingular, the first term equals
$\min_{\mathbf{u}}\|\mathbf{r}-\mathbf{A}_i\mathbf{u}\|_2^2$. For the
unclipped proposal $\mathbf{q}^+$, let
$t_\ell=\bigl(v_\ell-\sum_{h>\ell}(\mathbf{R}_i)_{\ell h}(q^+_h-z)\bigr)/(\mathbf{R}_i)_{\ell\ell}$,
so that $q^+_\ell=\operatorname{round}(z+t_\ell)$ by \eqref{eq:babai_update}
and $|t_\ell-(q^+_\ell-z)|\le 1/2$. The $\ell$th entry of
$\mathbf{v}-\mathbf{R}_i(\mathbf{q}^+-z\mathbf{1})$ is
$(\mathbf{R}_i)_{\ell\ell}\bigl(t_\ell-(q^+_\ell-z)\bigr)$, so its squared
norm is at most $\sum_\ell(\mathbf{R}_i)_{\ell\ell}^2/4$. Hence
$\mathbf{q}^+$ satisfies \eqref{eq:stable-quality}. With the other groups
and the scales fixed, $\mathcal{F}_j$ differs from the group's conditional
error by a constant, and Babai stability means that $\mathbf{q}^+$ does
not strictly lower $\mathcal{F}_j$. The current codes therefore have
conditional error no larger than that of $\mathbf{q}^+$.
\end{proof}

The bound is deterministic and needs no distributional assumption.
It compares the codes with the best real-valued weights for the group
on the conditional target and shrinks with the grid spacing $s_{i,j}$.

\subsection{Relation to GPTQ as Babai's Algorithm}
\label{app:gptq-babai}

\begin{remark}[Relation to GPTQ]
\label{rem:gptq-babai}
\citet{chen2025geometry} show that GPTQ, executed from the last to the
first dimension, is Babai's nearest-plane algorithm for a closest-vector
problem on a lattice defined by the Hessian, on a fixed quantization
grid. The code update \eqref{eq:babai_update} applies the same
recursion to a single group: the lattice is generated by
$s_{i,j}\mathbf{R}_i$ at the current scale, and the target is the
conditional residual \eqref{eq:conditional_residual}, from which the
latest contributions of all other groups are subtracted. With $K=0$,
{\upshape\method} returns the host state, whether or not the host is GPTQ.
The differences from GPTQ are as follows. (i) The grid changes: all
scales of a column are refit jointly given the codes, and the codes are
decoded again on the new grid. (ii) Each proposal is accepted only if
it lowers $\mathcal{F}_j$, which gives Theorem~\ref{thm:descent};
GPTQ applies each rounding decision once. (iii) Decoding is repeated
across groups and iterations, so a group's codes respond to later
changes in other groups. When the host is GPTQ, its codes are already
nearest-plane points of a closely related lattice, which is consistent
with the lower Babai improvement rate for GPTQ in
Table~\ref{tab:mechanism-summary}
(Appendix~\ref{app:babai-vs-rounding}). The lattices differ in the
calibration inputs, the regularization, and the grouping, so this
is not an identity.
\end{remark}

\subsection{Computational Complexity}
\label{app:complexity}

Let $d=d_{\mathrm{in}}$, $p=d_{\mathrm{out}}$, and assume equal-sized
groups of $g$ weights, so $G=d/g$. We count dense arithmetic for
QR reduction followed by cached group factors and incremental
residual updates. The counts cover one linear layer's refinement;
host quantization, activation collection, and model evaluation are
separate costs.

\paragraph{Shared preprocessing.}
Thin QR of the $(n+d)\times d$ augmented matrix costs
$O((n+d)d^2)$. Forming $\overline{\mathbf{Y}}=\mathbf{U}^\top\mathbf{Y}$ as in
Algorithm~\ref{alg:jarq}, equivalently solving
$\mathbf{R}^\top\overline{\mathbf{Y}}=\mathbf{A}^\top\mathbf{Y}$, costs $O((n+d)dp+d^2p)$.
For the main target, $\mathbf{Y}=\mathbf{A}\mathbf{W}$, so the
identity $\overline{\mathbf{Y}}=\mathbf{R}\mathbf{W}$ permits a
target-projection shortcut costing $O(d^2p)$ instead. The orthogonal
residual in \eqref{eq:qr_equivalence} is zero in exact arithmetic
for this target.
QR factorization of each reduced $d\times g$ group matrix costs
$O(dg^2)$, or $O(d^2g)$ across all groups. These factors are reused
across output columns and iterations.

\paragraph{One iteration for one output column.}
Constructing $\mathbf{B}_j\in\mathbb{R}^{d\times G}$ costs $O(d^2)$.
Forming its Gram matrix costs $O(dG^2)$, and the scale solve costs
$O(G^3)$, including the spectral check and Cholesky factorization.
The right-hand side costs $O(dG)$.
For code updates, cache the full residual and update it whenever a
group proposal is accepted. Projection onto a group's QR basis and
updating its contribution each cost $O(dg)$; bounded Babai back
substitution costs $O(g^2)$. A complete group sweep therefore costs
$O(d^2+dg)$, dominated by the residual updates rather than the back
substitution.

\paragraph{Total time and workspace.}
For an iteration budget $K\geq1$, the resulting bound is
\begin{equation}
\begin{aligned}
T_{\mathrm{ref}}
&=O\!\left((n+d)(d^2+dp)+Kp\left[d^2+dG^2+G^3\right]\right).
\end{aligned}
\label{eq:refinement-complexity}
\end{equation}
Here $g\leq d$ and $K\geq1$ allow the $d^2p$, $d^2g$, and $Kpdg$
terms to be absorbed; early stopping reduces the actual work.
Using the main-target shortcut above removes the $(n+d)dp$ term,
giving $O((n+d)d^2+Kp[d^2+dG^2+G^3])$.
The shared reduced matrix and group factors occupy $O(d^2)$ storage.
Processing one output column at a time requires an additional
$O(dG+G^2+d)$ workspace for the scale system and residuals, excluding
the stored model, host states, calibration inputs, and outputs.
Batching columns increases this workspace with the batch size.

Since $G=d/g$, halving the group size multiplies the $dG^2$ term by
four and the $G^3$ term by eight, while reducing the $dg$ back-substitution
term. This explains why smaller groups can increase refinement time,
as observed in Table~\ref{tab:group-size-time}.

\end{document}